\documentclass[pdflatex,sn-mathphys-num]{sn-jnl}

\usepackage{graphicx}%
\usepackage{multirow}%
\usepackage{amsmath,amssymb,amsfonts}%
\usepackage{amsthm}%
\usepackage{mathrsfs}%
\usepackage[title]{appendix}%
\usepackage{xcolor}%
\usepackage{textcomp}%
\usepackage{manyfoot}%
\usepackage{booktabs}%
\usepackage{algorithm}%
\usepackage{algorithmicx}%
\usepackage{algpseudocode}%
\usepackage{listings}%

\usepackage{enumitem}
\usepackage{tikz}
\usetikzlibrary{calc,positioning}
\usepackage{makecell}
\newcommand{\R}{\mathbb R}

\DeclareMathOperator*{\argmin}{arg\,min}

\theoremstyle{thmstyleone}%
\newtheorem{theorem}{Theorem}
\newtheorem{proposition}[theorem]{Proposition}%
\newtheorem{lemma}[theorem]{Lemma}
\theoremstyle{thmstyletwo}%
\newtheorem{remark}{Remark}%

\theoremstyle{thmstylethree}%
\newtheorem{definition}{Definition}%

\begin{document}

\title[Article Title]{Gradient Projection onto Historical Descent Directions for Communication-Efficient Federated Learning

}


\author*[1]{\fnm{Arnaud} \sur{Descours}}\email{arnaud.descours@univ-lyon1.fr}
\equalcont{These authors contributed equally to this work.}

\author[2]{\fnm{Léonard} \sur{Deroose}}\email{leonard.deroose@inria.fr}
\equalcont{These authors contributed equally to this work.}

\author[2]{\fnm{Jan} \sur{Ramon}}\email{jan.ramon@inria.fr}

\affil*[1]{\orgdiv{ISFA}, \orgname{UCBL}, \orgaddress{\city{Lyon}, \country{France}}}

\affil[2]{\orgdiv{INRIA}, \orgaddress{\city{Lille}, \country{France}}}



\abstract{
Federated Learning (FL) enables decentralized model training across multiple clients while optionally preserving data privacy. However, communication efficiency remains a critical bottleneck, particularly for large-scale models.
In this work, we introduce two complementary algorithms: \texttt{ProjFL}, designed for unbiased compressors, and \texttt{ProjFL+EF}, tailored for biased compressors through an Error Feedback mechanism. Both methods rely on projecting local gradients onto a shared client–server subspace spanned by historical descent directions, enabling efficient information exchange with minimal communication overhead.
We establish convergence guarantees for both algorithms under strongly convex, convex, and non-convex settings. Empirical evaluations on standard FL classification benchmarks with deep neural networks show that \texttt{ProjFL} and \texttt{ProjFL+EF} achieve accuracy comparable to existing baselines while substantially reducing communication costs.
}




\maketitle

\section{Introduction}\label{sec1}

In recent years, Federated Learning (FL) \cite{MAL-083} has emerged as a promising paradigm for training Machine Learning (ML) models across distributed data owners, without requiring direct access to the underlying data. This framework is particularly attractive in a landscape where organizations are increasingly concerned about data privacy, regulatory constraints (GDPR, AI Act), and the rising costs of centralized data infrastructure.
FL is a setting where multiple entities, called clients, collaborate in solving a ML problem, under the coordination of a central server. 
Mathematically, various ML tasks can be formulated as the optimization problem: 
\begin{equation}\label{eq.Minimization_pb}
\text{Minimize } f(w) = \frac{1}{M}\sum_{i=1}^M f_i(w)\  \text{ over }\ \mathbb{R}^d,
\end{equation}
where the vector $w$ represents the parameters of a statistical model and $f_i:\R^d\to\R$ is the local objective function associated with client $i\in[M]$.
Typically, in ML, the functions $f_i$ take the form
\begin{equation*}
f_i(w) = \mathbb{E}_{(x,y)\sim \pi_i}\left[\mathsf{L}(y, h_w(x))\right],
\end{equation*}
where, in a supervised learning task, $h_w: \mathcal{X} \to \mathbb{R}$ is the prediction function (\textit{e.g.}, a neural network), and $\mathsf{L}: \mathbb{R} \times \mathbb{R} \to \mathbb{R}_+$ a loss function. The data distribution $\pi_i$ is the underlying data distribution of client $i$, and is typically unknown. In practice, $\pi_i$ may be defined as a discrete distribution over the local dataset of client $i$, leading to empirical risk minimization. 
This motivates the use of either deterministic gradient descent (GD) or, more commonly, to stochastic gradient descent (SGD), which is preferred for its computational efficiency. 

The communication cost is a major bottleneck in FL, especially for optimization of models with many parameters. Two main strategies have been explored to mitigate this issue:
\begin{enumerate}
    \item \textbf{Local training}, which consists of reducing the communication frequency. Clients perform multiple local updates before sending their weights to the server for averaging.
    
    \item \textbf{Gradient compression}, where compressed information is sent instead of full-dimensional gradients.

\end{enumerate}

The first strategy involves performing multiple local updates (\textit{e.g.}, several SGD steps) before each communication round with the server, thereby reducing communication frequency. This classical approach, known as \texttt{FedAvg} \cite{pmlr-v54-mcmahan17a}, has inspired several extensions, such as \texttt{SCAFFOLD} \cite{pmlr-v119-karimireddy20a} and \texttt{FedPAGE} \cite{zhao2021fedpage}.
The second strategy can be broadly divided into two subcategories.
The first focuses on designing and analyzing various types of compression operators (see, \textit{e.g.}, \cite[Table 1]{xu2020compressed} and the analysis of \cite{philippenko2023compressed}).
The second aims to develop algorithms that remain effective under general compression schemes-or at least under broad classes of them. For example, several methods incorporate mechanisms to track and correct the compression error \cite{pmlr-v97-karimireddy19a,NEURIPS2021_231141b3}.
In this work, we focus on the second strategy and propose a new algorithmic approach for gradient compression in FL.

\subsection{Main Contributions} Our contributions are summarized as follows:

\begin{itemize}
\item \textbf{Algorithm Design:} We propose \texttt{ProjFL} (Algorithm~\ref{fed-avg-proj-var2}), a novel federated learning method designed for unbiased compressors, and \texttt{ProjFL+EF} (Algorithm~\ref{fed-avg-proj-EF}), an extension of the approach tailored for biased compressors that incorporates Error Feedback. Both methods project gradients onto a shared client–server subspace.

\item \textbf{Theoretical Guarantees:} Under standard assumptions, we establish: \textit{(i)} a linear convergence rate to a noise ball when optimizing strongly-convex objectives  (item 1 of Theorems~\ref{thm:conv-sgd} and~\ref{thm:EF_strong_conv});  \textit{(ii)} a $O(1/T)$ convergence rate towards a noise ball around the minimum of convex objectives (item 2 of Theorems~\ref{thm:conv-sgd} and~\ref{thm:EF_strong_conv}); and \textit{(iii)} a $O(1/T)$ convergence rate towards near-stationary points for smooth non-convex objectives (item 3 of Theorems~\ref{thm:conv-sgd} and~\ref{thm:EF_strong_conv}). 
\item \textbf{Empirical Validation:} We evaluate our methods on large-scale neural network models (LeNet-5, ResNet-20). Our methods exhibit accuracy and stability comparable to those of state-of-the-art baselines: it converges quickly without requiring careful hyperparameter tuning. Furthermore, it achieves up to a $8$×  reduction in communication cost compared to existing methods.  

\end{itemize}
\paragraph{Outline.} The rest of the paper is organized as follows. Section~\ref{sec:setting} introduces the formal problem setting and reviews related work. Section~\ref{sec:improvements} presents our proposed methods, which are theoretically analyzed in Section~\ref{sec:theor_results}. Numerical experiments are reported in Section~\ref{sec:numerics_main}. We conclude with future research directions in Section~\ref{sec:conclusion}.

\section{Setting and State of the Art}
\label{sec:setting}

To solve the optimization problem~\eqref{eq.Minimization_pb} under communication constraints, a classical algorithm is \texttt{FedAvg} with compression (see Algorithm~\ref{alg:fedavg_compr} in Appendix~\ref{sec:appx_num_exp}), which can be written as:
\begin{equation}
\label{eq.update_FedAvg}
w_{t+1} = w_t - \frac{\eta}{M} \sum_{i=1}^M \mathcal{C}(g_t^i),
\end{equation}
where $g_t^i$ is the stochastic gradient computed by client $i$ at iteration $t \in \mathbb{N}$, and $\mathcal{C}: \mathbb{R}^d \to \mathbb{R}^d$ is a compression operator. Each $g_t^i$ is assumed to be an unbiased estimator of the true gradient, \textit{i.e.}, $\mathbb{E}[g_t^i \mid w_t] = \nabla f_i(w_t)$.

For ease of exposition and analysis, we treat $\mathcal{C}$ as a mapping from $\mathbb{R}^d$ to $\mathbb{R}^d$, although in practical implementations: \textit{(i)} the model parameters are typically updated layer-wise;
\textit{(ii)} several compression techniques require adaptations to the standard update rule~\eqref{eq.update_FedAvg}. 
We now review commonly used classes of compressors.

\subsection{Compressors}

\paragraph{Quantization-based Compressors.}
These methods reduce communication by lowering the bit precision of the transmitted gradient values. Two common approaches are:

\begin{itemize}
    \item \textbf{Precision reduction:} Directly reducing the bitwidth of floating-point values. For instance, 8-bit quantization \cite{dettmers20168bit} or even 1-bit representations \cite{seide20141} are commonly used.
    
    \item \textbf{Dictionary-based methods:} Gradients are quantized via a shared codebook. For example, QSGD \cite{NIPS2017_6c340f25} defines the compressed vector $\mathcal{C}(g)$ componentwise as:
    \begin{equation}\label{eq.dictionary-based}
    \mathcal{C}(g)_j = \|g\|_2 \cdot \operatorname{sgn}(g_j) \cdot \xi_j(g, s), \quad j \in [d],
    \end{equation}
    where $s$ is the number of quantization levels, and the random variable $\xi_j(g, s)$ satisfies
    \[
    \xi_j(g, s) =
    \begin{cases}
    \frac{\ell}{s} & \text{with probability } 1+\ell - \frac{|g_j|}{\|g\|_2}s, \\
    \frac{\ell + 1}{s} & \text{otherwise},
    \end{cases}
    \]
    with $\ell\in\mathbb N$ chosen so that $\frac{|g_j|}{\|g\|_2} \in \left[\frac{\ell}{s}, \frac{\ell+1}{s}\right]$.

    In this scheme, clients send the tuple $(\|g\|_2, \sigma, \zeta)$, where $\sigma$ contains the signs of $g$ and $\zeta$ the quantized coefficients.

\end{itemize}
For additional examples, including low-rank quantization methods, we refer the reader to \cite{xu2020compressed, philippenko2023compressed}.

\paragraph{Sparsification-based Compressors.}
These methods reduce communication by enforcing sparsity, \textit{i.e.}, zeroing out a large fraction of the gradient coordinates:

\begin{itemize}
    \item \textbf{Rand-$k$:} Uniformly selects $k$ coordinates to retain:
    \[
    \mathcal{C}(g) = \frac{d}{k} \sum_{j \in S} g_j e_j,
    \]
    where $S$ is a random subset of $[d]$ with cardinality $k$, and $\{e_j\}_{j=1}^d$ is the canonical basis of $\mathbb{R}^d$. The scaling factor $d/k$ ensures unbiasedness.

    In practice, sparsification is typically applied layer-wise, and $k$ may be set as a fixed proportion of non-zero coordinates.

    \item \textbf{Top-$k$:} Retains the $k$ largest-magnitude components:
    \[
    \mathcal{C}(g) = \sum_{i=d-k+1}^d g_{(i)} e_{(i)},
    \]
    where coordinates are sorted by absolute value: $|g_{(1)}| \le \cdots \le |g_{(d)}|$. Variants such as Threshold-$v$ \cite{Dutta_Bergou_Abdelmoniem_Ho_Sahu_Canini_Kalnis_2020} retain all components exceeding a fixed threshold.
\end{itemize}

At the implementation level, a sparse vector is typically represented using two components: one vector containing the values of the selected elements of $g$, and another containing the ordered indices of the nonzero elements of $\mathcal{C}(g)$.
Quantization and sparsification methods can be combined: for example, one can first apply a \textbf{Top}-$k$ compressor and then quantize the $k$ nonzero components to reduce their precision.
In the context of analyzing the convergence of~\eqref{eq.update_FedAvg}, it is useful to distinguish compressors based on their bias properties. A compressor $\mathcal{C}: \R^d \to \R^d$ is said to be unbiased if for all $g \in \R^d$, $\mathbb{E}[\mathcal{C}(g)] = g$ (see Assumption~\ref{as:compressor} for details); otherwise, it is biased (see Assumption~\ref{as:compressorB}). Examples of unbiased compressors include the dictionary-based method defined in~\eqref{eq.dictionary-based} and \textbf{Rand}-$k$, whereas \textbf{Top}-$k$ is biased. We refer to \cite{khirirat2018distributed,10.1093/imaiai/iaab006} for comprehensive discussions on unbiased compressors.
Unbiased compressors enjoy appealing theoretical guarantees: for example, in gradient descent, one can establish convergence to a neighborhood of the optimum under unbiased compression when optimizing convex objectives \cite{khirirat2018distributed}. In contrast, analyzing biased compressors is more challenging. In fact, it has been shown that no general convergence guarantees can be obtained for biased compression, even in the strongly convex setting \cite[Section 5.2]{JMLR:v24:21-1548}.
Nevertheless, in practice, biased compressors such as \textbf{Top}-$k$ or its variants often outperform unbiased, randomized alternatives \cite{7835789,aji2017sparse}. To better understand and stabilize these biased methods, several new algorithmic frameworks have been proposed, which we now review.

\subsection{Algorithm Design}

A widely adopted mechanism in communication-efficient optimization is Error Feedback (EF), where each client stores the accumulated compression error and reinjects it in the next update \cite{seide20141,pmlr-v97-karimireddy19a}. This approach enables convergence guarantees even when using biased compressors, see \cite{JMLR:v21:19-748,stich2018sparsified,NEURIPS2018_31445061,JMLR:v24:21-1548}. Moreover, EF can also be applied on the server side, especially in cross-device FL scenarios \cite{pmlr-v97-tang19d}.

To improve theoretical guarantees and stability, the \texttt{EF21} method was proposed in \cite{NEURIPS2021_231141b3}. Rather than applying compression directly to gradients, \texttt{EF21} compresses the difference between the current stochastic gradient and the previous descent direction. This design has inspired a series of extensions  \cite{fatkhullin2021ef21,makarenko2022adaptive,pmlr-v202-gruntkowska23a,NEURIPS2023_f0b1515b}.

Another influential approach is \texttt{DIANA} \cite{Mishchenko27092024}, which introduces a shared memory vector maintained by both clients and the server. Compression is applied to the deviation between this memory and the local stochastic gradient. While the original method includes a proximal step to address regularized problems, a simplified version without this step can be used in the absence of regularization. Further developments and analyses around \texttt{DIANA} can be found in \cite{horvath2023stochastic,pmlr-v190-condat22a,pmlr-v108-gorbunov20a}.



\section{Two New Algorithms\label{sec:improvements}}

In this section, we introduce the algorithms evaluated in this work, specifically designed for FL under gradient compression.

\subsection{Algorithm~\ref{fed-avg-proj-var2}: \texttt{ProjFL}}

Our first contribution is the \texttt{ProjFL} algorithm, which we now describe in detail.  
This algorithm leverages the fact that both the server and each client $i$ have access to the local descent direction $\mathsf{D}_t^i$ at iteration $t$.  
At the next iteration, instead of compressing the full stochastic gradient $g_{t+1}^i$, client $i$ projects this gradient onto the one-dimensional subspace generated by $\bar{\mathsf{D}}_t^i$, the average of the previous descent directions $\mathsf{D}_t^i$. This yields the decomposition  
$
g_{t+1}^i = \alpha_{t+1}^i \bar{\mathsf{D}}_t^i + (g_{t+1}^i)^\perp$,
where $(g_{t+1}^i)^\perp$ is orthogonal to $\bar{\mathsf{D}}_t^i$, \textit{i.e.}, $\bar{\mathsf{D}}_t^i \cdot (g_{t+1}^i)^\perp = 0$ (see line 6 of Algorithm~\ref{fed-avg-proj-var2}).  
The scalar $\alpha_{t+1}^i$ thus captures the entire component of the gradient in the known direction, enabling a highly compact representation.
Instead of compressing the full gradient, the client compresses only the orthogonal component $(g_{t+1}^i)^\perp$ and transmits the pair $(\alpha_{t+1}^i, \mathsf{M}_{t+1}^i)$ to the server, where $\mathsf{M}_{t+1}^i$ is the compressed version of $(g_{t+1}^i)^\perp$.

Two extreme cases help illustrate the benefits of this approach:  
\begin{itemize}
  \item If $g_{t+1}^i$ is nearly aligned with $\bar{\mathsf{D}}_t^i$, then $\alpha_{t+1}^i$ captures nearly all the gradient information, and the compressed component is negligible.  
  \item Conversely, if $g_{t+1}^i$ is orthogonal to $\bar{\mathsf{D}}_t^i$, then $\alpha_{t+1}^i = 0$, and only the orthogonal part is transmitted, avoiding any misleading bias from projecting in the wrong direction.
\end{itemize}
Figure~\ref{fig:drawing_EF21_ProjFedAvg} illustrates the key differences between Algorithm~\ref{fed-avg-proj-var2} and \texttt{EF21}.

\begin{algorithm}
\caption{\texttt{ProjFL}}\label{fed-avg-proj-var2}
\begin{algorithmic}[1]
\State \textbf{Initialization:} $w_0\in\R^d$, $\mathsf D_0^i=0$. Number of previous directions of descent to consider $K\in\mathbb N^*$.
\For{$t=0,\dots, T$}
\For{Each client $i$}
\State Receive $w_t$.
\State Compute Stochastic Gradient $g_{t+1}^i$.
\State\label{l6_a} $\bar{\mathsf D}_t^i = \frac1K\sum_{k=0}^{K-1}\mathsf D_{t-k}^i$\Comment{with obvious adaptation\footnotemark when $t<K-1$}
\State $\alpha_{t+1}^i = \frac{g_{t+1}^i\cdot\bar{\mathsf D}_t^i}{\|\bar{\mathsf D}_t^i\|_2^2}$ such that $g_{t+1}^i = \alpha_{t+1}^i\bar{\mathsf D}_t^i+(g_{t+1}^i)^\perp$ satisfying $\bar{\mathsf D}_t^i\cdot (g_{t+1}^i)^\perp=0$. 
\State $\mathsf M_{t+1}^i = \mathcal C((g_{t+1}^i)^\perp)$
\State $\mathsf D_{t+1}^i = \alpha_{t+1}^i\bar{\mathsf D}_t^i +\mathsf M_{t+1}^i$ \Comment{Update the descent direction}
\State Send $(\alpha_{t+1}^i, \mathsf M_{t+1}^i)$ to the Central Server.
\EndFor
\State \textbf{Central Server:}
\State $\bar{\mathsf D}_t^i = \frac1K\sum_{k=0}^{K-1}\mathsf D_{t-k}^i$ \Comment{Same computation as line \ref{l6_a}}
\State $\mathsf D_{t+1}^i = \alpha_{t+1}^i\bar{\mathsf D}_t^i +\mathsf M_{t+1}^i$. \Comment{Update the descent direction}
\State $w_{t+1} = w_t -\frac \eta M\sum_{i=1}^M\mathsf D_{t+1}^i$. 
\EndFor
\end{algorithmic}
\end{algorithm}
\footnotetext{If $t<K-1$, then $\bar{\mathsf D}_t^i$ is defined as $\bar{\mathsf D}_t^i = \frac{1}{t+1}\sum_{k=0}^{t}D_{t-k}^i$. }
\begin{figure}
    \centering
    \begin{tikzpicture}[dot/.style={circle,inner sep=1pt,fill,label={#1},name=#1},
  extended line/.style={shorten >=-#1,shorten <=-#1},
  extended line/.default=1cm]
 \coordinate (A) at (0,0);
 \coordinate (grad) at (4.5,3);
 \coordinate (D) at (2,0.5); 
 \draw [-stealth] (A) -- (grad) node[pos=1.08, font=\small]{$g_{t+1}^i$} ;
 \draw [-stealth] (A) -- (D) node[below ,font=\small]{$\mathsf D_t^i$} ;
 \draw [-stealth,red] (D) -- (grad) node[midway, below right, font=\small] {\textcolor{red}{$g_{t+1}^i- \mathsf D_t^i$}}; 
 \draw[stealth-,blue] (grad) --  node[midway,right]{\textcolor{blue}{$(g_t^i)^\perp$}} ($(A)!(grad)!(D)$) node[below right]{\textcolor{black}{$\alpha_{t+1}^i \mathsf D_t^i$}};     
 \draw [-stealth, dashed] (D) -- ($(A)!(grad)!(D)$);
\end{tikzpicture}
    \caption{Comparison between \texttt{EF21} and \texttt{ProjFL}: In \texttt{EF21}, client $i$ sends the compressed difference $\mathcal{C}(\textcolor{red}{g_{t+1}^i - \mathsf{D}_t^i})$, while in \texttt{ProjFL}, the message is $\mathcal{C}(\textcolor{blue}{(g_t^i)^\perp})$. }
    \label{fig:drawing_EF21_ProjFedAvg}
\end{figure}
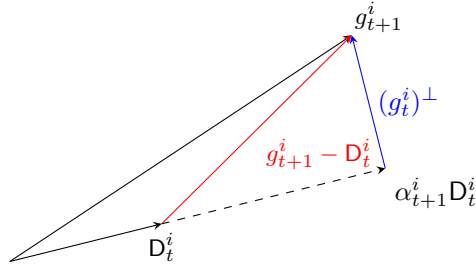

\subsection{Algorithm~\ref{fed-avg-proj-EF}: \texttt{ProjFL+EF}}

We now present an extension of Algorithm~\ref{fed-avg-proj-var2} that incorporates the Error Feedback mechanism.  
In Algorithm~\ref{fed-avg-proj-EF}, we modify Algorithm~\ref{fed-avg-proj-var2} to explicitly account for the compression error.  
More precisely, after computing the decomposition \( g_{t+1}^i = \alpha_{t+1}^i \bar{\mathsf D}_t^i + (g_{t+1}^i)^\perp \), each client adds the previous compression error \( e_t^i \) to the orthogonal component.  
The message sent to the server is then \( (\eta\alpha_{t+1}^i, \mathcal{C}(\eta(g_{t+1}^i)^\perp + e_t^i)) \).

\begin{algorithm}
\caption{\texttt{ProjFL+EF}} \label{fed-avg-proj-EF}
\begin{algorithmic}[1]
\State \textbf{Initialization:} $w_0\in \R^d$,  $\mathsf D_0^i=e_t^i=0$,  $K\in\mathbb N^*$. 
\For{$t=0,\dots, T$}
\For{each client $i$}
\State Receive $w_t$.
\State Compute Stochastic Gradient $g_{t+1}^i$.  
\State\label{l6_b} $\bar{\mathsf D}_t^i = \frac1K\sum_{k=0}^{K-1}\mathsf D_{t-k}^i$\Comment{with obvious adaptation when $t<K-1$ (see Alg. \ref{fed-avg-proj-var2})}
\State  $\alpha_{t+1}^i = \frac{g_{t+1}^i\cdot\bar{\mathsf D}_t^i}{\|\bar{\mathsf D}_t^i\|_2^2}$ such that $g_{t+1}^i = \alpha_{t+1}^i\bar{\mathsf D}_t^i+(g_{t+1}^i)^\perp$ satisfying $\bar{\mathsf D}_t^i\cdot (g_{t+1}^i)^\perp=0$.
\State $\mathsf M_{t+1}^i = \mathcal C(\eta(g_{t+1}^i)^\perp+e_t^i)$
\State$\mathsf D_{t+1}^i = \eta\alpha_{t+1}^i\bar{\mathsf D}_t^i +\mathsf M_{t+1}^i$ \Comment{Update the descent direction}
\State$e_{t+1}^i = \eta(g_{t+1}^i)^\perp+e_t^i - \mathsf M_{t+1}^i$ \Comment{Update compression error}
\State Send $(\eta\alpha_{t+1}^i, \mathsf M_{t+1}^i)$ to the Central Server.
\EndFor
\State \textbf{Central Server:}
\State $\bar{\mathsf D}_t^i = \frac1K\sum_{k=0}^{K-1}\mathsf D_{t-k}^i$ \Comment{Same computation as line \ref{l6_b}}
\State  $\mathsf D_{t+1}^i = \eta\alpha_{t+1}^i\bar{\mathsf D}_t^i +\mathsf M_{t+1}^i$. \Comment{Update the descent direction}
\State $w_{t+1} = w_t -\frac 1M\sum_{i=1}^M\mathsf D_{t+1}^i$. 
\EndFor
\end{algorithmic}
\end{algorithm}

\section{Theoretical convergence results\label{sec:theor_results} }

Let us first introduce some definitions. 

\begin{definition}[$\mu$-strongly convex functions]\label{def:mu-strong_conv}
A function $f:\R^d\to\R$ is called strongly convex if there exists $\mu>0$ such that for all $x,y\in\R^d$, 
$$f(y)\ge f(x)+\langle\nabla f(x),y-x\rangle+\frac\mu 2\|x-y\|^2.$$     
\end{definition}

\begin{definition}[$L$-smooth functions]\label{def:L-smooth}
A function $f:\R^d\to\R$ is called $L$-smooth if $\nabla f$ is $L$-Lipschitz continuous, \textit{i.e.}, for all $x,y\in\R^d$, \begin{equation}\label{ineg:defLsmooth}
\|\nabla f(x)-\nabla f(y)\|\le L\|x-y\|.
\end{equation}
\end{definition}

\subsection{On Algorithm~\ref{fed-avg-proj-var2}: \texttt{ProjFL}}
We give below convergence results on \texttt{ProjFL} (Algorithm~\ref{fed-avg-proj-var2}). 
Let us introduce the following assumptions: 
\begin{enumerate}[label=\textsc{h\arabic*}, ref=\textsc{h\arabic*}]
\item\label{as:compressor} There exists an i.i.d. sequence of compressors $(\mathcal C_t^i)_{i\in[M],t\ge1}$ where for any $i\in[M]$ and $t\ge0$, $\mathcal C_{t+1}^i:\R^d\to\R$ is the compressor employed by client $i$ at iteration $t$. For convenience, we simply write $\mathcal C$. We assume that for all $w\in\R^d$, $\mathbb E\big[\mathcal C(w)\big] = w$  and $\mathbb E\big[\|\mathcal C(w)\|^2\big]\le \beta \|w\|^2$ for some $\beta\ge1$.\footnote{Note that since $\mathbb E[\|X-\mathbb E[X]\|_2^2]=\mathbb E[\|X\|_2^2]-\|\mathbb E[X]\|_2^2$, Assumption~\ref{as:compressor} implies $\mathbb E[\|\mathcal C(w)-w\|_2^2]\le (\beta-1)\|w\|_2^2$.}  
\item\label{as:bounded_g_diss} There exist $a,b>0$ such that for all $w\in\R^d, \frac1M\sum_{i=1}^M\|\nabla f_i(w)\|^2\le a+b\|\nabla f(w)\|^2$.  
\item\label{as:stoc_gradient} There exists an i.i.d. sequence of random vector fields $(\xi_t^i)_{i\in[M],t\ge1}$ such that for any $i\in[M]$ and  $t\ge0$,  $g_{t+1}^i=\nabla f_i(w_t)+\xi_{t+1}^i(w_t)$ and for any $w\in\R^d$, $\mathbb E[\xi_t^i(w)] = 0$ and $\mathbb E[\|\xi_t^i(w)\|^2]\le\sigma^2$ for some $\sigma>0$. Moreover, the sequences $(\xi_t^i)_{i\in[M],t\ge1}$ and $(\mathcal C_t^i)_{i\in[M],t\ge1}$ are independent.  
\end{enumerate} 
We now discuss the role and implications of these assumptions.  
\begin{itemize}
    \item Assumption~\ref{as:compressor} restricts our analysis to \emph{unbiased} compression operators. Crucially, since Algorithm~\ref{fed-avg-proj-var2} lacks an error feedback mechanism, convergence guarantees cannot be extended to biased compressors in this framework.  For biased compressors, see Section~\ref{sec:theory_EF}.  

    \item Assumption~\ref{as:bounded_g_diss} formalizes the \emph{bounded gradient dissimilarity} condition, which quantifies data heterogeneity across clients by controlling the deviation between local gradients $\nabla f_i(w)$ and their global average $\nabla f(w)$.  

    \item Assumption~\ref{as:stoc_gradient} bounds the variance of stochastic gradient noise arising from mini-batch sampling. Notably, this noise intensity may vary both across clients (due to differences in local loss landscapes) and across parameter regions (\textit{e.g.}, near optima versus high-curvature regions).  
\end{itemize}
 
In the following theorem---whose proof is given in Appendix~\ref{sec:proof_conv_sgd}---we prove a linear convergence rate for Algorithm~\ref{fed-avg-proj-var2} towards a neighborhood of the optimal point. 
\begin{theorem}\label{thm:conv-sgd}
Assume~\ref{as:compressor}-\ref{as:bounded_g_diss}-\ref{as:stoc_gradient}. Then the following holds: 
\begin{enumerate}
    \item\label{thm:noEF:itemSconv} if each $f_i$ is $\mu$-strongly convex and $L$-smooth (see Definitions~\ref{def:mu-strong_conv} and~\ref{def:L-smooth}), then, for any learning rate $0\le \eta\le (1+b\frac{\beta-1}{M})^{-1}\frac{2}{\mu+L}$, the sequence $(w_t)_{t\ge0}$ generated by Algorithm~\ref{fed-avg-proj-var2} satisfies, for any $t\in\mathbb N$,
\begin{equation}\label{eq_conv_rate_sgd}
\mathbb E[\|w_{t}-w^*\|^2] \le \Big(1-2\eta\frac{\mu L}{\mu+L}\Big)^t\mathbb E[\|w_{0}-w^*\|^2] + \eta\frac{\mu+L}{2\mu LM}(a(\beta-1)+\beta\sigma^2), 
\end{equation}
where $w^*=\argmin_{\R^d} f$ is the unique global minimizer. 
\item\label{thm:noEF:itemconv} if each $f_i$ is  convex and $L$-smooth, and 
$f$ admits minimizers\footnote{This holds, for instance, if $f$ is coercive, i.e., $|f(x)|\to+\infty$ as $\|x\|\to+\infty$.} then, for any learning rate $0<\eta\le\frac{1}{2L}(1+b\frac{\beta-1}{M})^{-1}$,  the sequence $(w_t)_{t\ge0}$ generated by Algorithm~\ref{fed-avg-proj-var2} satisfies, for any $T\ge1$,
$$\mathbb E[f(w^{\mathrm{out}})] -f^* \le \frac{1}{(T+1)\eta}\mathbb E[ \|w_0-w^*\|^2] + \frac{\eta}{M}(a(\beta-1)+\beta\sigma^2),$$
where the output $w^{\mathrm{out}}$ is sampled uniformly at random in $\{w_t\}_{t=0}^T$, and $w^*\in\argmin_{\R^d} f$ is any minimizer.  
\item\label{thm:noEF:itemnoconv} if each $f_i$ is  $L$-smooth and bounded from below, then, for any learning rate $0< \eta\le\frac1L(1+b\frac{\beta-1}{M})^{-1}$, the sequence $(w_t)_{t\ge 0}$ generated by Algorithm~\ref{fed-avg-proj-var2} satisfies, for any $T\ge1$,
\begin{align*}
\mathbb E[\|\nabla f(w^{\mathrm{out}})\|^2]\le \frac{2}{(T+1)\eta}\mathbb E[f(w_0)-f^*]+\frac {L\eta}{M}(a(\beta-1)+\beta\sigma^2),
\end{align*}
where the output $w^{\mathrm{out}}$ is sampled uniformly at random in $\{w_t\}_{t=0}^T$ and $f^*=\inf_{\mathbb R^d} f$.
\end{enumerate}
\end{theorem}

\begin{remark}
Since $f(w_t) - f^* \le \frac{L}{2} \|w_t - w^*\|^2$ for any $L$-smooth function, item~\ref{thm:noEF:itemSconv} yields
$$
\mathbb{E}[f(w_t)] - f^* \le \frac{L}{2}\Big(1 - 2\eta\frac{\mu L}{\mu + L}\Big)^t \mathbb{E}[\|w_0 - w^*\|^2]
+ \eta\frac{\mu + L}{4\mu M}\big(a(\beta - 1) + \beta\sigma^2\big).
$$
\end{remark}

The first term on the right-hand side of~\eqref{eq_conv_rate_sgd} reflects the initial distance to the optimum.  
The second term is a variance term corresponding to the (squared) radius of the neighborhood around the optimal point to which the algorithm converges.  

This result also provides theoretical support for the common practice of decreasing the learning rate during training.  
A useful heuristic is as follows: at initialization, when the term $\mathbb{E}[\|w_0 - w^*\|^2]$ typically dominates, a relatively large learning rate should be used so that this term is quickly reduced thanks to the linear convergence rate.  
As training progresses and the second term on the right-hand side of~\eqref{eq_conv_rate_sgd} becomes dominant, the effective "starting point" is now closer to the optimum, and using a smaller step size becomes beneficial to reduce variance and improve stability.  

Similar comments apply to items~\ref{thm:noEF:itemconv} and~\ref{thm:noEF:itemnoconv}.  
Note also that when $\beta = 1$ (\textit{i.e.}, no compression), we recover the standard noise terms on the right-hand side of each inequality.  

Finally, note that in many practical applications, these convergence rates may coexist in the following sense: even if $f$ is non-convex, it may still be (strongly) convex in a neighborhood of a local minimizer, meaning that locally, the linear convergence rate of item~\ref{thm:noEF:itemSconv} may describe the final phase of training.

\subsection{On Algorithm~\ref{fed-avg-proj-EF}: \texttt{ProjFL+EF}\label{sec:theory_EF} }

Since the Error Feedback mechanism is designed to mitigate the bias introduced by the compressor, we deal here with biased compressors and thus replace Assumption~\ref{as:compressor} by the following 

\begin{enumerate}[label=\textsc{h\arabic*'}, ref=\textsc{h\arabic*'}]
    \item\label{as:compressorB} There exists an i.i.d. sequence of compressors $(\mathcal C_t^i)_{i\in[M],t\ge1}$ where for any $i\in[M]$ and $t\ge0$, $\mathcal C_{t+1}^i:\R^d\to\R$ are the compressors employed by client $i$ at iteration $t$. For convenience, we simply write $\mathcal C$. We assume that there exists $0<\delta\le 1$ such that for all $w\in\mathbb R^d$, $\mathbb E[\|\mathcal C(w)-w\|^2]\le (1-\delta)\|w\|^2$. 
\end{enumerate}
Our result follows the same structure as Theorem~\ref{thm:conv-sgd}, providing convergence rates for the strongly convex, convex, and non-convex settings.
\begin{theorem}\label{thm:EF_strong_conv}
Assume~\ref{as:compressorB}-\ref{as:bounded_g_diss}-\ref{as:stoc_gradient}. Then the following holds: 
\begin{enumerate}
    \item\label{thm:EF_strong_conv:itemSconv} if each $f_i$ is $\mu$-strongly convex and $L$-smooth (see Definitions~\ref{def:mu-strong_conv} and~\ref{def:L-smooth}), then, for any learning rate $0< \eta\le\min\Big(\frac{\delta}{L(4+\delta)}, \frac{\delta}{\sqrt{40(2L+\mu)bL}} \Big)$, the sequence $(w_t)_{t\ge 0}$ generated by Algorithm~\ref{fed-avg-proj-EF} satisfies, for any $T\ge1$, 
\begin{equation*}
\mathbb E[f(w^{\mathrm{out}})]-f^* \le\frac{10}{\eta}\Big(1-\frac{\eta\mu}{2}\Big)^{T+1} \mathbb E[\|w_0-w^*\|^2]+20(2L+\mu)\frac{(1-\delta)\eta^2}{\delta}\left(\frac{2a}{\delta}+\sigma^2\right) + 10\frac{\eta\sigma^2}{M},
\end{equation*}
where the output $w^{\mathrm{out}}\in\{w_t\}_{t=0}^T$ is chosen to be $w_t$ with probability proportional to $(1-\frac{\eta\mu}{2})^{-t}$.
\item\label{thm:EF_strong_conv:itemconv} if each $f_i$ is convex and $L$-smooth, and assume also that $f$ admits minimizers\footnote{This holds, for instance, if $f$ is coercive, i.e., $|f(x)|\to+\infty$ as $\|x\|\to+\infty$.},  then, for any learning rate $0< \eta\le\min\Big(\frac{\delta}{L(4+\delta)}, \frac{\delta}{\sqrt{80b}L} \Big)$ , the sequence $(w_t)_{t\ge 0}$ generated by Algorithm~\ref{fed-avg-proj-EF}  satisfies, for any $T\ge1$, 
\begin{equation*}
\mathbb E[f(w^{\mathrm{out}})]-f^* \le\frac{10}{\eta(T+1)} \mathbb E[\|w_0-w^*\|^2]+40L\frac{2(1-\delta)\eta^2}{\delta}\left(\frac{2a}{\delta}+\sigma^2\right) + 10\frac{\eta\sigma^2}{M},
\end{equation*}
where the output $w^{\mathrm{out}}$ is sampled uniformly at random in $\{w_t\}_{t=0}^T$, and $w^*\in\argmin_{\R^d} f$ is any minimizer.
\item\label{thm:EF_strong_conv:itemnoconv} if each $f_i$ is  $L$-smooth and bounded from below, then, for any learning rate $0< \eta\le\frac{\delta}{4\sqrt{2(b+1)}L}$, the sequence $(w_t)_{t\ge 0}$ generated by Algorithm~\ref{fed-avg-proj-EF} satisfies, for any $T\ge1$,
\begin{align*}
\mathbb E[\|\nabla f(w^{\mathrm{out}})\|^2]\le \frac{8}{(T+1)\eta}\mathbb E[f( w_0)-f^*] +  \frac{8(1-\delta)\eta^2L^2}{\delta}\left(\frac{2a}{\delta}+\sigma^2\right)+\frac{8\eta L\sigma^2}{2M},
\end{align*}
where $w^{\mathrm{out}}$ is sampled uniformly at random in $\{w_t\}_{t=0}^T$ and $f^*=\inf_{\mathbb R^d} f$.
\end{enumerate}
\end{theorem}

\begin{remark}
Since $\|w^{\mathrm{out}} - w^*\|^2 \le \frac{2}{\mu}\big(f(w^{\mathrm{out}}) - f^*\big)$ for any $\mu$-strongly convex function, item~\ref{thm:EF_strong_conv:itemSconv} yields
$$
\mathbb{E}\big[\|w^{\mathrm{out}} - w^*\|^2\big]
\le
\frac{20}{\eta\mu}\Big(1 - \frac{\eta\mu}{2}\Big)^{T+1}\mathbb{E}\big[\|w_0 - w^*\|^2\big]
+ 40(2L + \mu)\frac{(1 - \delta)\eta^2}{\delta\mu}\left(\frac{2a}{\delta} + \sigma^2\right)
+ 20\frac{\eta\sigma^2}{\mu M}.
$$
\end{remark}

The comments made after Theorem~\ref{thm:conv-sgd} remain relevant in this case.  
Again, when there is no compression $(\delta = 1)$, we recover the usual rates toward a SGD neighborhood of the optimal point.  

\textbf{On the proof of Theorem~\ref{thm:EF_strong_conv}.}  
The proof relies on the analysis of the auxiliary sequence $\tilde{w}_t = w_t - e_t$, where $e_t = \frac{1}{M}\sum_i e_t^i$.  
(Note that $e_t$ is never computed in practice; it is introduced solely for mathematical analysis.)  
The sequence $(\tilde{w}_t)_{t \ge 0}$ enjoys the appealing property
$\tilde{w}_{t+1} = \tilde{w}_t - \eta g_{t+1}$, where $g_{t+1} = \frac{1}{M}\sum_i g_{t+1}^i$.  
Since $\mathbb{E}[g_{t+1} \mid \tilde{w}_t] = \nabla f(w_t)$, we almost recover the unbiased case.  
To conclude, it remains to bound the error term $e_t$, which is the purpose of Lemma~\ref{lem:bound_e_t_projFLEF}.  
The detailed proof is provided in Appendix~\ref{appx:proof_thm2}.

\section{Numerical experiments\label{sec:numerics_main}}
In this section, we numerically evaluate the benefits of Algorithms~\ref{fed-avg-proj-var2} and~\ref{fed-avg-proj-EF} to reduce the communication cost in FL, for classification tasks using neural networks.

\subsection{Datasets, models and preprocessing.}

We conduct experiments on MNIST and CIFAR-10 datasets\footnote{\url{https://docs.pytorch.org/vision/main/datasets.html}}.
MNIST contains 70,000 grayscale images of handwritten digits (28×28, 1 channel), split into 60,000 training and 10,000 test samples.
CIFAR-10 comprises 60,000 color images (32×32, 3 channels), with 50,000 for training and 10,000 for testing, across 10 classes.
We use LeNet-$5$ \cite{lecun2015lenet} for MNIST: two convolutional blocks (with $5{\times}5$ kernels, ReLU activations, and $2{\times}2$ max pooling), followed by two fully connected layers with ReLU.
For CIFAR-10, we adopt ResNet-$20$ \cite{He_2016_CVPR} implemented in \cite{Idelbayev18a}.
These models are standard architectures for their respective tasks.
In all settings, $20\%$ of the training set is used for validation.
Images are normalized to $[-1, 1]$, data is shuffled, and samples are evenly distributed across clients, with slight variation due to indivisibility.

\subsection{Experimental methodology.}
To assess the efficiency of our algorithms in solving learning tasks while maintaining reasonable communication costs, we monitor convergence-related metrics—specifically, the loss (cross-entropy in our case), accuracy, and the norm of the loss gradient—as functions of the total number of communicated bits.
Unless otherwise specified, the results reported in this section are based on the total communication cost (including both uplink and downlink). An exception is Figure~\ref{fig:1000c}, which presents results for uplink communication only.
Additional experiments, including uplink or downlink evaluations and further performance metrics, are provided in Appendix~\ref{sec:appx_num_exp}.
We focus on sparsification-based compressors (\textbf{Top-}$k$) and evaluate performance across $M$ clients, where $M \in \{3, 10, 100, 1000\}$.

\subsection{Parameters Selection}
\subsubsection{Hyperparameters\label{subsubsec:hyperparam}}
Our objective is not to surpass state-of-the-art performance, as such results often exhibit strong dependence on hyperparameter configurations and may inadvertently reflect tuning bias. To ensure a fair and transparent comparison, we adopt a consistent tuning strategy: all hyperparameters are optimized without compression (i.e., $\mathcal{C} = \mathrm{Id}$), and the resulting configurations are applied uniformly across all algorithms. The experimental settings is detailed in~\ref{sec:app_exp_details}.

\subsubsection{Algorithm parameters}
All algorithms are run using the same hyperparameter settings.
Before comparing Algorithm~\ref{fed-avg-proj-var2} and Algorithm~\ref{fed-avg-proj-EF} with existing methods, we first need to select the parameters $k$ (used in the \textbf{Top-}$k$ compressor) and $K$ (as defined in Algorithms~\ref{fed-avg-proj-var2} and~\ref{fed-avg-proj-EF}). 
We empirically set these parameters by balancing the trade-offs between convergence speed, communication cost, and stability across both datasets. 
In our experiments, we fix\footnote{$k = 0.01$ means that, in each layer, only the fraction $k$ of components with the largest magnitudes are retained. } $k = 0.01$ and $K = 3$. 
Detailed parameter-tuning results are provided in Figure~\ref{fig:find_best_K} of Appendix~\ref{sec:app_exp_details}.

Furthermore, under our experimental setup, we found that, without further tuning, \texttt{EF21} and \texttt{DIANA} underperformed. A similar behavior was already observed in \cite[Section 2.2]{NEURIPS2023_f0b1515b}. To address this, we introduced an additional parameter $\gamma$ into both methods, which significantly improved their performance (see Algorithms~\ref{alg:ef21} and~\ref{alg:diana} for details). A thorough discussion is provided in Appendix~\ref{sec:appx_num_exp}.

\subsection{Evaluation of Algorithms~\ref{fed-avg-proj-var2} and~\ref{fed-avg-proj-EF}}
We compare our methods to the baselines \texttt{FedAvg} with compression, \texttt{EF}, \texttt{EF21}, and \texttt{DIANA} (see Algorithms~\ref{alg:fedavg_compr},~\ref{fedavg-ef},~\ref{alg:ef21}, and~\ref{alg:diana} in Appendix~\ref{sec:appx_num_exp}). Both training and test losses are reported. 
The results---with an increasing number of clients---are presented in Figure~\ref{fig:3c}, ~\ref{fig:10c}, ~\ref{fig:100c}, ~\ref{fig:1000c} (with further comparisons using different metrics and parameters in Appendix~\ref{sec:appx_num_exp}).

\begin{figure}
    \caption{Comparison of  Algorithms~\ref{fed-avg-proj-var2} and~\ref{fed-avg-proj-EF} with \texttt{FedAvg} with compression, \texttt{EF}, \texttt{EF21}, and \texttt{DIANA} for $M=3$ clients.}
    \label{fig:3c}
    \includegraphics[scale=.39]{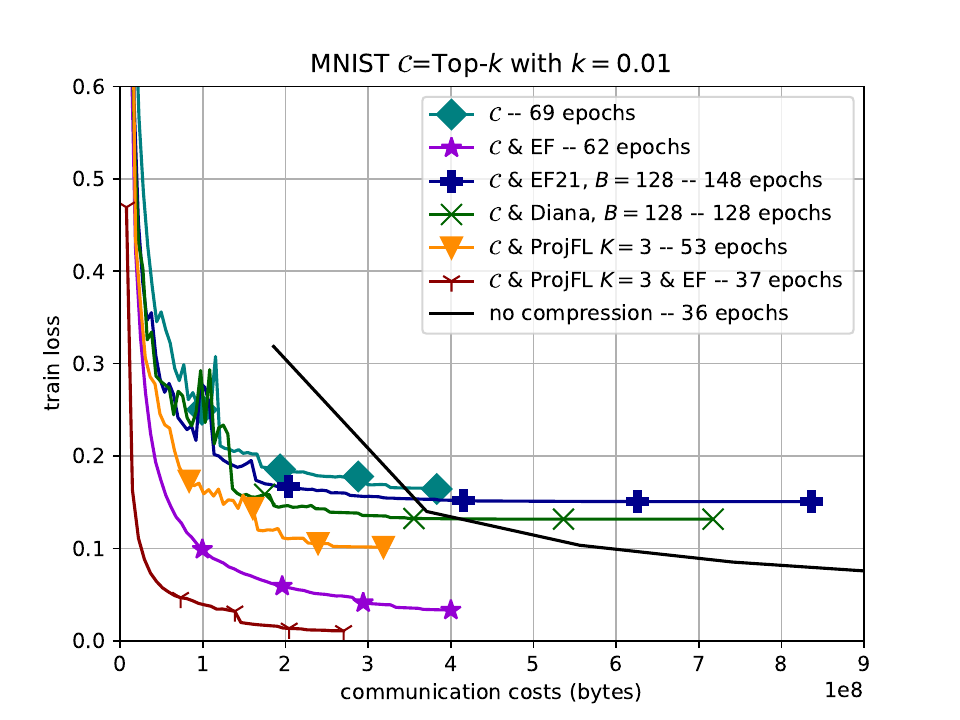} 
    \includegraphics[scale=.39]{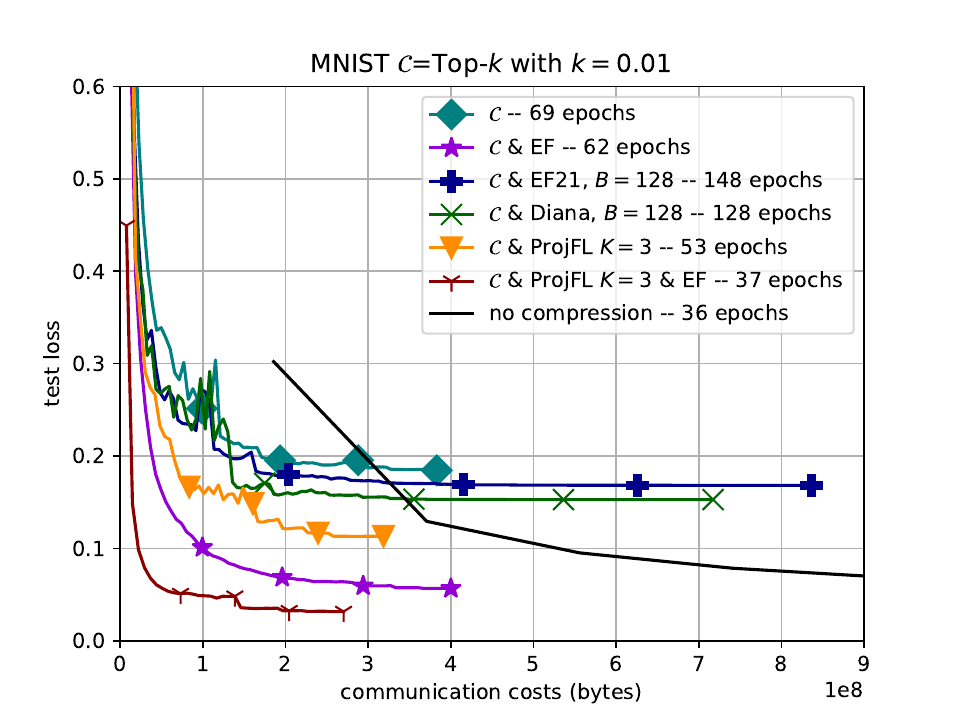} \\
    \includegraphics[scale=.39]{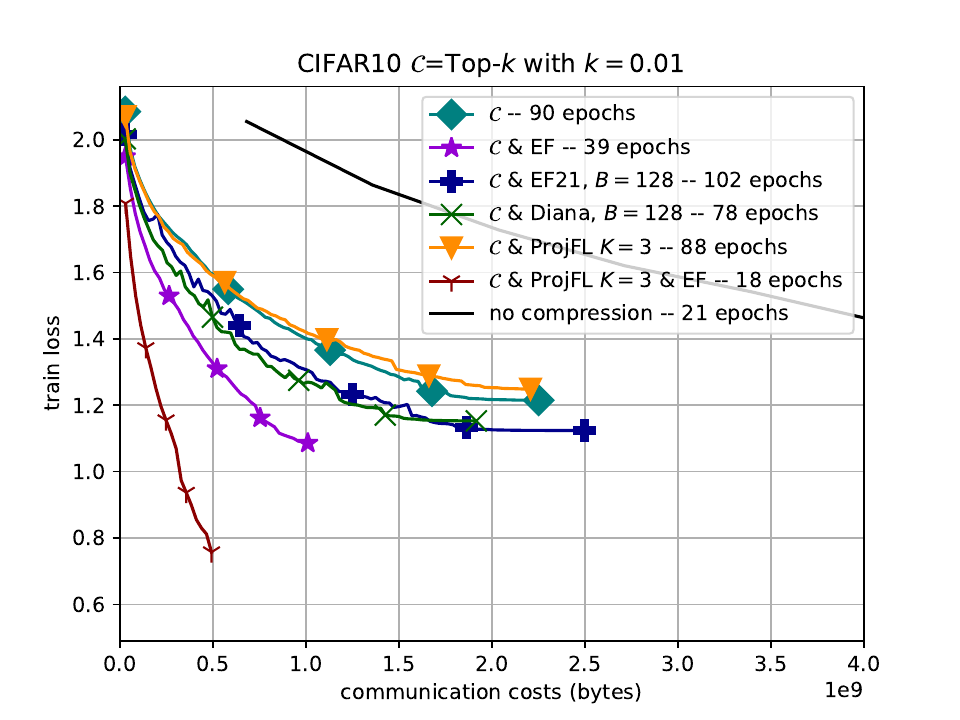} 
    \includegraphics[scale=.39]{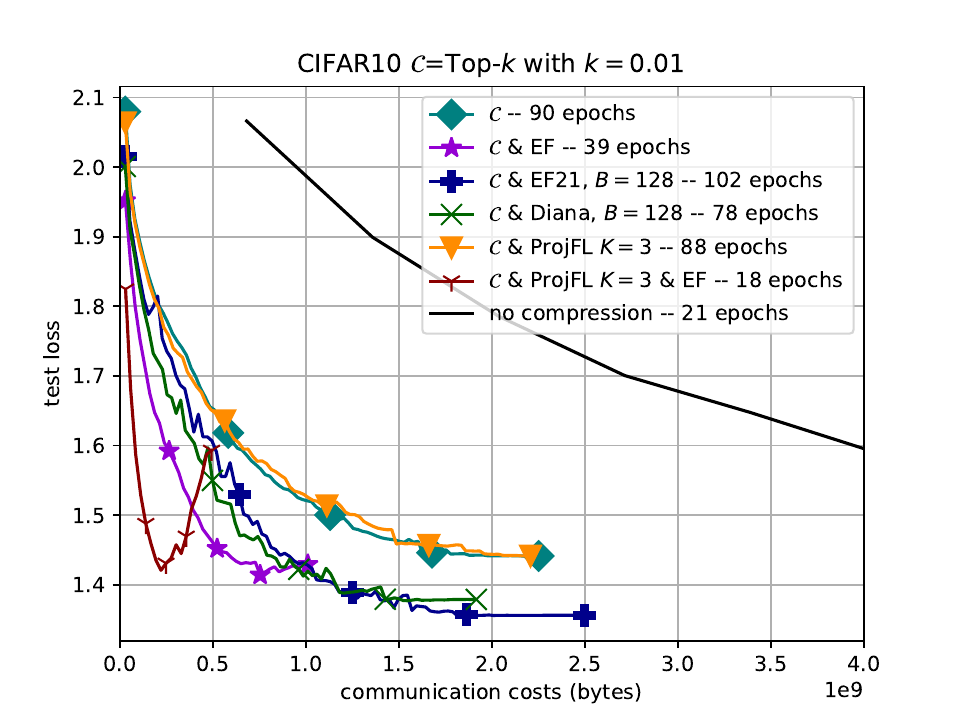} 
\end{figure}

\begin{figure}[H]
    \caption{Comparison of Algorithms \ref{fed-avg-proj-var2} and \ref{fed-avg-proj-EF} with \texttt{FedAvg} with compression, 
    \texttt{EF}, \texttt{EF21}, and \texttt{DIANA} for $M=10$ clients.}
    \label{fig:10c}
    \includegraphics[scale=.39]{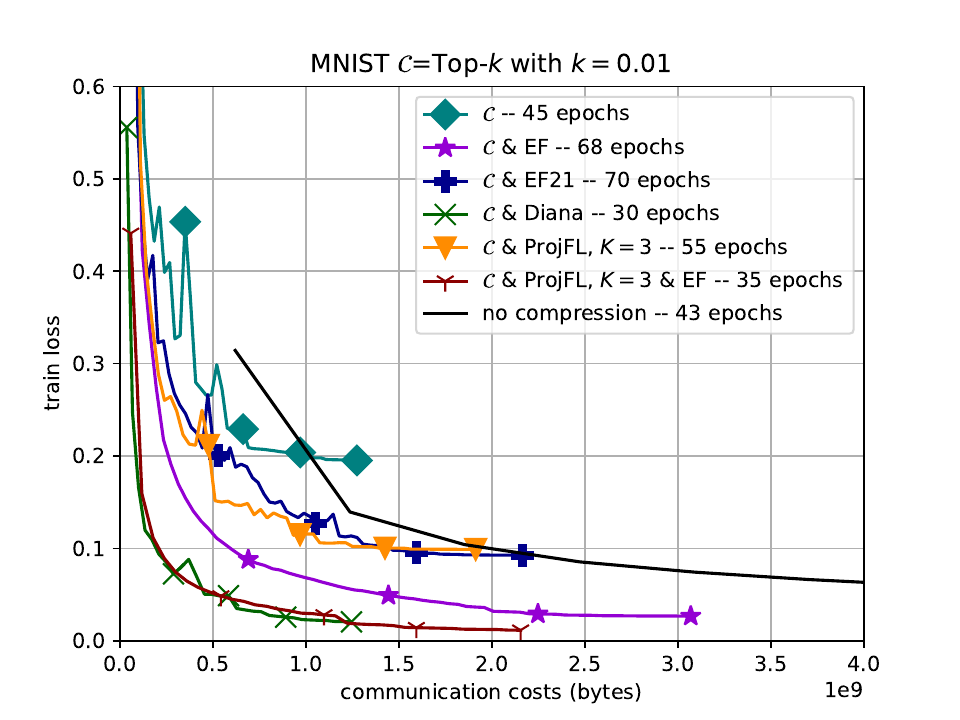} 
    \includegraphics[scale=.39]{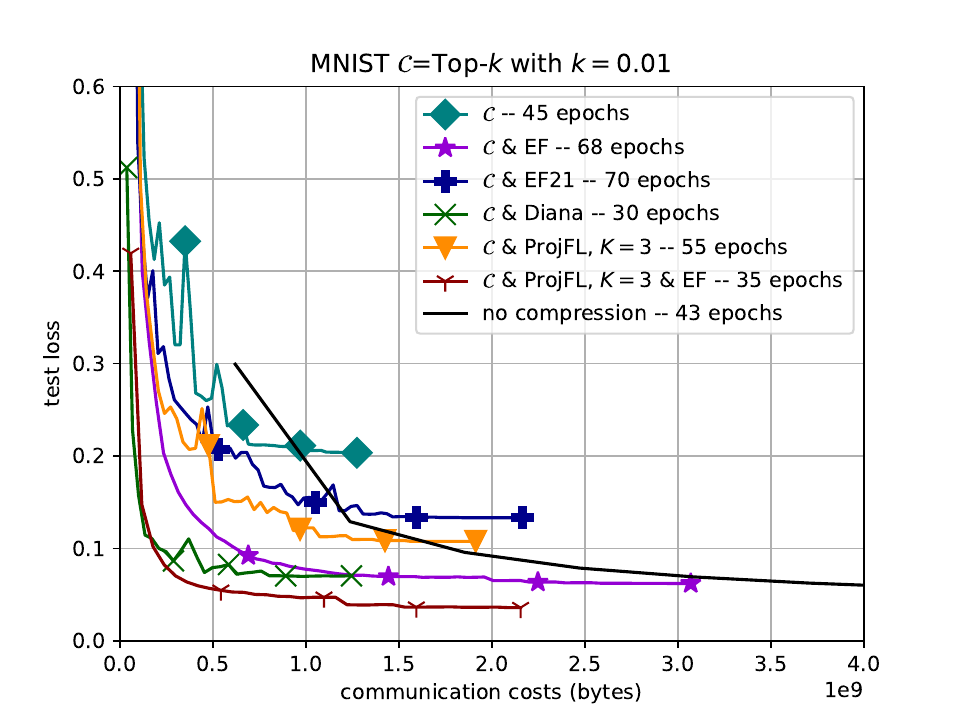} \\
    \includegraphics[scale=.39]{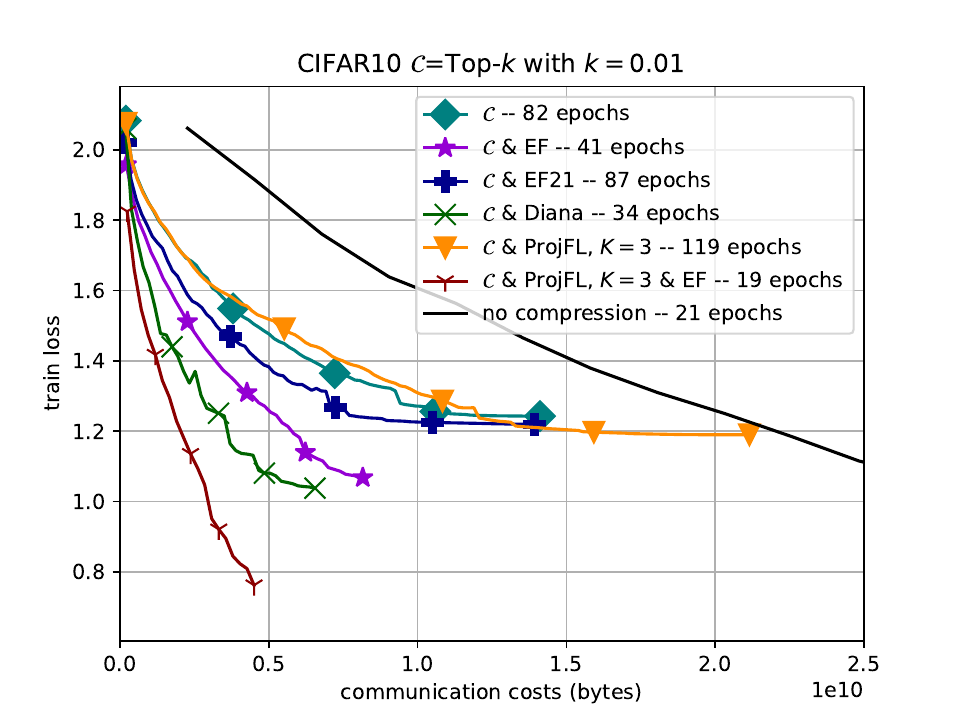} 
    \includegraphics[scale=.39]{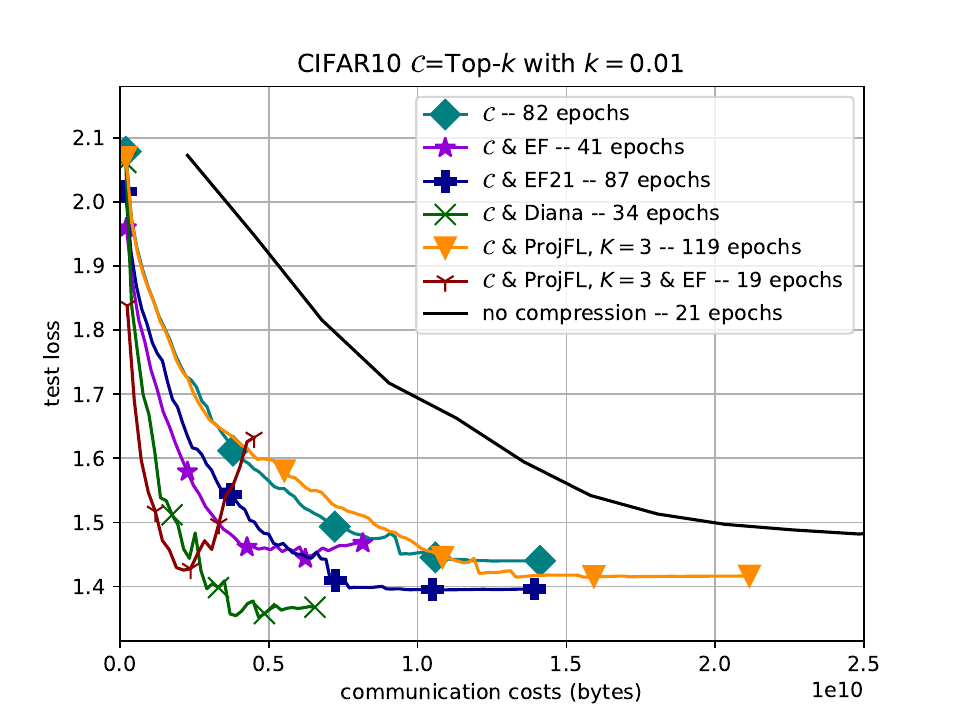} 
\end{figure}

\begin{figure}
    \caption{Comparison of Algorithms \ref{fed-avg-proj-var2} and \ref{fed-avg-proj-EF} with \texttt{FedAvg} with compression, 
    \texttt{EF}, \texttt{EF21}, and \texttt{DIANA} for $M=100$ clients.}
    \label{fig:100c}
    \includegraphics[scale=.39]{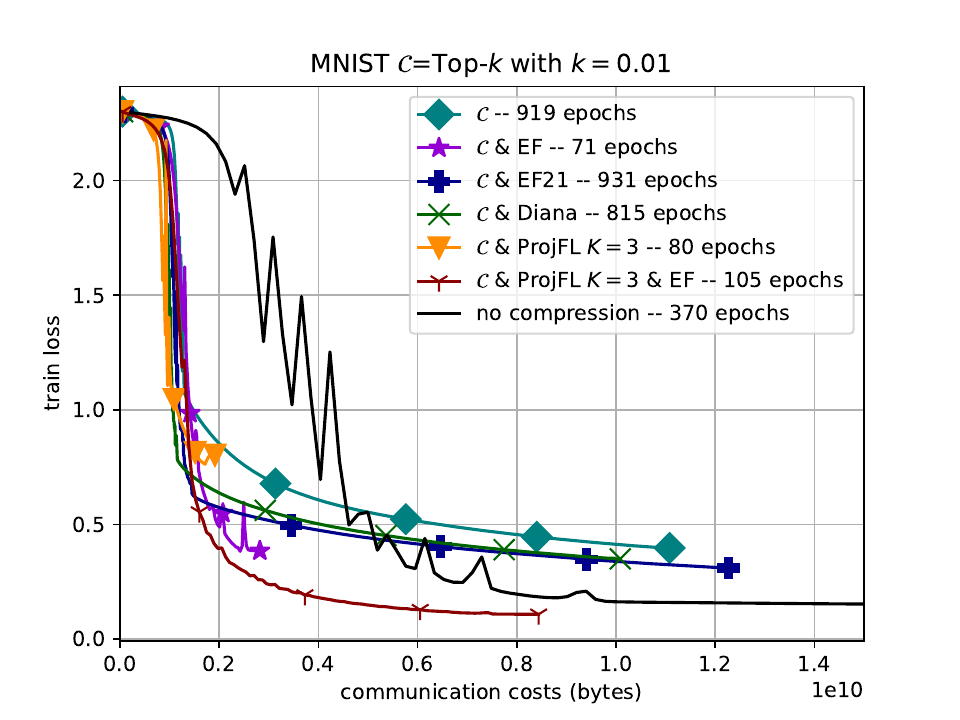} 
    \includegraphics[scale=.39]{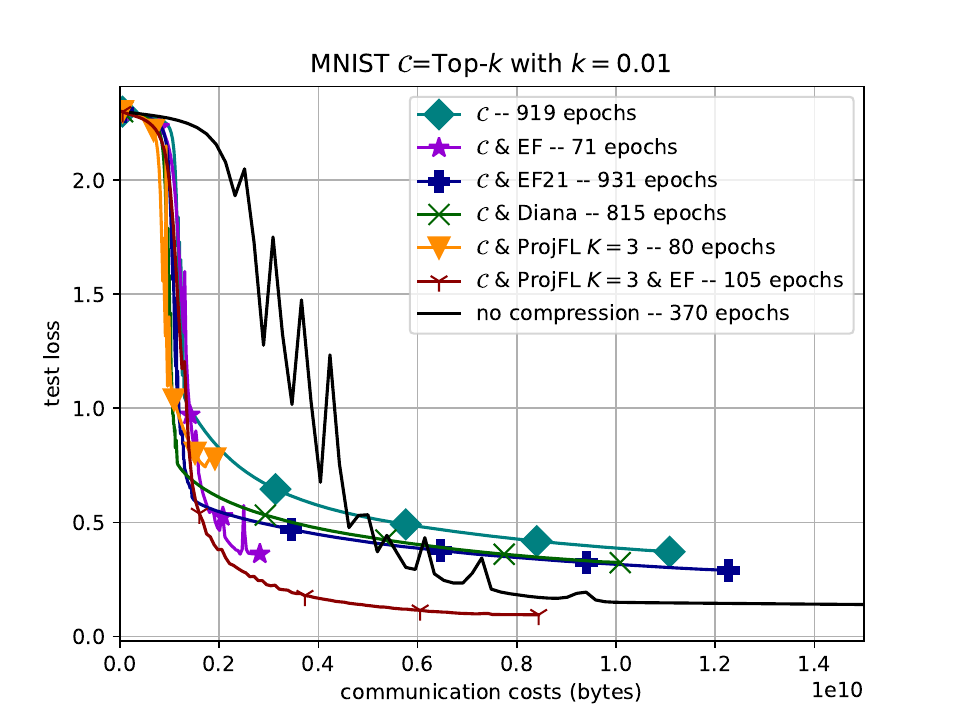} \\
    \includegraphics[scale=.39]{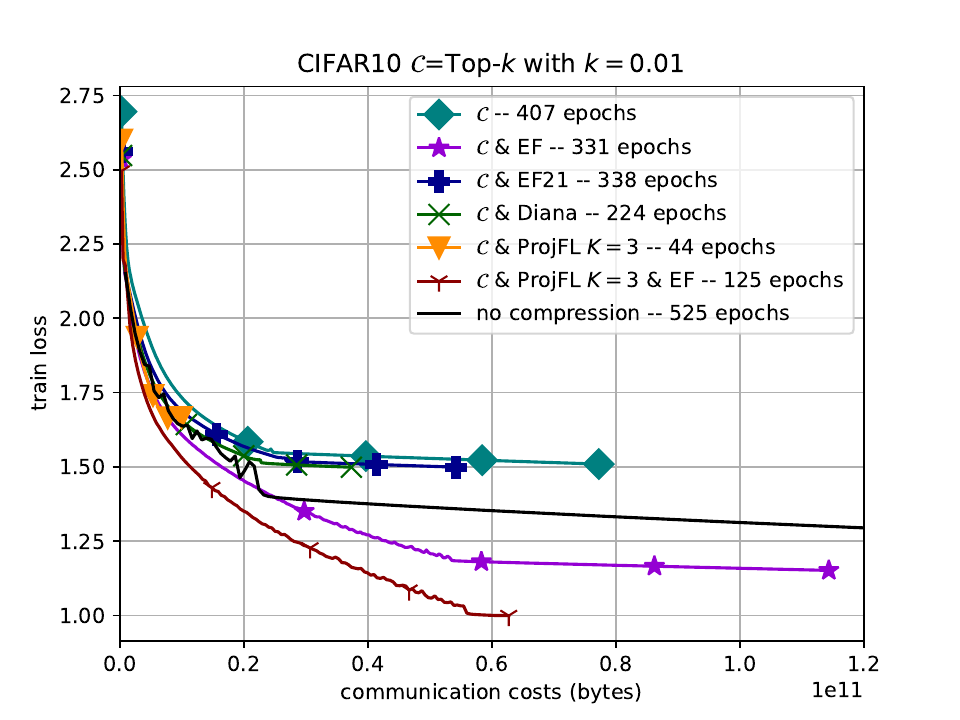} 
    \includegraphics[scale=.39]{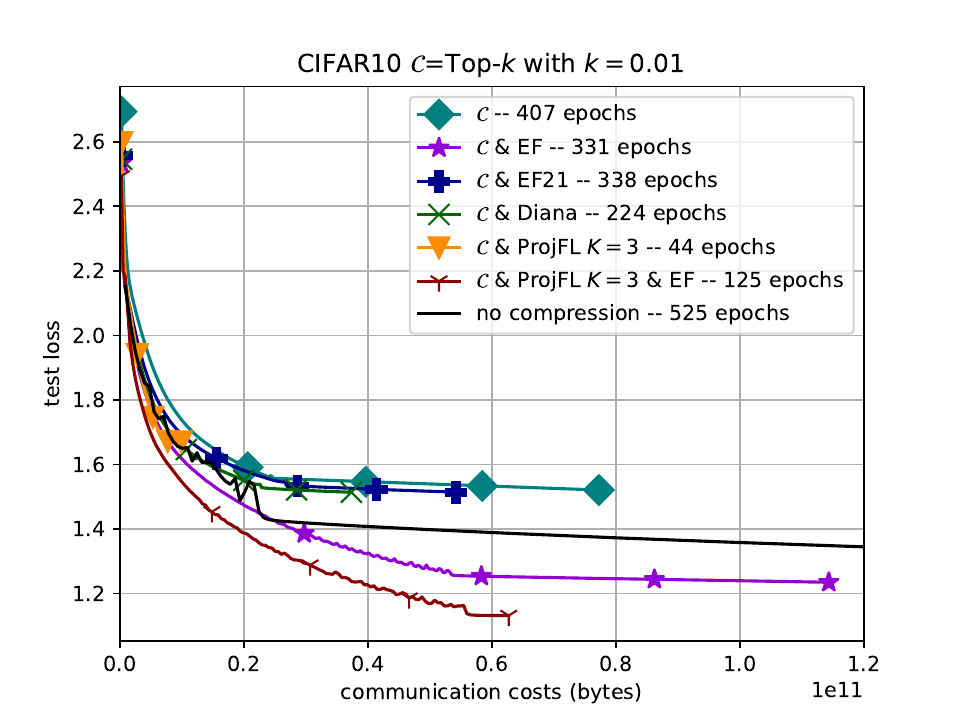} 
\end{figure}

\begin{figure}
    \caption{Comparison of Algorithms \ref{fed-avg-proj-var2} and \ref{fed-avg-proj-EF} with \texttt{FedAvg} with compression, 
    \texttt{EF}, \texttt{EF21}, and \texttt{DIANA} for $M=1000$ clients. We consider here the uplink communication cost only. }
    \label{fig:1000c}
    \includegraphics[scale=.39]{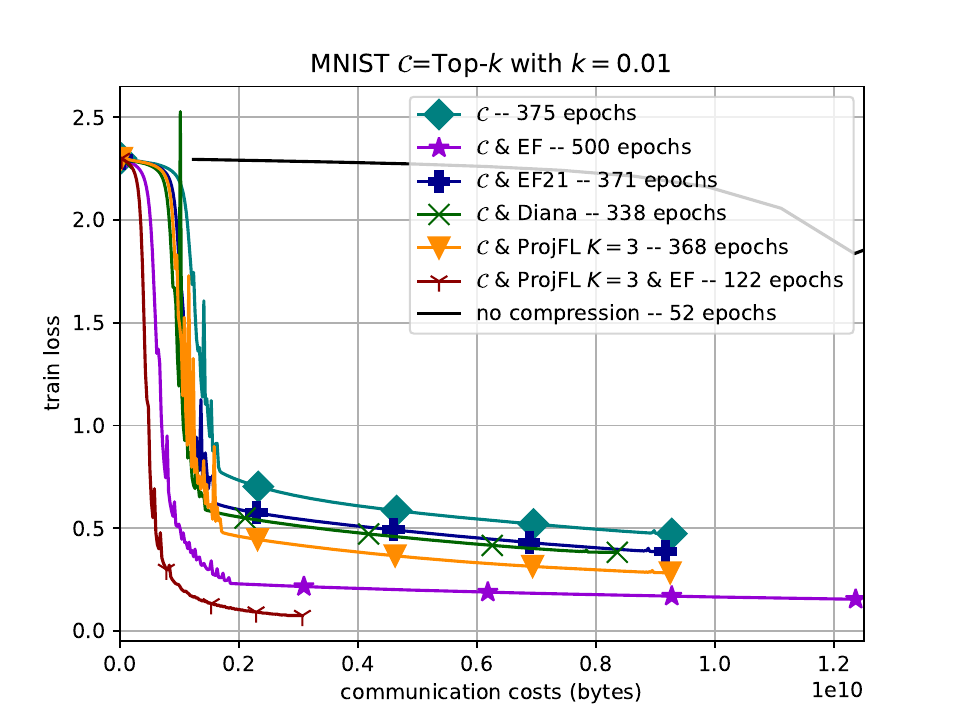} 
    \includegraphics[scale=.39]{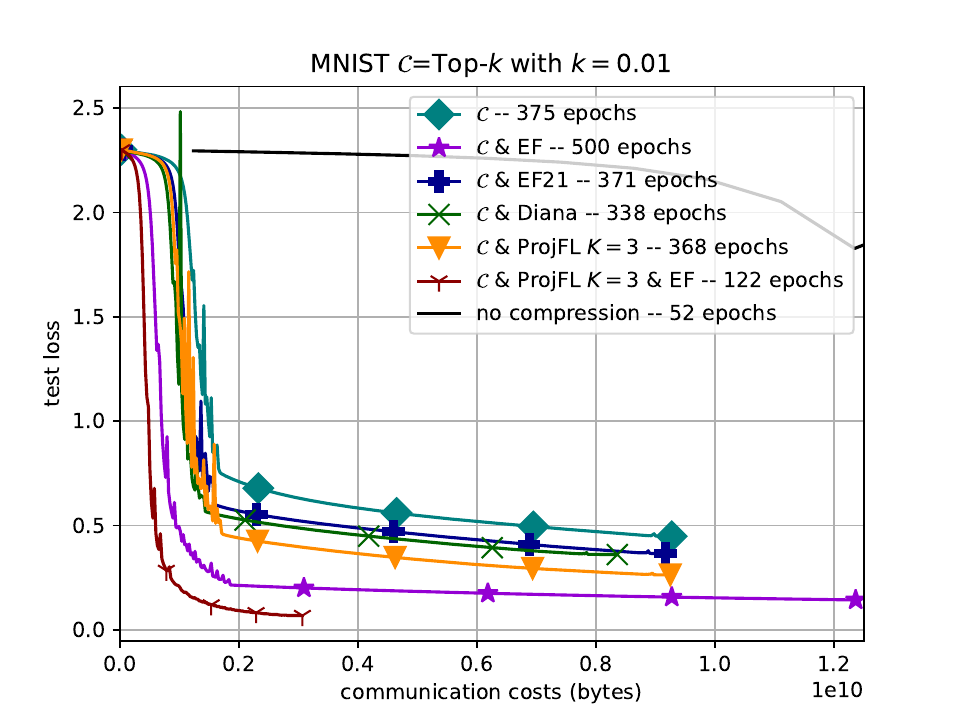} \\
    \includegraphics[scale=.39]{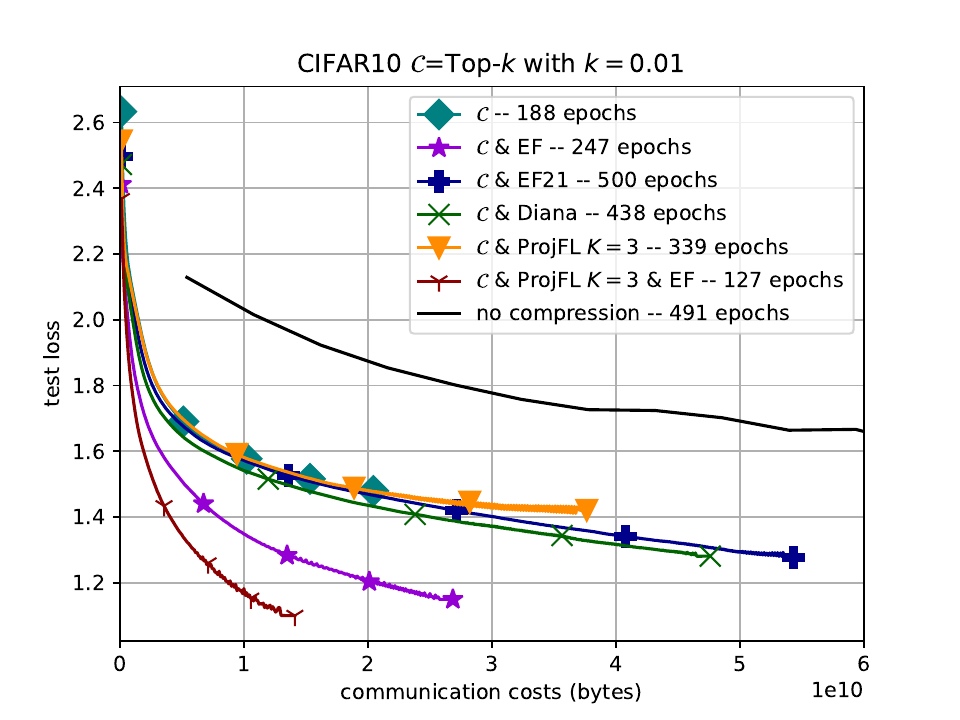} 
    \includegraphics[scale=.39]{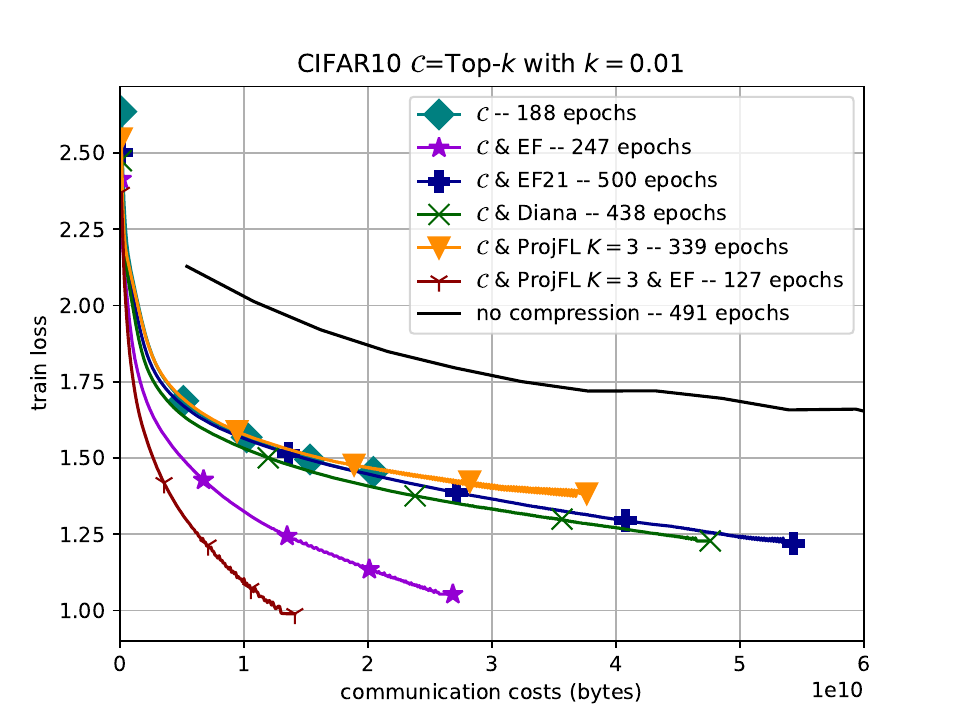} 
\end{figure}

\paragraph{General comments: }
Regardless of the algorithm, the first observation we can make is that as the number of clients increases, both communication costs and losses tend to rise.
The first effect is straightforward: communication cost grows linearly with the number of clients $M$. This is why we report the uplink communication cost for $M=1000$ clients in Figure~\ref{fig:1000c}, rather than the total communication cost.
The increase in losses can be explained by the fact that, in datasets with a finite number of instances—such as MNIST or CIFAR-10—the effective batch size per client decreases as the number of clients increases, resulting in noisier stochastic gradients.
This reduction amplifies the variance of the gradient estimates, an effect that is further exacerbated by compression. These observations are consistent with the noise terms appearing in the convergence rates of all the considered algorithms (including ours; see the right-hand side of the bounds in Theorems~\ref{thm:conv-sgd} and~\ref{thm:EF_strong_conv}).  

\paragraph{On the MNIST dataset:}
Algorithm~\ref{fed-avg-proj-var2} outperforms both \texttt{EF21} and \texttt{DIANA} when using 3 or 1000 clients.
Its performance with 10 clients remains competitive, albeit slightly below the best-performing methods.

Across these experiments, \texttt{EF21} and \texttt{DIANA} exhibit very similar behavior overall, although \texttt{DIANA} performs noticeably better in the 10-client configuration.
When error feedback is incorporated, the \texttt{EF} algorithm demonstrates consistently competitive results regardless of the number of clients.

Importantly, Algorithm~\ref{fed-avg-proj-EF} achieves the best overall trade-off between accuracy and communication efficiency in both training and testing.
For comparable test accuracy, it requires up to 8× less communication than \texttt{EF} in Figure~\ref{fig:3c} and 6× less in Figure~\ref{fig:10c} than \texttt{DIANA}.

\paragraph{On the CIFAR-10 dataset:}
\texttt{ProjFL} delivers slightly better performance than \texttt{Top-$k$} across all experiments, except for the 100-client setting, where it attains the lowest accuracy among the compared methods---consistent with the trend observed in the MNIST analysis.

The observations made for \texttt{EF21} and \texttt{DIANA} on the MNIST dataset also hold for CIFAR-10: both algorithms display very similar overall behavior, though \texttt{DIANA} performs noticeably better in the 10-client configuration.

The \texttt{EF} algorithm continues to offer competitive performance in terms of both accuracy and communication cost, except when $M=10$.
Once again, Algorithm~\ref{fed-avg-proj-EF} demonstrates a clear advantage, achieving strong results with significantly fewer training epochs, thereby reducing local computation time. 

One might be surprised by the increase in test loss in Figure~\ref{fig:3c} and Figure~\ref{fig:10c} at the end of training for Algorithm~\ref{fed-avg-proj-EF}. This rise is due to overfitting, as the training loss continues to decrease monotonically (see Appendix~\ref{sec:appx_num_exp}). The reason the algorithm was not stopped earlier lies in the early stopping criterion (specifically the patience), which was tuned based on the behavior of the baseline \texttt{FedAvg}, as previously discussed in \ref{subsubsec:hyperparam}. Naturally, if we had optimized early stopping specifically for Algorithm~\ref{fed-avg-proj-EF}, more favorable stopping conditions would have been chosen.

\section{Conclusion and Discussion\label{sec:conclusion}}

This work contributes to the growing body of research on communication-efficient FL by proposing two theoretically grounded and practically effective algorithms. 
Our algorithms provably improve convergence speed while requiring the transmission of only one additional scalar per iteration. 

We established convergence guarantees for a variety of smooth objectives, ranging from strongly convex to non-convex settings. Furthermore, we conducted extensive experiments on large-scale neural networks to evaluate the empirical performance of our methods. 

We conclude by outlining three promising directions for future work.  
First, our approach could be combined with acceleration techniques, such as those proposed in \cite{pmlr-v119-li20g,NEURIPS2023_9602d22a}, to further improve convergence.  
Second, instead of projecting onto the one-dimensional subspace spanned by the average of the last $K$ descent directions, one could consider projections onto the full $K$-dimensional subspace generated by these directions.  
Third, for a large number of clients, one could apply bidirectional compression, as done in \cite{pmlr-v97-tang19d}.

\section*{Acknowledgements }
This work has been partly supported by the FLUTE project, EC grant 101095382, and French State support under the France 2030 program with the reference ANR-23-PEIA-005 (REDEEM project). This work was partially carried out while A.D. was a postdoctoral researcher at Inria Lille – Nord Europe.

\newpage
\bibliography{biblioFL}

\newpage

\begin{appendices}

\tableofcontents

This appendix is organized as follows.
In Section~\ref{sec:proof_conv_sgd}, we present the proof of Theorem~\ref{thm:conv-sgd}.
In Section~\ref{appx:proof_thm2}, we prove Theorem~\ref{thm:EF_strong_conv}.
In Section~\ref{sec:appx_num_exp}, we provide implementation details and additional numerical evaluations.

\section{Proof of Theorem~\ref{thm:conv-sgd}\label{sec:proof_conv_sgd}}

This section is devoted to the proof of Theorem~\ref{thm:conv-sgd}.
We first recall the following property satisfied $\mu$-strongly convex and $L$-smooth functions. 
\begin{proposition}[{\cite[Theorem 2.1.12]{nesterov2018lectures}}]\label{prop:nestero_mu_strong_conv}
For any $\mu$-strongly convex and $L$-smooth function $f:\R^d\to \R$  (see Definitions~\ref{def:mu-strong_conv} and~\ref{def:L-smooth}), it holds, for any $x,y\in\R^d$, 
$$\langle\nabla f(x)-\nabla f(y),x-y\rangle\ge \frac{\mu L}{\mu +L}\|x-y\|^2+\frac{1}{\mu+L}\|\nabla f(x)-\nabla f(y)\|^2.$$
\end{proposition}

\subsection{Proof of Theorem~\ref{thm:conv-sgd}: convex case}

\begin{proof}[{Proof of item~\ref{thm:noEF:itemSconv} of Theorem~\ref{thm:conv-sgd}}]
Let $t\in\mathbb N$ and $w^*=\argmin f$ be the unique minimizer (by strong convexity).  
We have 
\begin{align*}
\|w_{t+1}-w^* \|^2 = \|w_t-w^*\|^2 - 2\eta\Big\langle w_t-w^*,\frac1M\sum_{i=1}^M\mathsf D_{t+1}^i\Big\rangle + \eta^2\Big\|\frac1M\sum_{i=1}^M\mathsf D_{t+1}^i\Big\|^2.
\end{align*}
By Assumptions~\ref{as:compressor} and~\ref{as:stoc_gradient}, we have
\[
\mathbb{E}[\mathcal{C}((g_{t+1}^i)^\perp)] = \mathbb{E}[(g_{t+1}^i)^\perp].
\]
Therefore, using Assumption~\ref{as:stoc_gradient} again, it follows that
\[
\mathbb{E}[\mathsf{D}_{t+1}^i] = \mathbb{E}[g_{t+1}^i] = \mathbb{E}[\nabla f_i(w_t)].
\]
Conditioning on $w_t$, we obtain, using~\ref{as:stoc_gradient},
\begin{align*}
\mathbb E\left[\|w_{t+1}-w^* \|^2 \right]= \mathbb E\left[\|w_t-w^*\|^2\right] - 2\eta\mathbb E\left[\langle w_t-w^*,\nabla f(w_t)\rangle\right] + \eta^2\mathbb E\Big[\Big\|\frac1M\sum_{i=1}^M\mathsf D_{t+1}^i\Big\|^2\Big].
\end{align*} 
Using Proposition~\ref{prop:nestero_mu_strong_conv} (since $\nabla f(w^*)=0$), it becomes 
\begin{align}\label{eq_thmSGD:eq1}
\mathbb E\left[\|w_{t+1}-w^* \|^2 \right]&\le \left(1-2\eta\frac{\mu L}{\mu+L}\right)\mathbb E\left[\|w_t-w^*\|^2\right] - 2\frac{\eta}{\mu+L}\mathbb E\left[\|\nabla f(w_t)\|\right]\nonumber\\
&\quad + \eta^2\mathbb E\Big[\Big\|\frac1M\sum_{i=1}^M\mathsf D_{t+1}^i\Big\|^2\Big].
\end{align} 
Using again~\ref{as:stoc_gradient} and the fact that $\mathbb E[\mathsf D_{t+1}^i] = \mathbb E[\nabla f_i(w_t)]$, we have 
\begin{align}\label{eq_thmSGD:last_term}
\mathbb E\Big[\Big\|\frac 1M\sum_{i=1}^M \mathsf D_{t+1}^i \Big\|^2\Big] &= \mathbb E[\|\nabla f(w_t)\|^2] + \mathbb E\Big[\Big\|\frac 1M\sum_{i=1}^M\mathsf D_{t+1}^i-\nabla f_i(w_t) \Big\|^2\Big] \nonumber\\
&=  \mathbb E[\|\nabla f(w_t)\|^2] + \frac{1}{M^2}\sum_{i=1}^M\mathbb E[\|\mathsf D_{t+1}^i-\nabla f_i(w_t) \|^2].
\end{align}
Using~\ref{as:compressor} and~\ref{as:stoc_gradient},
\begin{align*}
\mathbb E[\|\mathsf D_{t+1}^i-\nabla f_i(w_t) \|^2]
&= \mathbb E[\|\mathsf D_{t+1}^i\|^2] - \mathbb E[ \|\nabla f_i(w_t) \|^2]\\
&=\mathbb E[\|\alpha_{t+1}^i\bar{\mathsf D}_t^i+\mathcal C((g_{t+1}^i)^\perp)\|^2] - \mathbb E[\|\nabla f_i(w_t)\|^2] \\
&=\mathbb E[\|\alpha_{t+1}^i\bar{\mathsf D}_t^i\|^2]+\mathbb E[\|\mathcal C((g_{t+1}^i)^\perp)\|^2] - \mathbb E[\|\nabla f_i(w_t)\|^2] \\
&\le \mathbb E[\|\alpha_{t+1}^i\bar{\mathsf D}_t^i\|^2]+\beta\mathbb E[\|(g_{t+1}^i)^\perp\|^2] - \mathbb E[\|\nabla f_i(w_t)\|^2]\\
&\le \beta\mathbb E[\|\alpha_{t+1}^i\bar{\mathsf D}_t^i\|^2]+\beta\mathbb E[\|(g_{t+1}^i)^\perp\|^2] - \mathbb E[\|\nabla f_i(w_t)\|^2] \quad (\beta\ge 1)\\
&=\beta \mathbb E[\|g_{t+1}^i\|^2] - \mathbb E[\|\nabla f_i(w_t)\|^2]\le (\beta-1) \mathbb E[\|\nabla f_i(w_t)\|^2]+\beta\sigma^2.
\end{align*}
Going back to~\eqref{eq_thmSGD:last_term}, we have  
\begin{align*}
\mathbb E\Big[\Big\|\frac 1M\sum_{i=1}^M \mathsf D_{t+1}^i \Big\|^2\Big] \le  \mathbb E[\|\nabla f(w_t)\|^2] + \frac{\beta-1}{M^2}\sum_{i=1}^M\mathbb E[\|\nabla f_i(w_t)\|^2] + \frac{\beta\sigma^2}{M}.
\end{align*}
By~\ref{as:bounded_g_diss}, 
\begin{equation}\label{eq:bound_E|D|}
\mathbb E\Big[\Big\|\frac 1M\sum_{i=1}^M \mathsf D_{t+1}^i \Big\|^2\Big]\le (1+b\frac{\beta-1}{M})\mathbb E[\|\nabla f(w_t)\|^2] + a\frac{\beta-1}{M}+ \frac{\beta\sigma^2}{M}.
\end{equation}
Going back to \eqref{eq_thmSGD:eq1} we have
\begin{align*}
\mathbb E[\|w_{t+1}-w^*\|^2] &\le \left(1-2\eta\frac{\mu L}{\mu+L}\right)\mathbb E[\|w_{t}-w^*\|^2]\\
&\quad +\eta\Big(\eta\Big(1+b\frac{\beta-1}{M}\Big)-\frac{2}{\mu+L}\Big)  \mathbb E[\|\nabla f(w_t)\|^2]\\
&\quad + \eta^2 a\frac{\beta-1}{M}+ \eta^2\frac{\beta\sigma^2}{M}.
\end{align*}
Since $\eta\Big(1+b\frac{\beta-1}{M}\Big)\le\frac{2}{\mu+L}$, we obtain 
\begin{align*}
\mathbb E[\|w_{t+1}-w^*\|^2] &\le \left(1-2\eta\frac{\mu L}{\mu+L}\right)\mathbb E[\|w_{t}-w^*\|^2]+ \eta^2 a\frac{\beta-1}{M}+ \eta^2\frac{\beta\sigma^2}{M}.
\end{align*}
Hence, noticing\footnote{Note that our assumption $\eta\le (1+b\frac{\beta-1}{M})^{-1}\frac{2}{\mu+L}$ implies $\eta<\frac{\mu+L}{2\mu L}$ as soon as $b(\beta-1)>0$ or $L>\mu$.} that $\eta< \frac{\mu+L}{2\mu L}$, we have  
$$\mathbb E[\|w_{t}-w^*\|^2] \le \Big(1-2\eta\frac{\mu L}{\mu+L}\Big)^t\mathbb E[\|w_{0}-w^*\|^2] + \eta\frac{\mu+L}{2\mu LM}(a(\beta-1)+\beta\sigma^2). $$
The proof is complete.
\end{proof}

\begin{proof}[Proof of item~\ref{thm:noEF:itemconv} of Theorem~\ref{thm:conv-sgd}]
Let $t\in\mathbb N$. 
Denote by $w^*$ any optimal point of $f$. We have 
\begin{equation*}
\|w_{t+1}-w^* \|^2 = \|w_t-w^*\|^2-2\eta\Big\langle w_t-w^*, \frac{1}{M}\sum_{i=1}^M\mathsf D_{t+1}^i\Big\rangle+\eta^2\Big\|\frac{1}{M}\sum_{i=1}^M\mathsf D_{t+1}^i\Big\|^2. 
\end{equation*}
By \eqref{eq:bound_E|D|},~\ref{as:compressor} and~\ref{as:stoc_gradient}, we have, 
\begin{align*}
\mathbb E[\|w_{t+1}-w^* \|^2] &\le\mathbb E[ \|w_t-w^*\|^2]-2\eta\mathbb E\Big[\Big\langle w_t-w^*,\nabla f(w_t)\Big\rangle\Big]\\ &\quad +\eta^2(1+b\frac{\beta-1}{M})\mathbb E[\|\nabla f(w_t)\|^2] + \frac{\eta^2}{M}(a(\beta-1)+\beta\sigma^2)
\end{align*} 
Since $f$ is convex, we have
$$-2\langle w_t-w^*,\nabla f(w_t)\rangle\le -2(f(w_t)-f^*).$$
Hence, using also Proposition~\ref{prop:ineg-L-smooth:normgrad}, 
\begin{align*}
\mathbb E[\|w_{t+1}-w^* \|^2] &\le\mathbb E[ \|w_t-w^*\|^2]-2\eta\mathbb E[f(w_t)-f^*] \\
&\quad +2L\eta^2(1+b\frac{\beta-1}{M})\mathbb E[f(w_t)-f^*] + \frac{\eta^2}{M}(a(\beta-1)+\beta\sigma^2)\\
&= \mathbb E[ \|w_t-w^*\|^2]+2\eta(L\eta(1+b\frac{\beta-1}{M})-1)\mathbb E[f(w_t)-f^*] + \frac{\eta^2}{M}(a(\beta-1)+\beta\sigma^2).
\end{align*}
Using that $L\eta(1+b\frac{\beta-1}{M})\le 1/2$, we obtain 
\begin{align*}
\mathbb E[\|w_{t+1}-w^* \|^2]\le \mathbb E[ \|w_t-w^*\|^2]-\eta\mathbb E[f(w_t)-f^*] +2 + \frac{\eta^2}{M}(a(\beta-1)+\beta\sigma^2). 
\end{align*}
Let $T\ge 1$. Summing over $t$, we have, by telescopic sum, 
\begin{align*}
\frac{\eta}{T+1}\sum_{t=0}^{T}\mathbb E[f(w_t)-f^*] \le \frac{1}{T+1}\mathbb E[ \|w_0-w^*\|^2] + \frac{\eta^2}{M}(a(\beta-1)+\beta\sigma^2). 
\end{align*}
The proof is complete since $\mathbb E[f(w^{\mathrm{out}})]-f^* = \frac{1}{T+1}\sum_{t=0}^{T}\mathbb E[f(w_t)-f^*]$.
\end{proof}

\subsection{Proof of Theorem~\ref{thm:conv-sgd}: non-convex case}

\begin{proof}[Proof of item~\ref{thm:noEF:itemnoconv} of Theorem~\ref{thm:conv-sgd}]
Let $t\in\mathbb N$. 
Since $f$ is $L$-smooth, we have, by \eqref{prop-2Lf}, 
$$f(w_{t+1})\le f(w_t) + \langle\nabla f(w_t),w_{t+1}-w_t\rangle + \frac L2\|w_{t+1}-w_t\|^2.$$ 
By~\ref{as:compressor} and~\ref{as:stoc_gradient}, $\mathbb E[\mathsf D_{t+1}^i|w_t] =\nabla f_i(w_t)$ for all $i\in[M]$. Hence, taking the expectation we obtain 
$$\mathbb E[f(w_{t+1})]\le\mathbb E[ f(w_t)] -\eta\mathbb E[ \|\nabla f(w_t)\|^2] + \frac {L\eta^2}{2}\mathbb E\Big[\Big\|\frac{1}{M}\sum_{i=1}^M\mathsf D_{t+1}^i\Big\|^2\Big]$$
By \eqref{eq:bound_E|D|},
$$\mathbb E[f(w_{t+1})]\le\mathbb E[ f(w_t)] -\eta\mathbb E[ \|\nabla f(w_t)\|^2] + \frac {L\eta^2}{2}\Big(1+b\frac{\beta-1}{M}\Big)\mathbb E[ \|\nabla f(w_t)\|^2]+\frac {L\eta^2}{2M}(a(\beta-1)+\beta\sigma^2)$$
Since $L\eta(1+b\frac{\beta-1}{M})\le 1$, we obtain 
$$ \frac\eta 2\mathbb E[ \|\nabla f(w_t)\|^2]\le\mathbb E[f(w_{t})-f^*]-\mathbb E[f(w_{t+1})-f^*]+\frac {L\eta^2}{2M}(a(\beta-1)+\beta\sigma^2) $$
Let $T\ge 1$. Summing over $t$,
$$  \frac\eta 2\sum_{t=0}^T\mathbb E[ \|\nabla f(w_t)\|^2]\le\mathbb E[f(w_0)-f^*]+(T+1)\frac {L\eta^2}{2M}(a(\beta-1)+\beta\sigma^2) $$
Hence, 
$$  \frac{1}{T+1}\sum_{t=0}^T\mathbb E[ \|\nabla f(w_t)\|^2]\le\frac{2}{(T+1)\eta}\mathbb E[f(w_0)-f^*]+\frac {L\eta}{M}(a(\beta-1)+\beta\sigma^2),  $$
which concludes the proof since $\mathbb E[ \|\nabla f(w^{\mathrm{out}})\|^2]=\frac{1}{T+1}\sum_{t=0}^T\mathbb E[ \|\nabla f(w_t)\|^2]$. 
\end{proof}

\section{Proof of Theorem~\ref{thm:EF_strong_conv}\label{appx:proof_thm2}}

We start this section with following lemma, which controls the second moment of the compression error. 

\begin{lemma}\label{lem:bound_e_t_projFLEF}
Assume~\ref{as:compressorB}-\ref{as:bounded_g_diss}-\ref{as:stoc_gradient} and that each $f_i$ is differentiable. 
Then, for all $t\in \mathbb N$ and $\eta>0$,   
\begin{equation}\label{proof-pEF:lem:eq0}
\mathbb E[\|e_{t+1}\|^2]\le \frac{2(1-\delta)b\eta^2}{\delta}\sum_{s=0}^t\left(1-\frac{\delta}{2}\right)^{t-s}\mathbb E[\|\nabla f(w_s)\|^2] + \frac{2(1-\delta)\eta^2}{\delta}\left(\frac{2a}{\delta}+\sigma^2\right).
\end{equation}
Moreover, for any sequence $(\theta_t)_{t\ge 0}$ satisfying $0<\theta_{t}\le \theta_{t+1}\le (1+\delta/4)\theta_t$ for all $t\in \mathbb N$, it holds, for all $T\ge 1$,  
\begin{equation}\label{proof-PEF:eq0w}
\sum_{t=0}^T\theta_{t}\mathbb E[\|e_{t}\|^2]\le\frac{8b\eta^2}{\delta^2}\sum_{t=0}^{T-1} \theta_t\mathbb E[\|\nabla f(w_t)\|^2] + \frac{2(1-\delta)\eta^2}{\delta}\left(\frac{2a}{\delta}+\sigma^2\right)\sum_{t=1}^T\theta_{t}.
\end{equation}
\end{lemma}

\begin{proof}
Let $t\in\mathbb N$ and $i\in [M]$.  By~\ref{as:compressorB},
\begin{equation}\label{proof-PEF:eq7'}
\mathbb E_{\mathcal C}[\|e_{t+1}^i\|^2]\le(1-\delta)\|\eta(g_{t+1}^i)^\perp+e_t^i\|^2.
\end{equation}
Defining $(\nabla f_i(w_t))^\perp$ and $(\xi_{t+1}^i(w_t))^\perp$  such that $(\nabla f_i(w_t))^\perp\cdot\bar{\mathsf D}_t^i =0 $ and $(\xi_{t+1}^i(w_t))^\perp \cdot\bar{\mathsf D}_t^i = 0$,  we have 
\begin{align}\label{proof-PEF:eq8}
\mathbb E[\|\eta(g_{t+1}^i)^\perp+e_t^i\|^2| w_t] &= \mathbb E[\|\eta(\nabla f_i(w_t))^\perp+\eta(\xi_{t+1}^i(w_t))^\perp+e_t^i\|^2| w_t]\nonumber\\
&= \|\eta(\nabla f_i(w_t))^\perp+e_t^i\|^2 + \eta^2\mathbb E[\|\xi_{t+1}^i(w_t)\|^2|w_t]\nonumber \\
&\le (1+\beta)\|e_t^i\|^2+ \Big(1+\frac{1}{\beta}\Big)\eta^2 \|(\nabla f_i(w_t))^\perp\|^2 +\eta^2\sigma^2,\quad \forall\beta>0,
\end{align}
where, for the last inequality, we used~\ref{as:bounded_g_diss} and the inequality $\|\mathsf a+\mathsf b\|^2\le (1+\beta)\|\mathsf a\|^2+ (1+1/\beta)\|\mathsf b\|^2$, $\mathsf a,\mathsf b\in\mathbb R^d, \beta>0$. Since $\|(\nabla f_i(w_t))^\perp\|\le \|\nabla f_i(w_t)\|$, it follows from \eqref{proof-PEF:eq8} that 
\begin{equation}\label{proof-PEF:eq9}
\mathbb E[\|\eta(g_{t+1}^i)^\perp+e_t^i\|^2]\le (1+\beta)\mathbb E[\|e_t^i\|^2]+ \Big(1+\frac{1}{\beta}\Big)\eta^2 \mathbb E[\|\nabla f_i(w_t)\|^2] +\eta^2\sigma^2,\quad \forall\beta>0.
\end{equation}
By \eqref{proof-PEF:eq7'} and \eqref{proof-PEF:eq9}, we have
\begin{equation}\label{proof-PEF:eq10}
\mathbb E[\|e_{t+1}^i\|^2]\le (1-\delta)(1+\beta)\mathbb E[\|e_t^i\|^2]+ (1-\delta)\Big(1+\frac{1}{\beta}\Big)\eta^2 \mathbb E[\|\nabla f_i(w_t)\|^2] +(1-\delta)\eta^2\sigma^2,\quad \forall\beta>0.
\end{equation}
Let us define the auxiliary sequence $U_t = \frac 1M\sum_{i=1}^M\|e_{t}^i\|^2$. We have, by \eqref{proof-PEF:eq10}, for any $\beta>0$,  
\begin{align*}
\underbrace{\mathbb E[U_{t+1}]}_{\mathsf b_{t+1}} \le \underbrace{(1-\delta)(1+\beta)}_{\mathsf a}\underbrace{\mathbb E[U_t]}_{\mathsf b_t} + \underbrace{(1-\delta)\Big(1+\frac1\beta\Big)\frac{\eta^2}{M}\sum_{i=1}^M\mathbb E[\|\nabla f_i(w_t)\|^2]}_{\mathsf c_t} + \underbrace{(1-\delta)\eta^2\sigma^2}_{\mathsf d}. 
\end{align*}
Since the recursion $\mathsf b_{t+1}\le \mathsf a\mathsf b_t+\mathsf c_t+\mathsf d$ leads to $\mathsf b_{t+1}\le\mathsf a^{t+1}\mathsf b_0+\sum_{s=0}^t\mathsf a^{t-s}\mathsf c_s+\mathsf d\sum_{s=0}^t\mathsf a^s$, we obtain, using also that $U_0=0$,  
\begin{align*}
\mathbb E[U_{t+1}]&\le (1-\delta)\Big(1+\frac1\beta\Big)\frac{\eta^2}{M}\sum_{s=0}^t[(1-\delta)(1+\beta)]^{t-s}\sum_{i=1}^M\mathbb E[\|\nabla f_i(w_s)\|^2] \\
&\quad + (1-\delta)\eta^2\sigma^2\sum_{s=0}^t[(1-\delta)(1+\beta)]^s.
\end{align*}
Consider now $\beta$ such that $1+\frac1\beta\le\frac2\delta$ and $(1-\delta)(1+\beta)\le 1-\frac\delta 2$.
Using also that $\sum_{s=0}^t(1-\delta/2)^s\le\sum_{s=0}^\infty(1-\delta/2)^s= \frac{2}{\delta}$, we have, 
\begin{align*}
\mathbb E[U_{t+1}]\le \frac{2(1-\delta)\eta^2}{\delta M}\sum_{s=0}^t\left(1-\frac{\delta}{2}\right)^{t-s}\sum_{i=1}^M\mathbb E[\|\nabla f_i(w_s)\|^2] + \frac{2(1-\delta)\eta^2\sigma^2}{\delta}.
\end{align*}
Using~\ref{as:bounded_g_diss}, we obtain 
\begin{align*}
\mathbb E[U_{t+1}]\le \frac{2(1-\delta)b\eta^2}{\delta}\sum_{s=0}^t\left(1-\frac{\delta}{2}\right)^{t-s}\mathbb E[\|\nabla f(w_s)\|^2] + \frac{4(1-\delta)a\eta^2}{\delta^2}  + \frac{2(1-\delta)\eta^2\sigma^2}{\delta}.
\end{align*}
Hence, using Jensen's inequality,
\begin{equation*}
\mathbb E[\|e_{t+1}\|^2]\le \mathbb E[U_{t+1}]\le \frac{2(1-\delta)b\eta^2}{\delta}\sum_{s=0}^t\left(1-\frac{\delta}{2}\right)^{t-s}\mathbb E[\|\nabla f(w_s)\|^2] + \frac{2(1-\delta)\eta^2}{\delta}\left(\frac{2a}{\delta}+\sigma^2\right).
\end{equation*}
This proves \eqref{proof-pEF:lem:eq0}. 
Now, let $(\theta_t)_{t\ge 0}$ be as in the statement of the lemma. 
Let $t\in \mathbb N$. We have, since $\theta_{t+1}\le (1+\frac{\delta}{4})^{t+1-s}\theta_s$ for any $s\in \{0,\dots,t+1\}$, 
\begin{align*}
\theta_{t+1}\mathbb E[\|e_{t+1}\|^2]&\le \frac{2(1-\delta)b\eta^2}{\delta}\Big(1+\frac{\delta}{4}\Big)\sum_{s=0}^t\left[\left(1+\frac{\delta}{4}\right)\left(1-\frac{\delta}{2}\right)\right]^{t-s}\theta_s\mathbb E[\|\nabla f(w_s)\|^2] + \\
&\quad \frac{2(1-\delta)\eta^2}{\delta}\left(\frac{2a}{\delta}+\sigma^2\right)\theta_{t+1}.
\end{align*}
Since $(1-\delta)(1+\delta/4)\le 1$ and $(1+\delta/4)(1-\delta/2)\le 1-\delta/4$, 
\begin{equation*}
\theta_{t+1}\mathbb E[\|e_{t+1}\|^2]\le \frac{2b\eta^2}{\delta}\sum_{s=0}^t\left(1-\frac{\delta}{4}\right)^{t-s}\theta_s\mathbb E[\|\nabla f(w_s)\|^2] + \frac{2(1-\delta)\eta^2}{\delta}\left(\frac{2a}{\delta}+\sigma^2\right)\theta_{t+1}.
\end{equation*}
Let $T\ge 1$.
Summing over $t$ and using that $e_0=0$, we obtain 
\begin{equation*}
\sum_{t=0}^T\theta_{t}\mathbb E[\|e_{t}\|^2]\le\frac{2b\eta^2}{\delta}\sum_{s=0}^{T-1} \theta_s\mathbb E[\|\nabla f(w_s)\|^2] \sum_{t=s}^{T-1}\left(1-\frac{\delta}{4}\right)^{t-s} + \frac{2(1-\delta)\eta^2}{\delta}\left(\frac{2a}{\delta}+\sigma^2\right)\sum_{t=1}^T\theta_{t}.
\end{equation*}
Since $\sum_{t=0}^\infty(1-\delta/4)^t= 4/\delta$, we obtain 
\begin{equation*}
\sum_{t=0}^T\theta_{t}\mathbb E[\|e_{t}\|^2]\le\frac{8b\eta^2}{\delta^2}\sum_{t=0}^{T-1} \theta_t\mathbb E[\|\nabla f(w_t)\|^2] + \frac{2(1-\delta)\eta^2}{\delta}\left(\frac{2a}{\delta}+\sigma^2\right)\sum_{t=1}^T\theta_{t}.
\end{equation*}
The proof is complete. 
\end{proof}

We are  now in position to prove Theorem~\ref{thm:EF_strong_conv}. We first recall a useful result on $L$-smooth functions. 

\begin{proposition}\label{prop:ineg-L-smooth:normgrad}
Let $f:\R^d\to \R_+$ be $L$-smooth. Then, it holds 
\begin{equation*}
    \|\nabla f(x) \|^2\le 2Lf(x),\quad \forall x\in\R^d.
\end{equation*}
\end{proposition}
\begin{proof}
It holds\footnote{See the proof of \cite[Theorem 2.1.5]{nesterov2018lectures}.}, for all $x,y\in\R^d$, 
\begin{equation}\label{prop-2Lf}
0\le f(y)\le f(x)+\langle\nabla f(x),y-x\rangle+\frac L2\|x-y\|^2:=\varphi_x(y)
\end{equation} 
Fix $x\in\R^d$. The function $\varphi_x$ attains its minimum at $y= x-\nabla f(x)/L$.
Evaluating \eqref{prop-2Lf} with this values yields the desired result. 
\end{proof}

\subsection{Proof of Theorem~\ref{thm:EF_strong_conv}: convex case}
Let us introduce $\tilde w_t = w_t - e_t$ where $e_t=\frac1M\sum_{i=1}^Me_t^i$, for any $t\in\mathbb N$. 
\begin{proof}[Proof of item~\ref{thm:EF_strong_conv:itemSconv} of Theorem~\ref{thm:EF_strong_conv}.]
The proof is divided into two steps. The first step consists in the derivation of \eqref{proof-PEF:eq7}. On the second step, we apply the bound of Lemma~\ref{lem:bound_e_t_projFLEF} to conclude the proof. \\
\medskip

\noindent\textbf{Step 1. }
Let $t\in\mathbb N$. We have 
$$\| \tilde w_{t+1}-w^*\|^2 = \|\tilde w_{t}  - w^*\|^2-2\eta\langle g_{t+1}, w_t-w^*\rangle+\eta^2\|g_{t+1}\|^2+2\eta\langle g_{t+1}, w_t-\tilde w_t\rangle.$$
Hence, 
\begin{equation*}
\mathbb E[\| \tilde w_{t+1}-w^*\|^2] = \mathbb E[\|\tilde w_{t}  - w^*\|^2] -2\eta \mathbb E[\langle \nabla f(w_t), w_t-w^*\rangle] +\eta^2\mathbb E[\|g_{t+1}\|^2] + 2\eta \mathbb E[\langle\nabla f(w_t), w_t-\tilde w_t\rangle].
\end{equation*}
Using~\ref{as:stoc_gradient} it holds 
\begin{equation}\label{eq:boundg_t+1}
\mathbb E[\|g_{t+1}\|^2]\le \mathbb E[\|\nabla f(w_t)\|^2] + \frac{\sigma^2}{M}.
\end{equation}
Thus, 
\begin{align}\label{proof-PEF:eq1}
\mathbb E[\| \tilde w_{t+1}-w^*\|^2] &= \mathbb E[\|\tilde w_{t}  - w^*\|^2] -2\eta \mathbb E[\langle \nabla f(w_t), w_t-w^*\rangle] +\eta^2\mathbb E[\|\nabla f(w_t)\|^2] + \frac{\eta^2\sigma^2}{M}\nonumber\\
&\quad + 2\eta \mathbb E[\langle\nabla f(w_t), w_t-\tilde w_t\rangle].
\end{align}
By $\mu$-strong convexity of $f$, 
\begin{equation}\label{proof-PEF:eq2}
-2\langle\nabla f(w_t), w_t-w^*\rangle\le-\mu\|w_t-w^*\|^2-2(f(w_t)-f^*). 
\end{equation} 
Moreover since $f$ is $L$-smooth, we have by Proposition~\ref{prop:ineg-L-smooth:normgrad},
\begin{equation}\label{proof-PEF:eq0}
\|\nabla f(w)\|^2\le 2L(f(w)-f^*), \quad \forall w\in\mathbb R^d.
\end{equation}
Using $2\langle\mathsf a,\mathsf b\rangle\le 2L\|\mathsf a\|^2 + \|\mathsf b\|^2/(2L),$ we have 
\begin{equation}\label{proof-PEF:eq3}
2\langle\nabla f(w_t),w_t-\tilde w_t\rangle\le \frac{1}{2L}\|\nabla f(w_t)\|^2 + 2L\|w_t-\tilde w_t\|^2\le f(w_t)-f^* + 2L\|w_t-\tilde w_t\|^2.
\end{equation}
Plugging \eqref{proof-PEF:eq2} and \eqref{proof-PEF:eq3} into \eqref{proof-PEF:eq1}, 
\begin{align}\label{proof-PEF:eq4}
\mathbb E[\| \tilde w_{t+1}-w^*\|^2]&\le  \mathbb E[\|\tilde w_{t}  - w^*\|^2] - \eta\mu \mathbb E[\|w_t-w^*\|^2] -2\eta \mathbb E[f(w_t)-f^*]\nonumber\\
& +\eta^2\mathbb E[\|\nabla f(w_t)\|^2] + \frac{\eta^2\sigma^2}{M} + \eta \mathbb E[f(w_t)-f^*] + 2\eta L\mathbb E[\|w_t-\tilde w_t\|^2].
\end{align}
Using $\|\mathsf a+\mathsf b\|^2\le 2\|\mathsf a\|^2+2\|\mathsf b\|^2$, we have $-\|w_t-w^*\|^2 \le -\frac12\|\tilde w_t-w^*\|^2+ \|w_t-\tilde w_t\|^2$. Hence, \eqref{proof-PEF:eq4} becomes 
\begin{align}\label{proof-PEF:eq5}
\mathbb E[\| \tilde w_{t+1}-w^*\|^2]&\le \Big(1-\frac{\eta\mu}{2}\Big) \mathbb E[\|\tilde w_{t}  - w^*\|^2] +\eta(2L+\mu) \mathbb E[\|w_t-\tilde w_t\|^2]\nonumber\\
&\quad-\eta \mathbb E[f(w_t)-f^*] +\eta^2\mathbb E[\|\nabla f(w_t)\|^2] + \frac{\eta^2\sigma^2}{M}.
\end{align}
Using \eqref{proof-PEF:eq0} again, \eqref{proof-PEF:eq5} becomes 
\begin{align*}
\mathbb E[\| \tilde w_{t+1}-w^*\|^2]&\le \Big(1-\frac{\eta\mu}{2}\Big) \mathbb E[\|\tilde w_{t}  - w^*\|^2] +\eta(2L+\mu) \mathbb E[\|w_t-\tilde w_t\|^2]\nonumber\\
&\quad-\eta(1-2L\eta) \mathbb E[f(w_t)-f^*]  + \frac{\eta^2\sigma^2}{M}.
\end{align*}
Since $\eta\le\frac{1}{4L}$, we obtain 
\begin{align}\label{proof-PEF:eq7}
\mathbb E[\| \tilde w_{t+1}-w^*\|^2]&\le \Big(1-\frac{\eta\mu}{2}\Big) \mathbb E[\|\tilde w_{t}  - w^*\|^2] +\eta(2L+\mu) \mathbb E[\|w_t-\tilde w_t\|^2]\nonumber\\
&\quad -\frac{\eta}{2} \mathbb E[f(w_t)-f^*]  + \frac{\eta^2\sigma^2}{M}.
\end{align}

\noindent\textbf{Step 2. } Let $s_t = \mathbb E[f(w_t)]-f^* $ and $r_t=\mathbb E[\|\tilde w_t-w^*\|^2]$. 
From \eqref{proof-PEF:eq7}, we have, for any $\theta_t>0$, 
$$\frac{\eta}{2}\theta_ts_t\le (1-\frac{\eta\mu}{2}) \theta_tr_t- \theta_tr_{t+1} +\eta(2L+\mu)\theta_t \mathbb E[\|e_t\|^2] + \frac{\eta^2\sigma^2}{M}\theta_t.$$
Let $T\ge 1$. Summing over $t$, we obtain,
$$\sum_{t=0}^T\frac{\eta}{2}\theta_ts_t\le \sum_{t=0}^T\left[\Big(1-\frac{\eta\mu}{2}\Big) \theta_tr_t- \theta_tr_{t+1}\right] +\eta(2L+\mu)\sum_{t=0}^T\theta_t \mathbb E[\|e_t\|^2] + \frac{\eta^2\sigma^2}{M}\sum_{t=0}^T\theta_t.$$
Let $\theta_t = (1-\frac{\eta\mu}{2})^{-t}$.  Since $\eta\le \frac{\delta}{L(4+\delta)}\le\frac{2\delta}{\mu(4+\delta)}$ (because $L\ge \mu$ for any $L$-smooth and $\mu$-strongly convex function), we have $\frac{\theta_{t+1}}{\theta_t}\le \frac{1}{1-\frac{\eta\mu}{2}}\le 1+\frac{\delta}{4}$. Hence, the sequence $(\theta_t)_{t\ge 0}$  satisfies the assumption of Lemma~\ref{lem:bound_e_t_projFLEF}. By \eqref{proof-PEF:eq0w}, 
\begin{align*}
\sum_{t=0}^T\frac{\eta}{2}\theta_ts_t&\le\sum_{t=0}^T\left[ \Big(1-\frac{\eta\mu}{2}\Big) \theta_tr_t- \theta_tr_{t+1}\right]+\eta(2L+\mu)\frac{8b\eta^2}{\delta^2}\sum_{t=0}^{T-1} \theta_t\mathbb E[\|\nabla f(w_t)\|^2] \\
&\quad + \eta(2L+\mu)\frac{2(1-\delta)\eta^2}{\delta}\left(\frac{2a}{\delta}+\sigma^2\right)\sum_{t=1}^T\theta_{t}  + \frac{\eta^2\sigma^2}{M}\sum_{t=0}^T\theta_t. 
\end{align*}
Since $\mathbb E[\|\nabla f(w_t)\|^2]\le 2Ls_t$ (by \eqref{proof-PEF:eq0}) and since  $\eta$ is such that $(2L+\mu)\frac{16bL\eta^2}{\delta^2}\le \frac 25$,  
we have 
\begin{align*}
\sum_{t=0}^T\frac{\eta}{2}\theta_ts_t&\le\sum_{t=0}^T\left[ \Big(1-\frac{\eta\mu}{2}\Big) \theta_tr_t- \theta_tr_{t+1}\right] +\frac{2}{5}\eta\sum_{t=0}^{T-1} \theta_ts_t \\
&\quad + \eta(2L+\mu)\frac{2(1-\delta)\eta^2}{\delta}\left(\frac{2a}{\delta}+\sigma^2\right)\sum_{t=1}^T\theta_{t}  + \frac{\eta^2\sigma^2}{M}\sum_{t=0}^T\theta_t. 
\end{align*}
Denoting $\Theta_T = \sum_{t=0}^T\theta_t$, we obtain 
\begin{align*}
\frac{1}{\Theta_T}\sum_{t=0}^T\theta_ts_t &\le\frac{10}{\eta \Theta_T}\sum_{t=0}^T\left[ \Big(1-\frac{\eta\mu}{2}\Big) \theta_tr_t- \theta_tr_{t+1}\right]+ 10(2L+\mu)\frac{2(1-\delta)\eta^2}{\delta}\left(\frac{2a}{\delta}+\sigma^2\right)  + 10\frac{\eta\sigma^2}{M}. 
\end{align*}
Since $(1-\frac{\eta\mu}{2})\theta_t=\theta_{t-1}$ (which also holds for $t=0$), we have, by telescopic sum, 
\begin{align*}
\frac{1}{\Theta_T}\sum_{t=0}^T\theta_ts_t &\le\frac{10}{\eta \Theta_T}\sum_{t=0}^T\left[ \theta_{t-1}r_t- \theta_tr_{t+1}\right]+ 10(2L+\mu)\frac{2(1-\delta)\eta^2}{\delta}\left(\frac{2a}{\delta}+\sigma^2\right) + 10\frac{\eta\sigma^2}{M}\\
&\le \frac{10}{\eta \Theta_T}\theta_{-1}r_0 +10(2L+\mu)\frac{2(1-\delta)\eta^2}{\delta}\left(\frac{2a}{\delta}+\sigma^2\right) + 10\frac{\eta\sigma^2}{M}.
\end{align*}
Now, we use that $\Theta_T\ge \theta_T$ to obtain 
\begin{equation*}
\frac{1}{\Theta_T}\sum_{t=0}^T\theta_ts_t\le \frac{10}{\eta}\Big(1-\frac{\eta\mu}{2}\Big)^{T+1} \mathbb E[\|w_0-w^*\|^2]+10(2L+\mu)\frac{2(1-\delta)\eta^2}{\delta}\left(\frac{2a}{\delta}+\sigma^2\right) + 10\frac{\eta\sigma^2}{M}.
\end{equation*}
This completes the proof since $\frac{1}{\Theta_T}\sum_{t=0}^T\theta_ts_t = \mathbb E[f(w^{\mathrm{out}})]-f^*$. 
\end{proof}

\begin{proof}[Proof of item~\ref{thm:EF_strong_conv:itemconv} of Theorem~\ref{thm:EF_strong_conv}.]
The proof follows the proof of item  \ref{thm:EF_strong_conv:itemSconv} of Theorem \ref{thm:EF_strong_conv}. In particular, the computations of Step 1 are still valid when $\mu=0$. In Step 2, we apply Lemma~\ref{lem:bound_e_t_projFLEF} with $\theta_t=1$. 
\end{proof}
\subsection{Proof of Theorem \ref{thm:EF_strong_conv}: non-convex case}

\begin{proof}[Proof of item \ref{thm:EF_strong_conv:itemnoconv} of Theorem \ref{thm:EF_strong_conv}.]
The proof is divided into two steps. The first step consists in the derivation of \eqref{proof-non_conv_eq1}. On the second step, we apply the bound of Lemma~\ref{lem:bound_e_t_projFLEF} to conclude the proof. \\
\medskip

\noindent\textbf{Step 1.}
Let $t\in\mathbb N$. 
By \eqref{prop-2Lf}, 
\begin{align*}
f(\tilde w_{t+1}) &\le  f(\tilde w_t) + \langle\nabla f(\tilde w_t), \tilde w_{t+1}-\tilde w_t\rangle +\frac L2 \|\tilde w_{t+1} - \tilde w_t\|^2 \\
&= f(\tilde w_t) - \eta\langle\nabla f(\tilde w_t), g_{t+1}\rangle +\frac L2\eta^2 \| g_{t+1} \|^2 
\end{align*}
Taking the expectancy and using  \eqref{eq:boundg_t+1}, 
\begin{align*}
\mathbb E[f(\tilde w_{t+1})] &\le  \mathbb E[f(\tilde w_t)] - \eta\mathbb E[\langle\nabla f(\tilde w_t),\nabla f(w_t)\rangle] +\frac L2\eta^2 \mathbb E[\| g_{t+1} \|^2]\\
&\le  \mathbb E[f(\tilde w_t)] - \eta\mathbb E[\langle\nabla f(\tilde w_t),\nabla f(w_t)\rangle] + \frac L2\eta^2\left(\mathbb E[\|\nabla f(w_t)\|^2]+\frac{\sigma^2}{M} \right)
\end{align*}
Using $\langle\mathsf a,\mathsf b\rangle\le\|\mathsf a\|^2/2 + \|\mathsf b\|^2/2,$ 
\begin{align*}
\mathbb E[f(\tilde w_{t+1})]&\le  \mathbb E[f(\tilde w_t)] - \eta\mathbb E[\langle\nabla f(\tilde w_t)-\nabla f(w_t)+\nabla f(w_t),\nabla f(w_t)\rangle] + \frac L2\eta^2\left(\mathbb E[\|\nabla f(w_t)\|^2]+\frac{\sigma^2}{M} \right)\\
&\le  \mathbb E[f(\tilde w_t)] +\frac{\eta}{2}(L\eta-1)\mathbb E[\|\nabla f(w_t)\|^2] + \frac{\eta}{2}\mathbb E[\| \nabla f(\tilde w_t)-\nabla f(w_t)\|^2]+\frac{\eta^2L\sigma^2}{2M}\\
&\le \mathbb E[f(\tilde w_t)] +\frac{\eta}{2}(L\eta-1)\mathbb E[\|\nabla f(w_t)\|^2] + \frac{\eta L^2}{2}\mathbb E[\| e_t\|^2]+\frac{\eta^2L\sigma^2}{2M}
\end{align*}
where we used \eqref{ineg:defLsmooth} for the last inequality.  
Since $\eta\le \frac{1}{2L}$ we obtain 
\begin{equation}\label{proof-non_conv_eq1}
\mathbb E[f(\tilde w_{t+1})]\le \mathbb E[f(\tilde w_t)] -\frac{\eta}{4}\mathbb E[\|\nabla f(w_t)\|^2] + \frac{\eta L^2}{2}\mathbb E[\| e_t\|^2]+\frac{\eta^2L\sigma^2}{2M}.
\end{equation}

\noindent\textbf{Step 2. }
Let $T\ge 1$. 
From \eqref{proof-non_conv_eq1}, we have 
\begin{equation*}
\frac{\eta}{4}\sum_{t=0}^T\mathbb E[\|\nabla f(w_t)\|^2]\le \sum_{t=0}^T\Big[ \mathbb E[f(\tilde w_t)-f^*]-\mathbb E[f(\tilde w_{t+1})-f^*]\Big]  + \frac{\eta L^2}{2}\sum_{t=0}^T\mathbb E[\| e_t\|^2]+(T+1)\frac{\eta^2L\sigma^2}{2M}.
\end{equation*}
Let us now apply Lemma \ref{lem:bound_e_t_projFLEF} with weights $\theta_t=1$: 
\begin{align*}
\frac{\eta}{4}\sum_{t=0}^T\mathbb E[\|\nabla f(w_t)\|^2]&\le \mathbb E[f(\tilde w_0)-f^*] +\frac{4b\eta^3 L^2}{\delta^2}\sum_{t=0}^{T-1}\mathbb E[\|\nabla f(w_t)\|^2]\\
&\quad+  \frac{(1-\delta)T\eta^3L^2}{\delta}\left(\frac{2a}{\delta}+\sigma^2\right)+(T+1)\frac{\eta^2L\sigma^2}{2M}.
\end{align*}
Since $\frac{4b\eta^2 L^2}{\delta^2}\le\frac18$,
\begin{align*}
\frac{\eta}{8}\sum_{t=0}^T\mathbb E[\|\nabla f(w_t)\|^2]\le \mathbb E[f(\tilde w_0)-f^*] +  \frac{(1-\delta)T\eta^3L^2}{\delta}\left(\frac{2a}{\delta}+\sigma^2\right)+(T+1)\frac{\eta^2L\sigma^2}{2M}.
\end{align*}
Hence, 
\begin{align*}
\frac{1}{T+1}\sum_{t=0}^T\mathbb E[\|\nabla f(w_t)\|^2]\le \frac{8}{(T+1)\eta}\mathbb E[f(\tilde w_0)-f^*] +  \frac{8(1-\delta)\eta^2L^2}{\delta}\left(\frac{2a}{\delta}+\sigma^2\right)+\frac{8\eta L\sigma^2}{2M}.
\end{align*}
\end{proof}

\section{Experimental details and Additional experiments\label{sec:appx_num_exp}}

In this section we start by giving implementation details in Subsection \ref{sec:app_exp_details} and then in the following subsections provide further experiments.

\subsection{Experimental details\label{sec:app_exp_details}}

\subsubsection{Hardware and software.}
All experiments were conducted on an internal cluster machine equipped with an Intel Xeon Gold 5320 CPU (104 cores, 2.20 GHz), 500 GB of RAM, and a 
single NVIDIA A30 GPU with 24 GB of memory (driver version 545.23.08, CUDA version 12.3). 
The software environment consisted of Python 3.9.19 and PyTorch 2.3, running on Debian GNU/Linux 12 (Bookworm).
Experimental run took approximately 1 hour on average with 3 clients on 10 processes, with the most computationally intensive run requiring up to 48 hours.
All experiments were performed on our institutional infrastructure on CPU; no cloud computing resources were used. 
Reproducing the CIFAR-10 experiments with $3$ clients and a batch size of $128$ requires $24$ GB of RAM.

\subsubsection{Hyperparameters}

\begin{itemize}
\item \textbf{Early stopping:} patience of 10 epochs and a minimum delta of 0.001.
\item \textbf{Batch size:} 128 per client.
\item \textbf{Learning rate scheduler:} PyTorch's \texttt{ReduceLROnPlateau}\footnote{\url{https://pytorch.org/docs/2.3/generated/torch.optim.lr_scheduler.ReduceLROnPlateau.html}} starting from 0.1, with a patience of 2 epochs, a decay factor of 0.5, and a minimum learning rate of 0.001.
\end{itemize}

\subsubsection{Algorithms parameter}

Based on the observed trade-offs between convergence speed, communication cost, and stability across both datasets, the values $k = 0.01$ and $K = 3$ appear to offer a good compromise. These settings strike a strong balance between training efficiency and robustness.
We note that although the algorithms differ in efficiency, they all converge to reasonable minima. Moreover, incorporating the Error Feedback mechanism consistently improves performance-an expected outcome given the use of biased compressors, as previously discussed.
On the CIFAR-10 experiments for selecting $K$ (see the two rightmost plots in Figure~\ref{fig:find_best_K}), one might be surprised by the increase in test loss at the end of training for Algorithm~\ref{fed-avg-proj-EF}. This rise is due to overfitting, as the training loss continues to decrease monotonically (see Appendix~\ref{sec:appx_num_exp}). The reason the algorithm was not stopped earlier lies in the early stopping criterion (specifically the patience). Naturally, had we optimized early stopping specifically for Algorithm~\ref{fed-avg-proj-EF}, more favorable stopping conditions would have been chosen. Note that the curves start at different points, as the x-axis value of the first point corresponds to the communication cost after the first epoch. 

\begin{figure}[H]
    \caption{Selection of $k$ (\textbf{Top-}$k$) and $K$ for Algorithms~\ref{fed-avg-proj-var2} and~\ref{fed-avg-proj-EF}, for $M=3$ clients. 
    }\label{fig:find_best_K}
    \begin{center}
    \includegraphics[scale=.39]{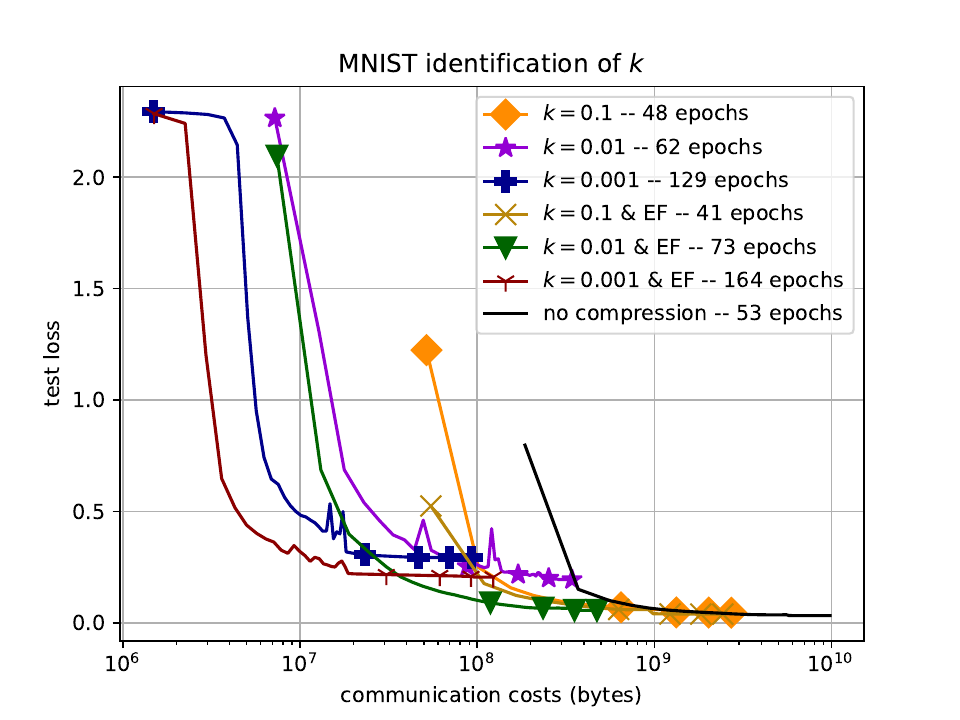}
    \includegraphics[scale=.39]{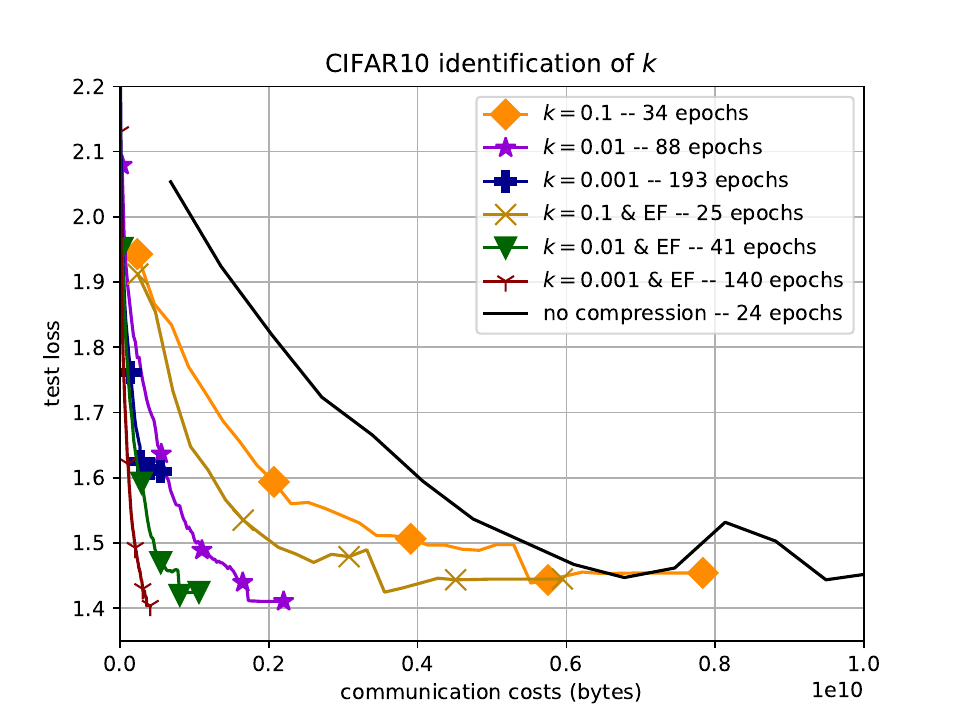}
    \includegraphics[scale=.39]{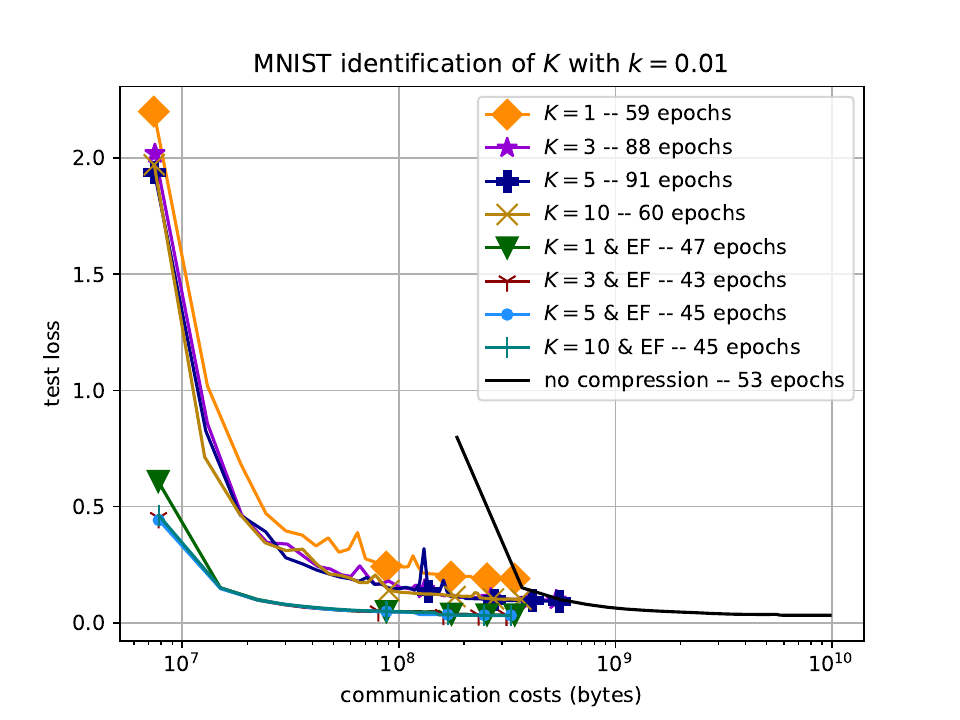} 
    \includegraphics[scale=.39]{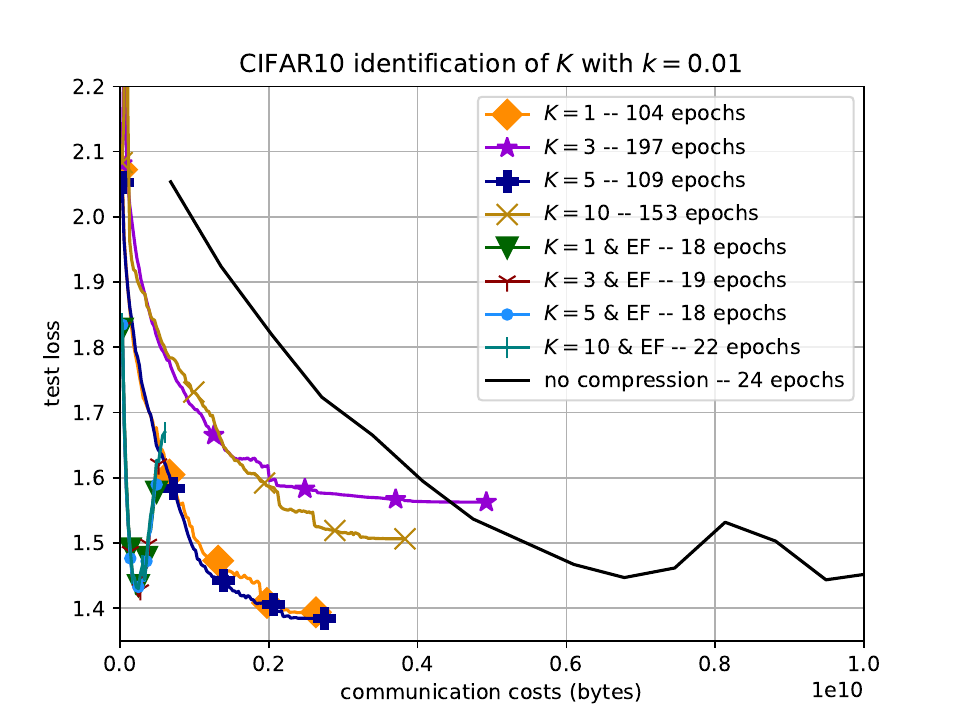} 
    \end{center}
\end{figure}

\subsubsection{Algorithms evaluated in our experiments.}

Algorithm~\ref{alg:fedavg_compr} corresponds to the standard \texttt{FedAvg} algorithm with gradient compression.  
Algorithm~\ref{fedavg-ef} is a variant that incorporates the compression error using the well-known Error Feedback mechanism. All our experiments are done setting $\zeta=0.75$.   
Algorithm~\ref{alg:ef21} implements the \texttt{EF21} algorithm from \cite{NEURIPS2021_231141b3}. In our experiments, we observed that \texttt{EF21} performs poorly on large models (see the analysis provided in Subsection~\ref{sec:ef21_diana_problem}). To address this limitation, we introduce a forgetting parameter $\gamma \in (0,1)$, which improves the robustness of the method (see Algorithm~\ref{alg:ef21_gamma}). 

Table~\ref{tab:ef21_ef21gamma} shows the first descent directions for Algorithms~\ref{alg:ef21} and \ref{alg:ef21_gamma}. This illustrates that in classical \texttt{EF21}, the compressed version of the initial gradients---and their propagation---persist across all descent directions $\mathsf D_k^i$, whereas in our modified version, their influence gradually vanishes at a rate governed by $\gamma$.

A similar effect is observed for \texttt{DIANA} (Algorithm~\ref{alg:diana}) and its modified counterpart with forgetting (Algorithm~\ref{alg:diana_with_gamma}).  Our experiments are done with hyperparemeters $\alpha=0.9$ and $\beta=0.1$ (see evaluation for different values of $\alpha$ and $\beta$ in Subsection \ref{sec:ef21_diana_problem}).

\begin{algorithm}[H]
\caption{\texttt{FedAvg} with compression}
\label{alg:fedavg_compr}
\begin{algorithmic}[1]
\State \textbf{Initialization: } $w_0\in\R^d$. 
\For{$t=0,\dots,T$}
\For{Each client $i$}
\State Receive $w_t$.
\State Compute Stochastic Gradient $g_{t+1}^i$.
\State Send $\mathcal C(g_{t+1}^i)$ to the Central Server. 
\EndFor
\State \textbf{Central Server:}
\State $w_{t+1}= w_t-\frac\eta M\sum_{i=1}^M\mathcal C(g_{t+1}^i)$.
\EndFor
\end{algorithmic}
\end{algorithm}

\begin{algorithm}[H]
\caption{\texttt{FedAvg with Error Feedback}}\label{fedavg-ef}
\begin{algorithmic}[1]
    \State \textbf{Initialization:} $w_0\in\R^d$, $e_0^i = 0 \in\mathbb R^d$, $\forall i\in[M]$. $\zeta\in(0,1]$. 
    \For{$t=0,\dots, T$}
	\For{each client $i$}
	\State Receive $w_t$  	  
    \State Compute Stochastic Gradient $g_{t+1}^i$.
    \State  $\tilde g_{t+1}^i=g_{t+1}^i + \zeta e_t^i$. \Comment{Add the previous compression error}
\State $e_{t+1}^i = \tilde g_{t+1}^i- \mathcal C(\tilde g_{t+1}^i)$ \Comment{Update the compression error}    
    \State Send $\mathcal C(\tilde g_{t+1}^i)$ to the Central Server 
    \EndFor 
    \State \textbf{Central Server:} 
    \State $w_{t+1} = w_t - \frac \eta M\sum_{i=1}^M\mathcal C(\tilde g_{t+1}^i)$
    \EndFor
\end{algorithmic}
\end{algorithm}

\begin{algorithm}[H]
\caption{\texttt{EF21} \cite{NEURIPS2021_231141b3}}
\label{alg:ef21}
\begin{algorithmic}[1]
\State \textbf{Initialization:} $\mathsf D_0^i=0\in\R^d$ 
	\For{$t=0,\dots, T$}
    \For{Each client $i\in [M]$}
    	\State Receive $w_t$
    	\State Compute Stochastic Gradient $g_{t+1}^{i}$
        \State $\mathsf M_{t+1}^i=\mathcal C(g_{t+1}^i-\mathsf D_t^i)$
        \State $\mathsf D_{t+1}^i = \mathsf D_t^i +\mathsf M_{t+1}^i$ \Comment{Update local direction of descent}
        \State Send $\mathsf M_{t+1}^i$  to the Central Server
    \EndFor
    \State\textbf{Central Server: }
    \State $\mathsf D_{t+1}^i = \mathsf D_t^i +\mathsf M_{t+1}^i$
    \State $w_{t+1} = w_t - \frac{\eta}{M}\sum_{i=1}^M\mathsf D_{t+1}^i$
    \EndFor
\end{algorithmic}
\end{algorithm}

\begin{algorithm}[H]
\caption{\texttt{EF21} with parameter $\gamma\in (0,1)$.  }
\label{alg:ef21_gamma}
\begin{algorithmic}[1]
\State \textbf{Initialization:} $\mathsf D_0^i=0\in\R^d$ 
	\For{$t=0,\dots, T$}
    \For{Each client $i\in [M]$}
    	\State Receive $w_t$
    	\State Compute Stochastic Gradient $g_{t+1}^{i}$
        \State $\mathsf M_{t+1}^i=\mathcal C(g_{t+1}^i-\gamma\mathsf D_t^i)$
        \State $\mathsf D_{t+1}^i = \gamma\mathsf D_t^i +\mathsf M_{t+1}^i$ \Comment{Update local direction of descent}
        \State Send $\mathsf M_{t+1}^i$  to the Central Server
    \EndFor
    \State\textbf{Central Server: }
    \State $\mathsf D_{t+1}^i = \gamma\mathsf D_t^i +\mathsf M_{t+1}^i$
    \State $w_{t+1} = w_t - \frac{\eta}{M}\sum_{i=1}^M\mathsf D_{t+1}^i$
    \EndFor
\end{algorithmic}
\end{algorithm}

\begin{algorithm}[H]
\caption{\texttt{DIANA} \cite{Mishchenko27092024}}
\label{alg:diana}
\begin{algorithmic}[1]
\State \textbf{Initialization:} $h_0^i=h_{0} =\mathsf D_0=0\in\R^d$
\For{$t = 0, \dots, T$}
    \For{Each client $i\in [M]$}
    	\State Receive $w_t$ 
        \State Compute Stochastic Gradient $g_{t+1}^{i}$
        \State $\mathsf M_{t+1}^i=\mathcal C(g_{t+1}^{i}- h_{t}^{i})$
        \State $h_{t+1}^i = h_{t}^{i} + \alpha\mathsf M_{t+1}^i$  \Comment{Update memory}
        \State Send $\mathsf M_{t+1}^i$ to the Central Server 
    \EndFor
    \State\textbf{Central Server: } 
    \State $\mathsf M_{t+1}= \frac{1}{M}\sum_{i=1}^{M} \mathsf M_{t+1}^i$
    \State $\mathsf D_{t+1} = \beta \mathsf D_{t} +  h_{t} +\mathsf M_{t+1}$ \Comment{Compute the descent direction with momentum when $\beta>0$}
    \State  $h_{t+1} = h_{t} + \alpha\mathsf M_{t+1}$ \Comment{Update memory}
    \State $w_{t+1} = w_t - \eta\mathsf D_{t+1}$
\EndFor
\end{algorithmic}
\end{algorithm}

\begin{algorithm}[H]
\caption{\texttt{DIANA} with parameter $\gamma\in (0,1)$.}
\label{alg:diana_with_gamma}
\begin{algorithmic}[1]
\State \textbf{Initialization:} $h_0^i=h_{0} =\mathsf D_0=0\in\R^d$
\For{$t = 0, \dots, T$}
    \For{Each client $i\in [M]$}
    	\State Receive $w_t$ 
        \State Compute Stochastic Gradient $g_{t+1}^{i}$
        \State $\mathsf M_{t+1}^i=\mathcal C(g_{t+1}^{i}-\gamma h_{t}^{i})$
        \State $h_{t+1}^i = \gamma h_{t}^{i} + \alpha\mathsf M_{t+1}^i$  \Comment{Update memory}
        \State Send $\mathsf M_{t+1}^i$ to the Central Server 
    \EndFor
    \State\textbf{Central Server: } 
    \State $\mathsf M_{t+1}= \frac{1}{M}\sum_{i=1}^{M} \mathsf M_{t+1}^i$
    \State $\mathsf D_{t+1} = \beta \mathsf D_{t} + \gamma h_{t} +\mathsf M_{t+1}$ \Comment{Compute the descent direction with momentum when $\beta>0$}
    \State  $h_{t+1} = \gamma h_{t} + \alpha\mathsf M_{t+1}$ \Comment{Update memory}
    \State $w_{t+1} = w_t - \eta\mathsf D_{t+1}$
\EndFor
\end{algorithmic}
\end{algorithm}

\subsubsection{Variability across runs.}
To assess the stability of our method, we report the test cross-entropy loss over 10 independent runs with different random seeds, measured at the end of training.

\begin{table}[ht]
    \centering
    \caption{Standard deviation of the test loss where $\mathcal{C}$=Top-$0.01$.}
    \label{tab:var_comp}
    \begin{tabular}{lccc}
    \toprule
    \textbf{Algorithm} & \textbf{MNIST} & \textbf{CIFAR-10} \\
    \midrule
    no compression                              & $0.0039$  & $0.0423$ \\
         Algorithm \ref{fed-avg-proj-var2}     & $0.0257$  & $0.0284$ \\
     Algorithm \ref{fed-avg-proj-EF} & $0.0045$  & $0.0149$ \\
    Algorithm \ref{alg:fedavg_compr}                              & $0.0201$  & $0.0211$ \\
    Algorithm \ref{fedavg-ef}                         & $0.0074$  & $0.0143$ \\
      Algorithm \ref{alg:ef21_gamma}                      & $0.0138$  & $0.0221$ \\
    Algorithm \ref{alg:diana_with_gamma}                      & $0.0140$  & $0.0358$ \\

    \bottomrule
    \end{tabular}
\end{table}

We observe that Algorithm \ref{fed-avg-proj-EF}  has a lower standard deviation compared to the other algorithms.

\begin{table}[h]
\caption{Comparison of the first descent direction between Algorithms \ref{alg:ef21} and \ref{alg:ef21_gamma}. }
\label{tab:ef21_ef21gamma}
\centering
\begin{tabular}{|c|c|c|}
\hline
Iteration $k$ & $\mathsf D_k^i$ (Alg. \ref{alg:ef21}) & $\mathsf D_k^i$ (Alg. \ref{alg:ef21_gamma}) \\
\hline
$1$ & $\mathcal C(g_1^i)$ & $\mathcal C(g_1^i)$ \\
\hline
$2$ & $\mathcal C(g_1^i)+ \mathcal C(g_2^i-\mathcal C(g_1^i))$ & $\gamma\mathcal C(g_1^i)+ \mathcal C(g_2^i-\gamma\mathcal C(g_1^i))$  \\
\hline
$3$ & \makecell{$\mathcal C(g_1^i)+ \mathcal C(g_2^i-\mathcal C(g_1^i))$\\ $+\mathcal C(g_3^i-\mathcal C(g_1^i)- \mathcal C(g_2^i-\mathcal C(g_1^i)))$} & \makecell{$\gamma^2\mathcal C(g_1^i)+ \gamma\mathcal C(g_2^i-\gamma\mathcal C(g_1^i))$ \\$+\mathcal C(g_3^i-\gamma^2\mathcal C(g_1^i)- \gamma\mathcal C(g_2^i-\gamma\mathcal C(g_1^i)))$} \\
\hline
\end{tabular}

\end{table}

\subsection{Overfitting on the Training Set for CIFAR-10}

As noticed in Figure \ref{fig:find_best_K}, \ref{fig:3c} and \ref{fig:10c}, we can see a rising loss on the test set.  This behavior is a  sign of overfitting. As can be seen, the training loss decreases monotonically, confirming that the model continues to fit the training data while generalization performance degrades.

\begin{figure}[H]
    \centering
    \caption{Selection of $K$ for Algorithms \ref{fed-avg-proj-var2} and \ref{fed-avg-proj-EF}, for $M=3$ clients on train set.}
    \label{fig4}
    \includegraphics[scale=.5]{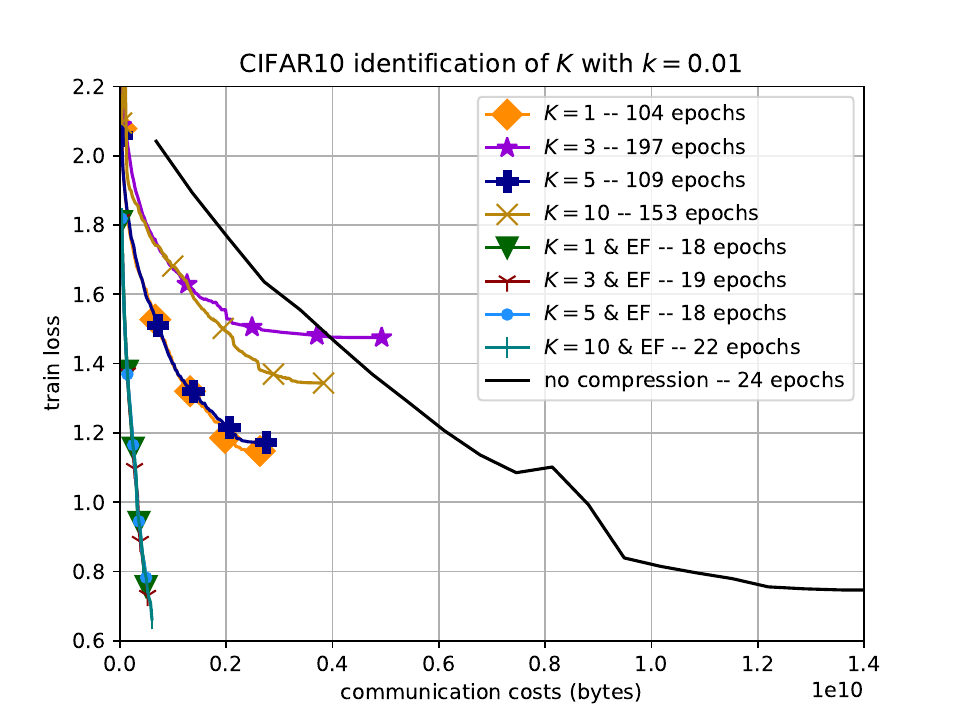} 
\end{figure}

\begin{figure}[H]
    \centering
    \caption{Comparison of Algorithms \ref{fed-avg-proj-var2} and \ref{fed-avg-proj-EF} with \texttt{FedAvg} with compression, \texttt{EF}, \texttt{EF21}, and \texttt{DIANA} for $M=3$ clients on train set.}
    \label{fig5}
    \includegraphics[scale=.5]{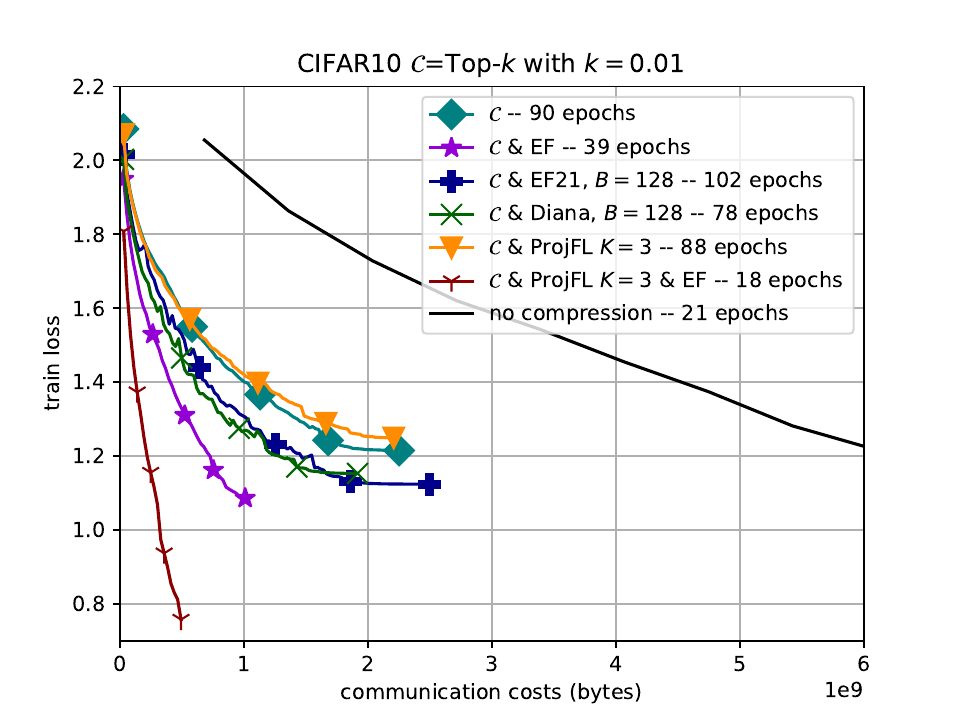} 
\end{figure}

\subsection{Additional Experiments with Uplink and Downlink Communication Costs}

In this subsection, we present the same experiments as in Figure~\ref{fig:3c}, but considering only the uplink communication cost (Figure~\ref{fig:uplink_cc}) or only the downlink communication cost (Figure~\ref{fig:downlink_cc}). 

Regarding the downlink communication cost, note that the central server only sends the difference $w_{t+1} - w_t$ (that is, only the values and indices of $w_{t+1}$ that differ from $w_t$).

\begin{figure}[H]
    \caption{Uplink communication cost with $M=3$ clients. The x-axis represents the total communication cost of the clients (\textit{i.e.}, 3 times the communication cost per client).}\label{fig:uplink_cc}
    \centering
    \includegraphics[scale=.39]{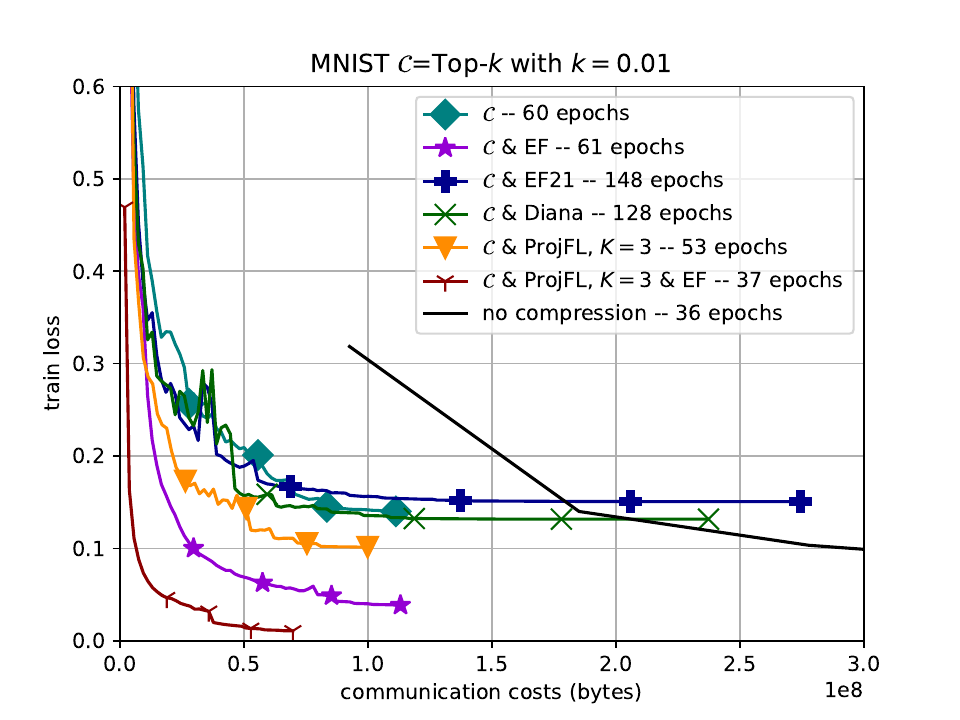} 
    \includegraphics[scale=.39]{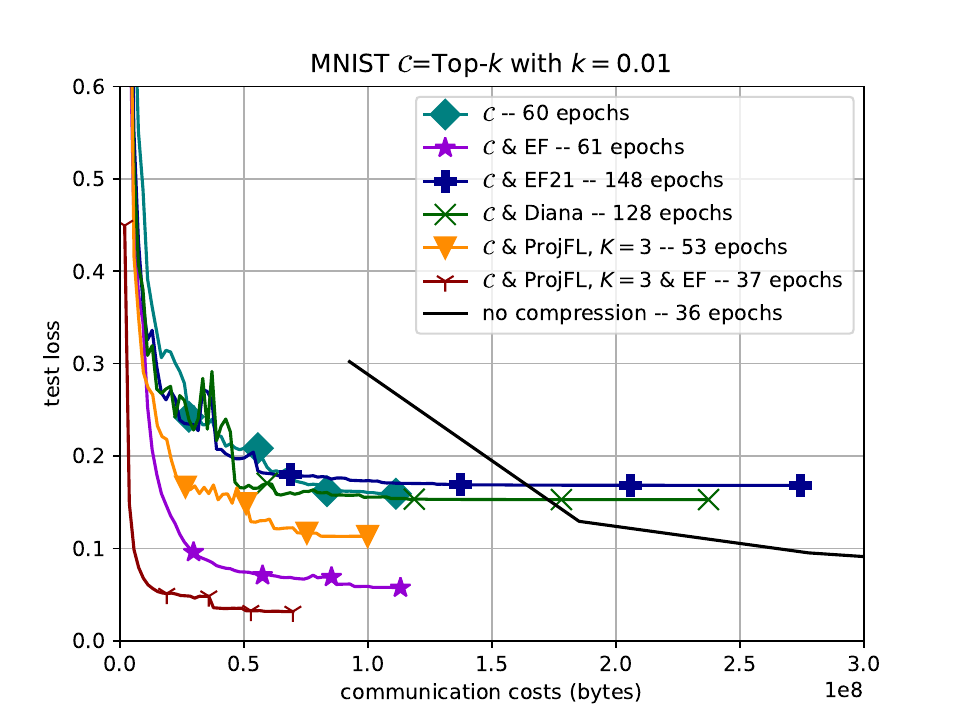}
    \includegraphics[scale=.39]{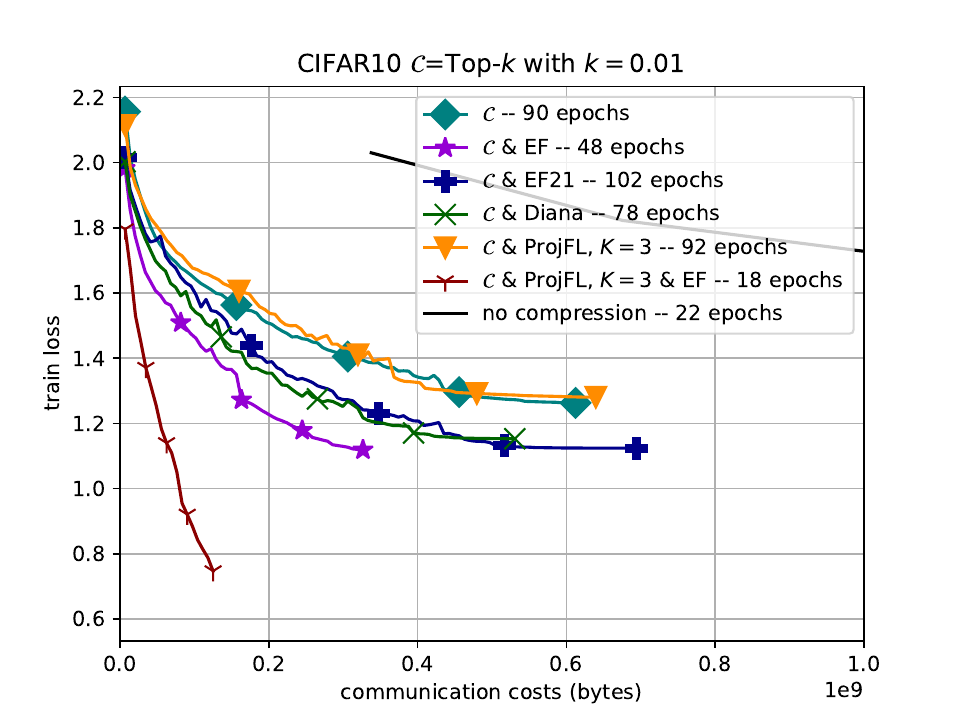} 
    \includegraphics[scale=.39]{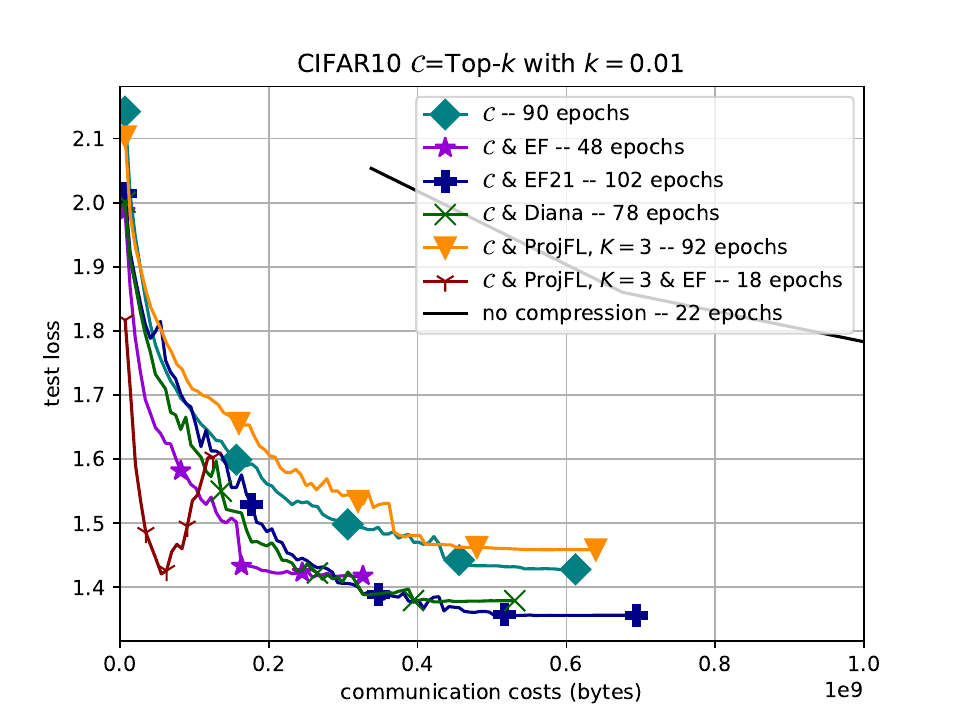}
\end{figure}

\begin{figure}[H]
    \caption{Downlink communication cost with $M=3$ clients.}\label{fig:downlink_cc}
    \centering
    \includegraphics[scale=.39]{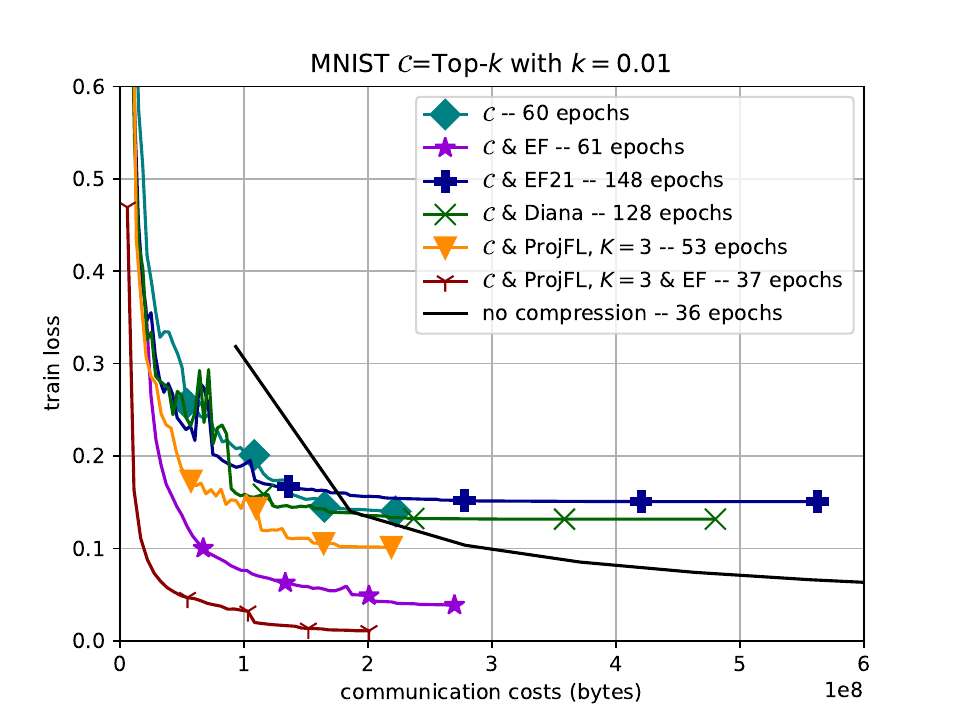} 
    \includegraphics[scale=.39]{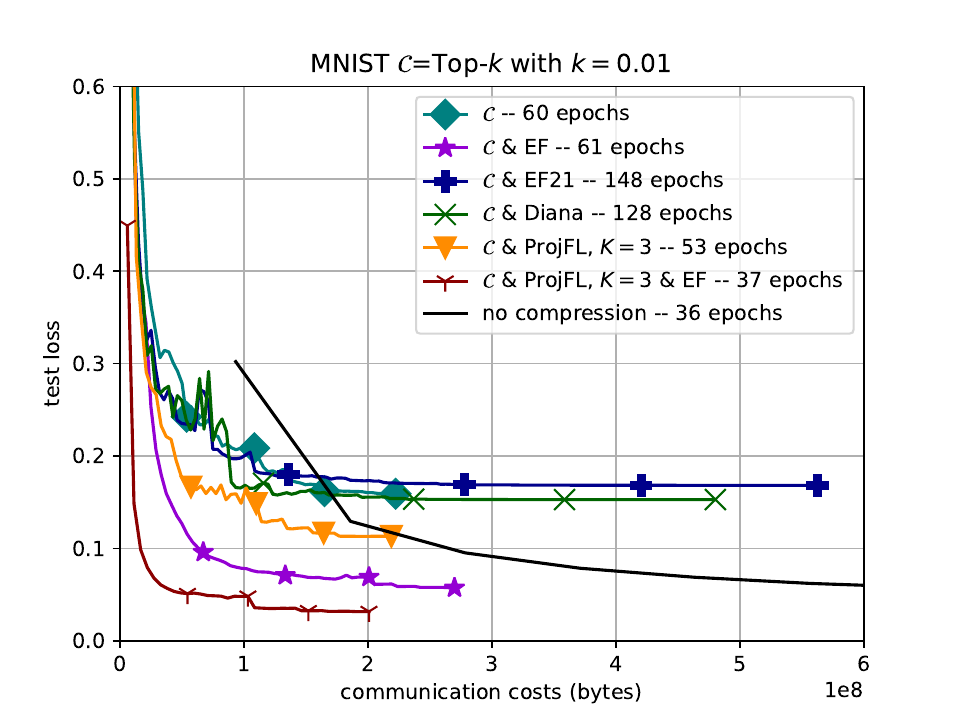} 
    \includegraphics[scale=.39]{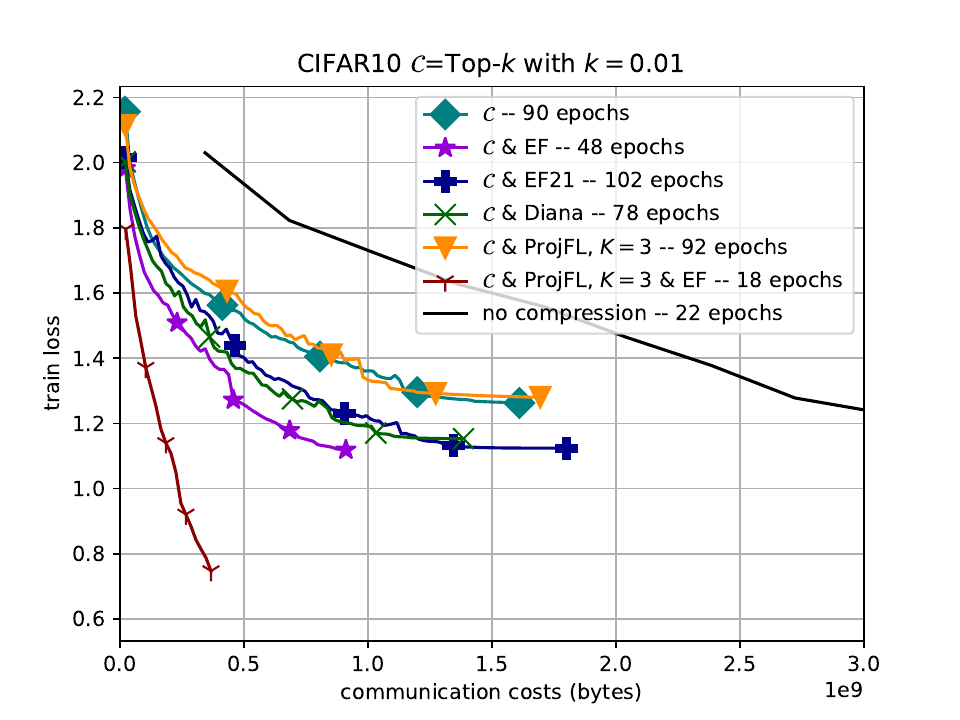} 
    \includegraphics[scale=.39]{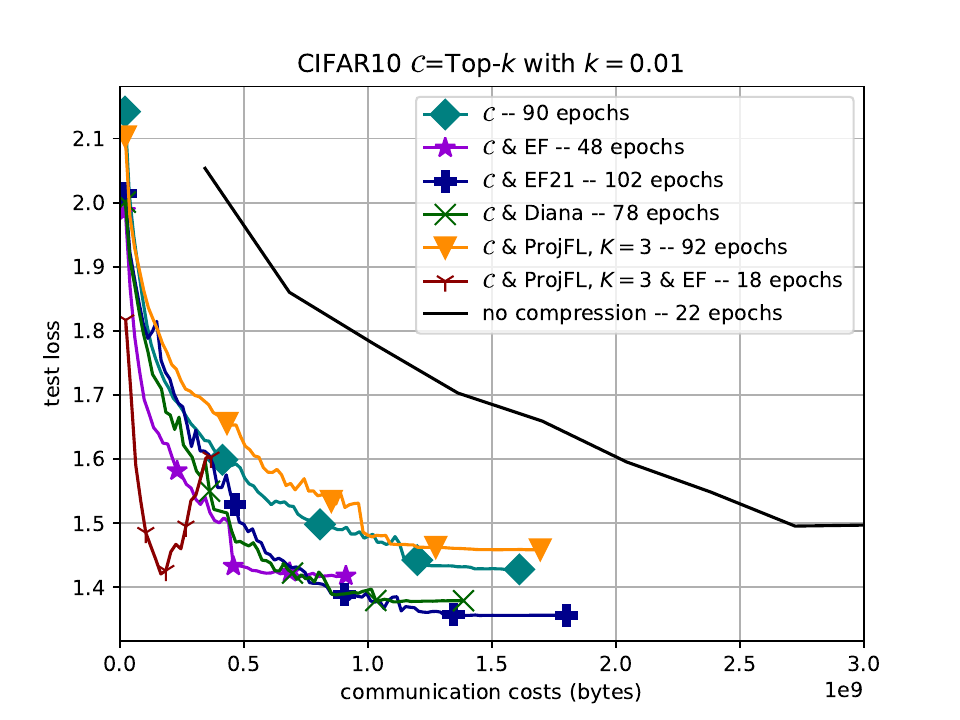}
\end{figure}

\subsection{Accuracy}

In Figure~\ref{fig:acc_mnist} (resp. Figure~\ref{fig:acc_cifar10}), we report both training and test accuracy for the MNIST (resp. CIFAR-10) dataset.

\begin{figure}[H]
    \caption{Accuracy on MNIST dataset with  $M=3$ clients. \textbf{First line:} Total communication cost. \textbf{Second line:} Uplink communication cost only. \textbf{Third line:} Donwlink communication cost only. }\label{fig:acc_mnist}
    \includegraphics[scale=.39]{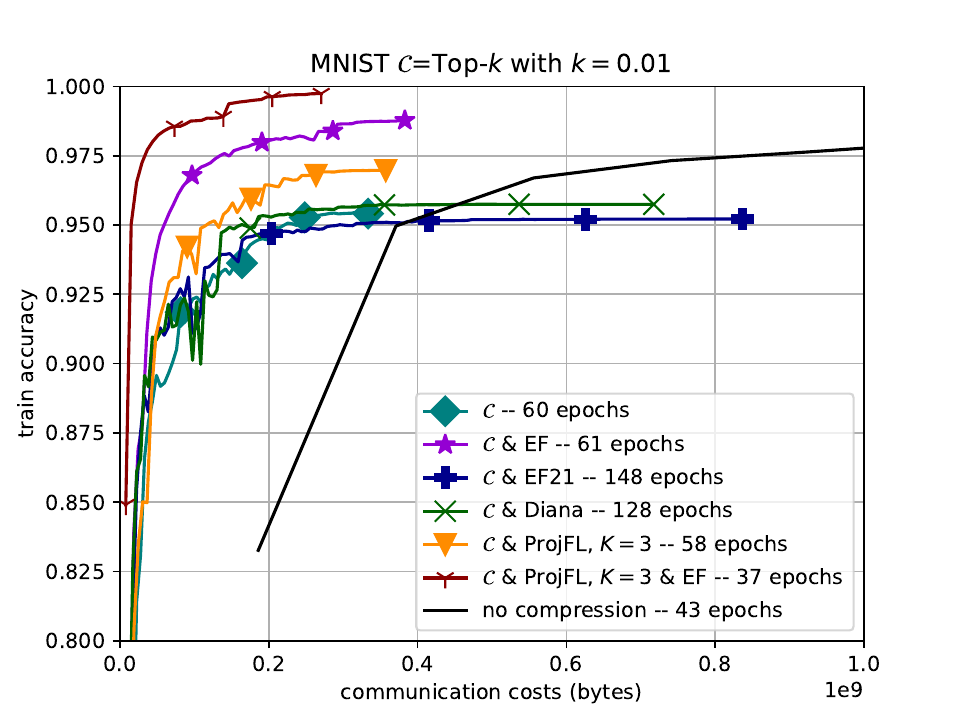}
    \includegraphics[scale=.39]{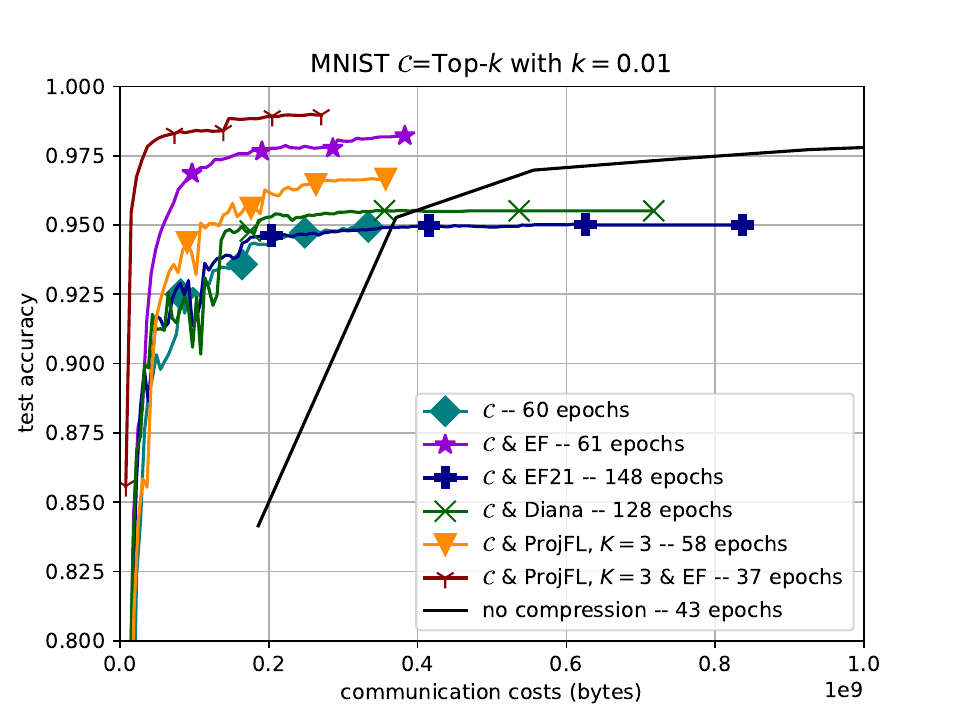} 
    \includegraphics[scale=.39]{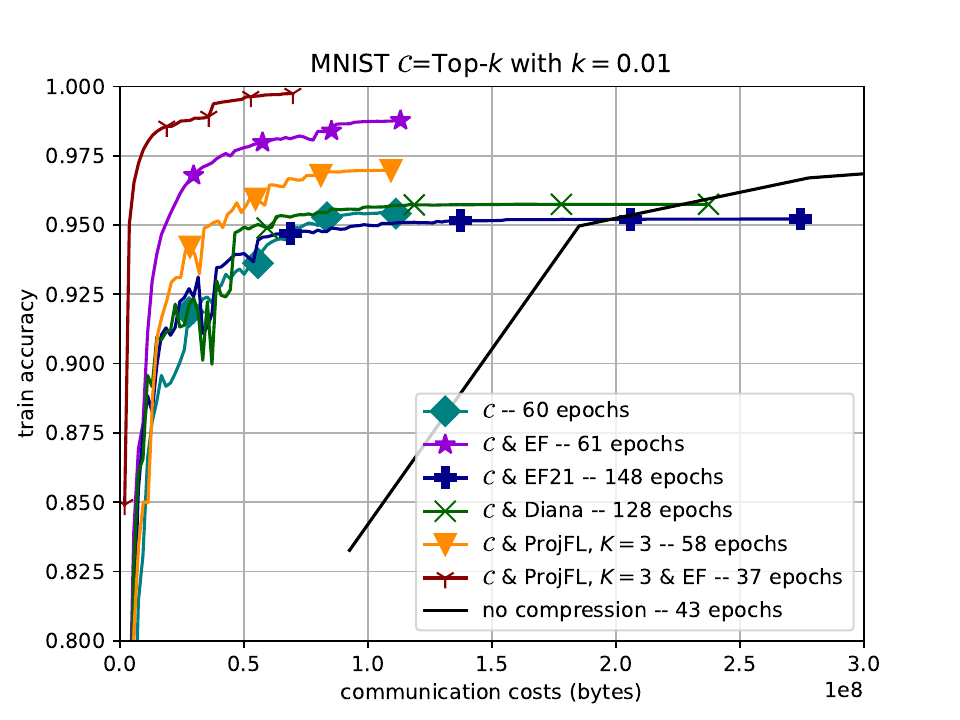} 
    \includegraphics[scale=.39]{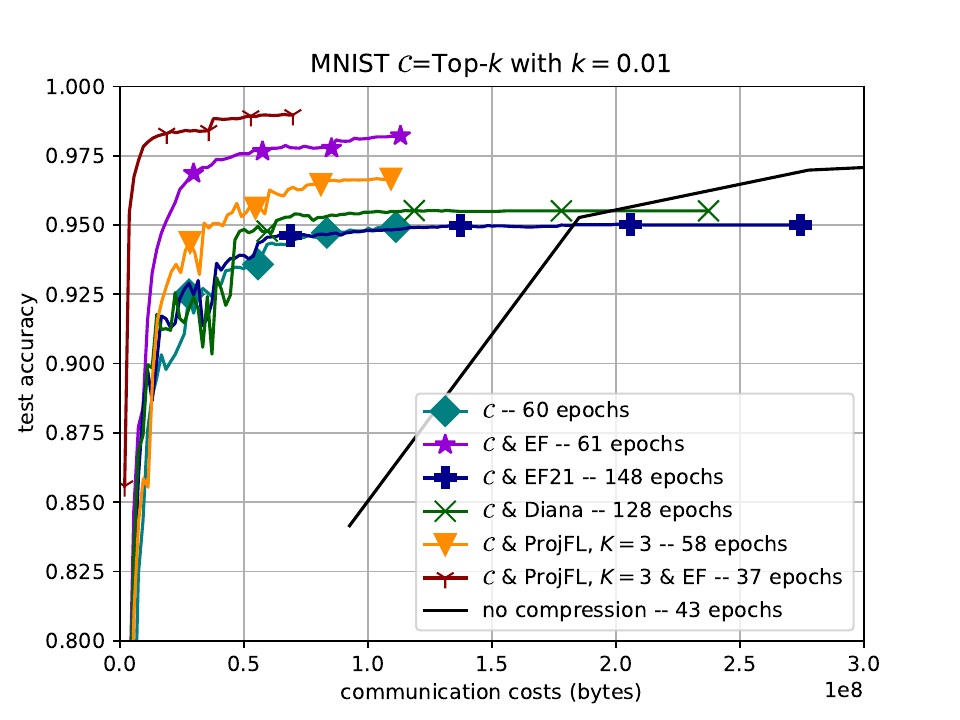} 
    \includegraphics[scale=.39]{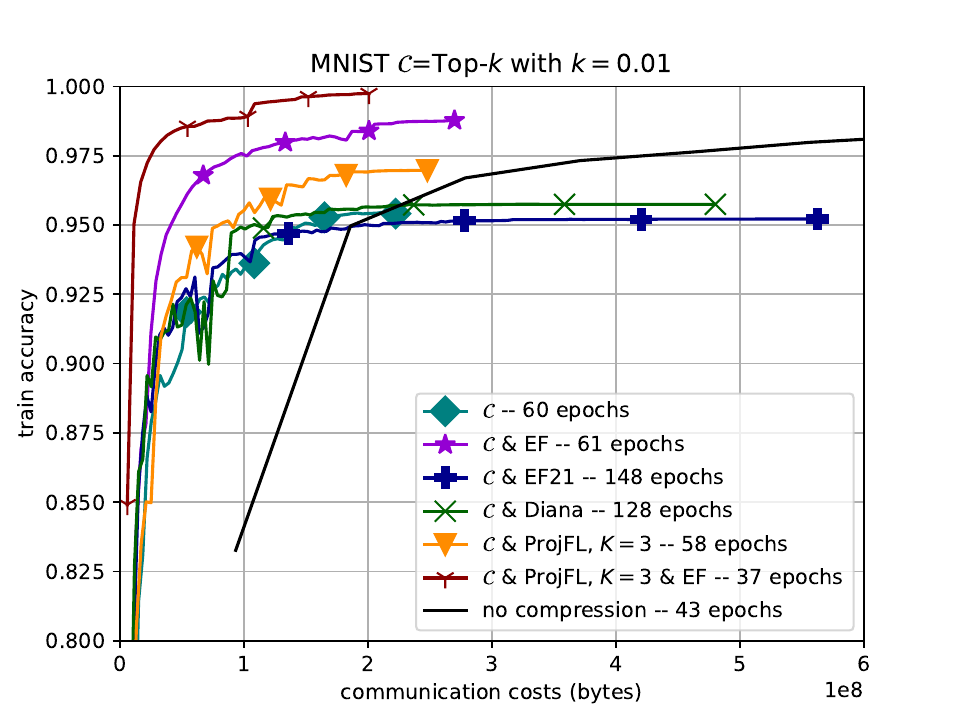} 
    \includegraphics[scale=.39]{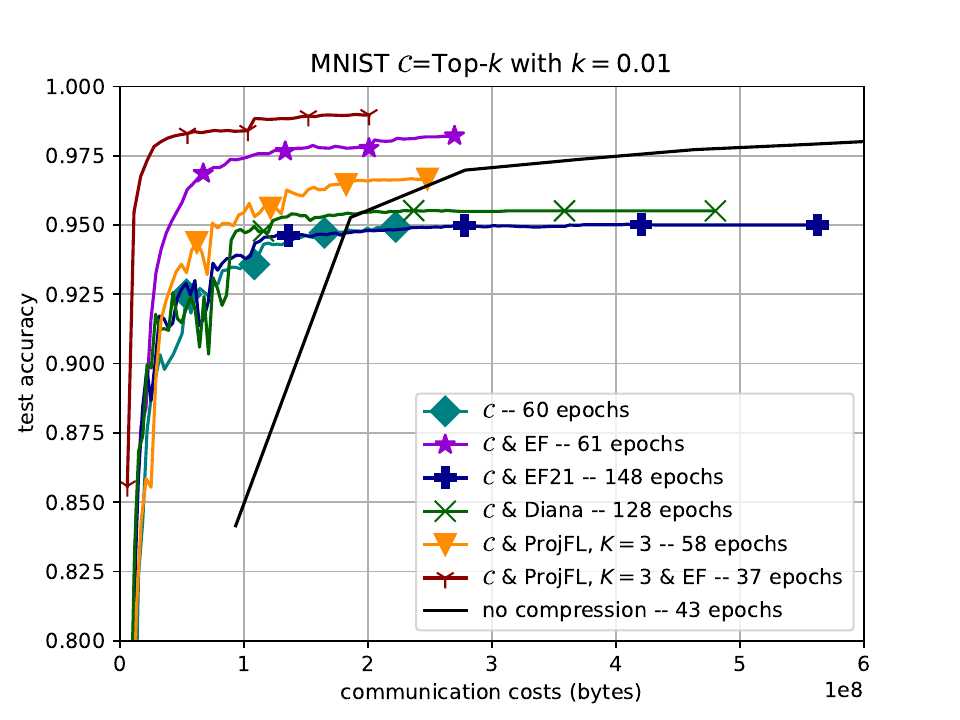}
\end{figure}

\begin{figure}
    \caption{Accuracy on CIFAR-10 dataset with  $M=3$ clients. \textbf{First line:} Total communication cost. \textbf{Second line:} Uplink communication cost only. \textbf{Third line:} Donwlink communication cost only.}\label{fig:acc_cifar10}
    \includegraphics[scale=.39]{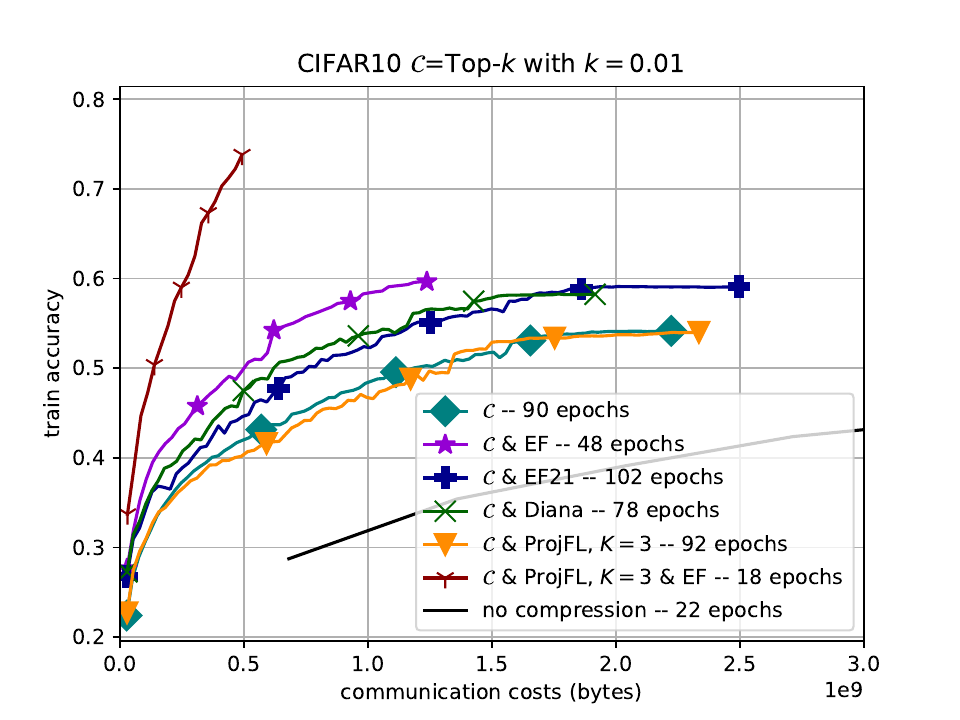}
    \includegraphics[scale=.39]{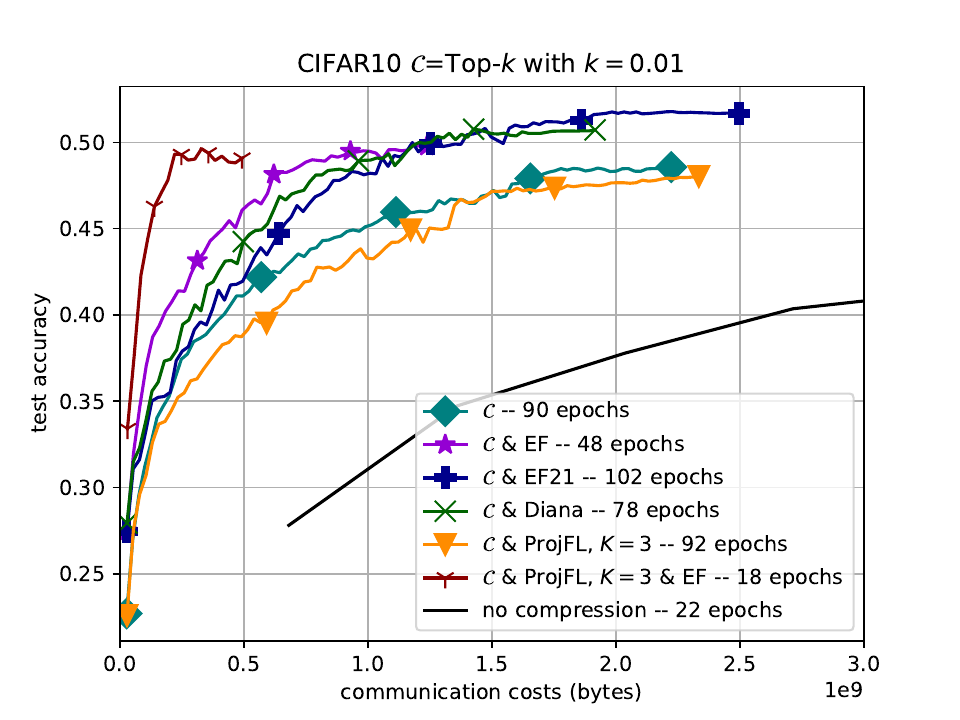} 
    \includegraphics[scale=.39]{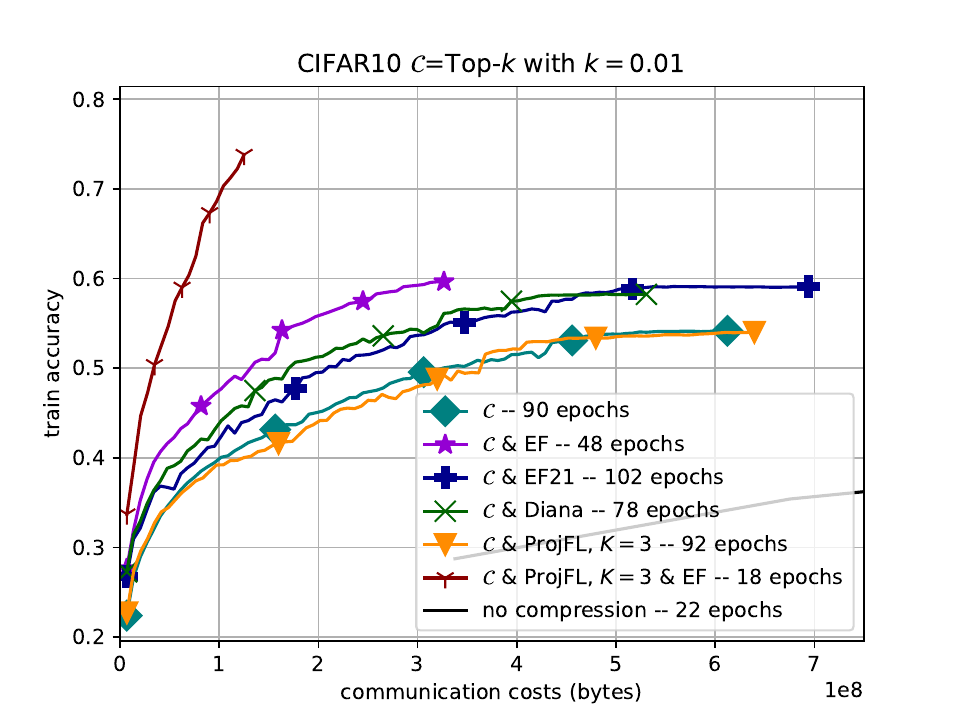} 
    \includegraphics[scale=.39]{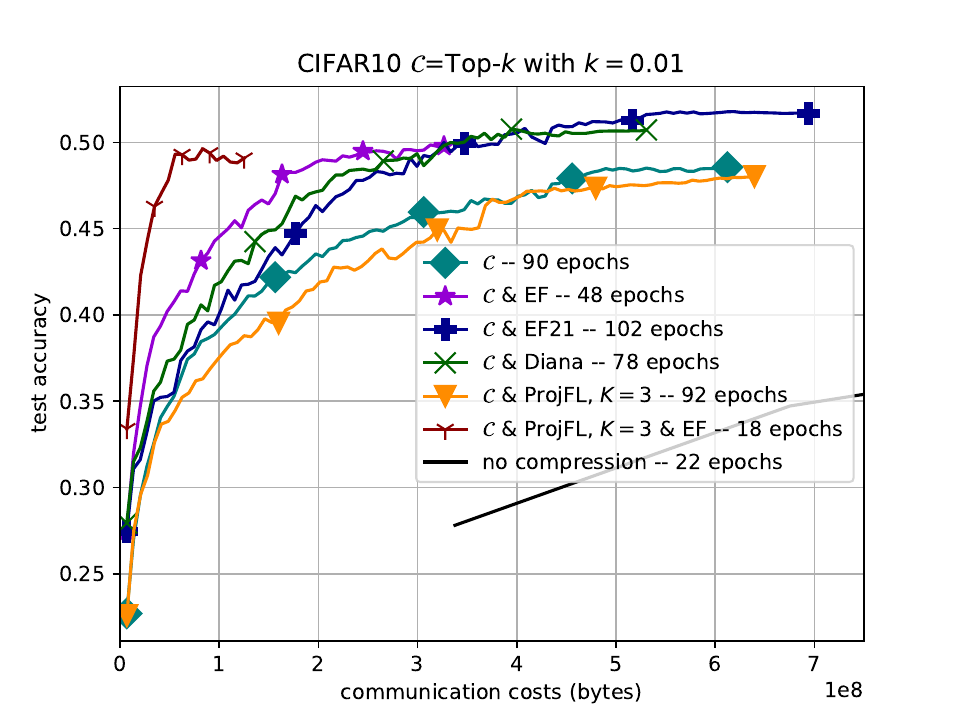}
    \includegraphics[scale=.39]{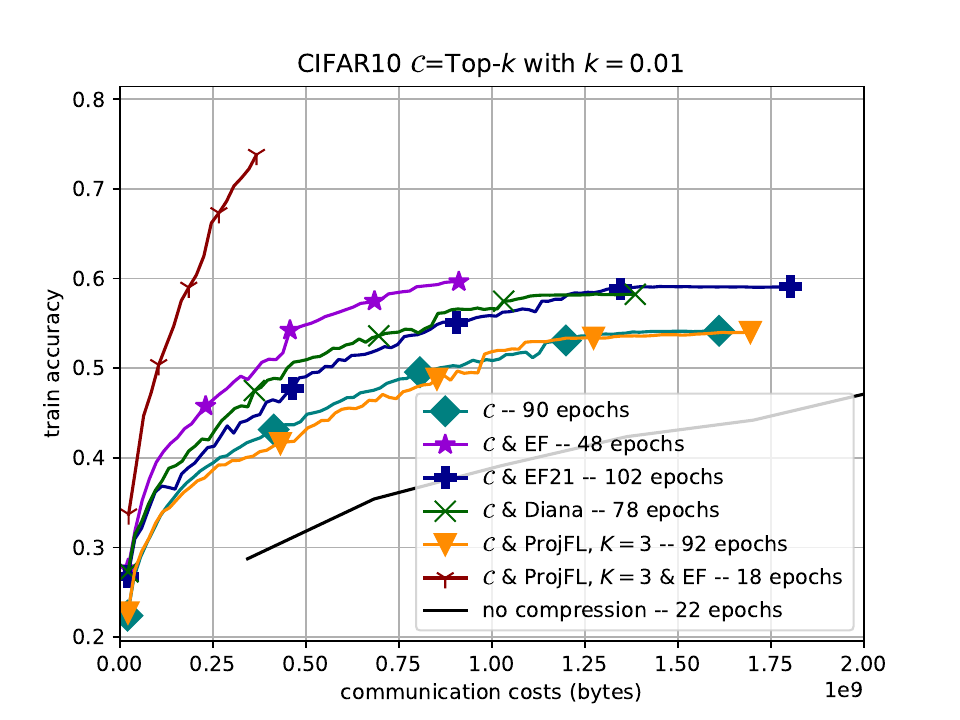}
    \includegraphics[scale=.39]{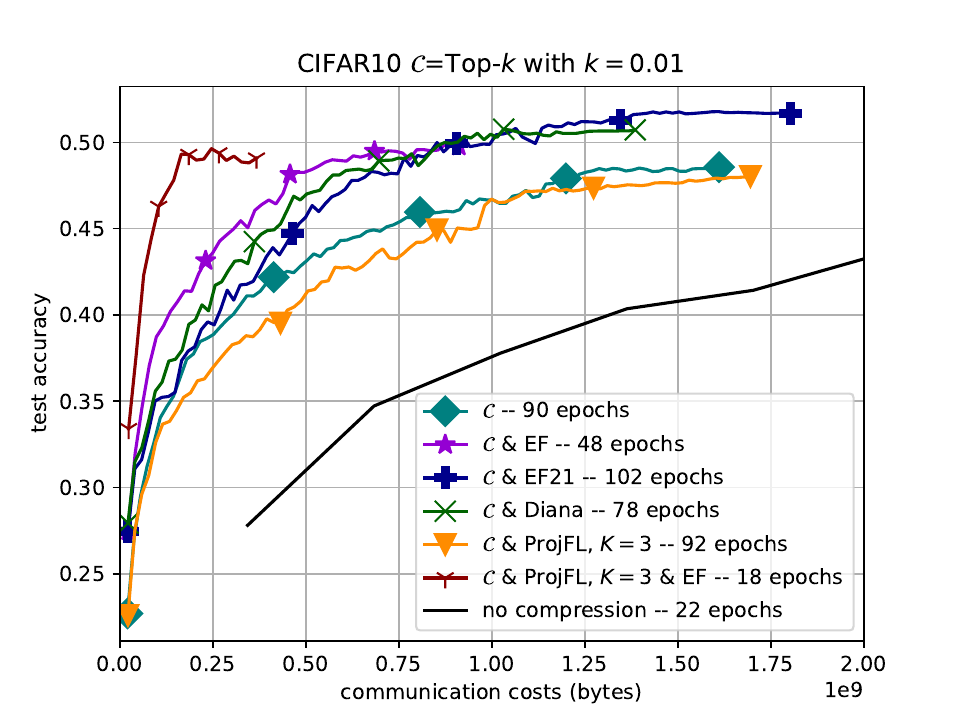} 
\end{figure}

\subsection{Evaluation of \texttt{EF21} and \texttt{DIANA} under different hyperparameters\label{sec:ef21_diana_problem}}

As already mentioned in Section~\ref{sec:numerics_main}, it was observed in \cite{NEURIPS2023_f0b1515b} that Algorithm~\ref{alg:ef21} is particularly sensitive to the batch size. However, on our models, none of the considered batch sizes yielded favorable results, as shown in Figure~\ref{fig11}. For this reason, all comparisons involving the \texttt{EF21} method in our experiments use Algorithm~\ref{alg:ef21_gamma} instead of Algorithm~\ref{alg:ef21}.

For \texttt{DIANA}, a similar observation holds. However, we also tuned the relevant hyperparameters of this algorithm, as shown in Figures~\ref{fig11} and~\ref{fig12}.

\begin{figure}
    \caption{\texttt{EF21} (Alg. \ref{alg:ef21}) under different batch sizes $\mathcal B$ with $M=3$ clients.}\label{fig11}
    \includegraphics[scale=.39]{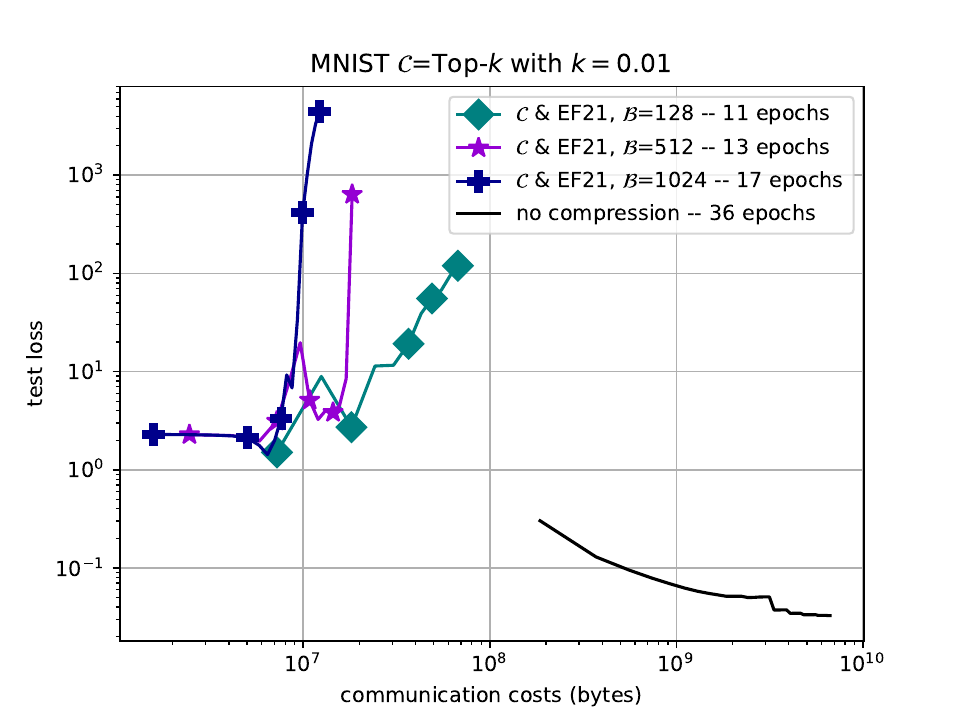} 
    \includegraphics[scale=.39]{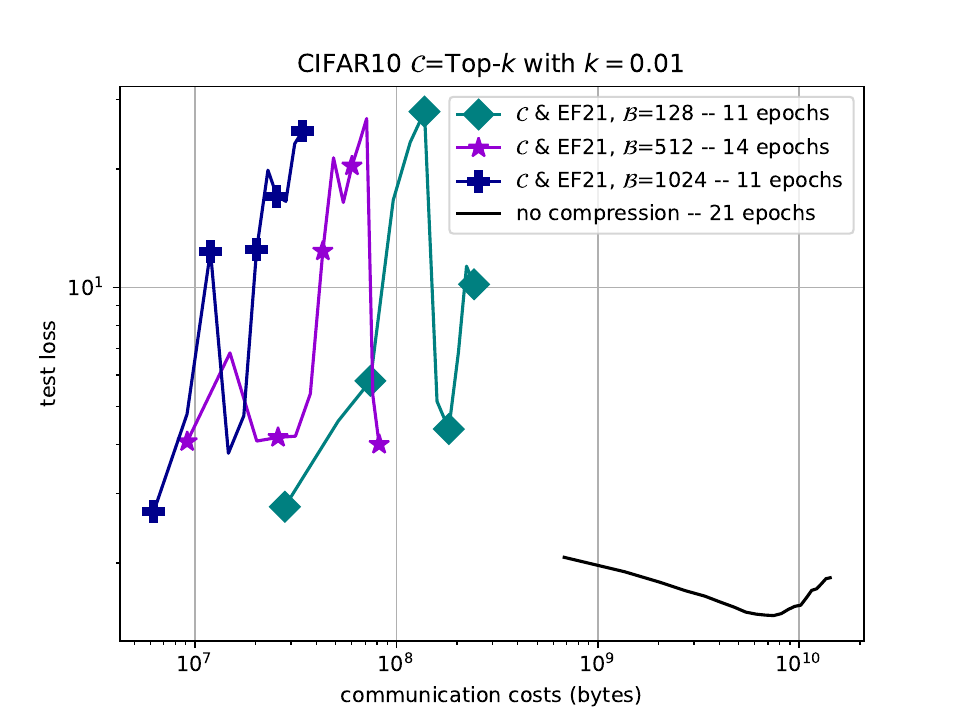} 
\end{figure}

\begin{figure}
    \caption{\texttt{DIANA} (Alg. \ref{alg:diana}) on MNIST under different batch sizes $\mathcal B$ and hyperparameters, with $M=3$ clients.}\label{fig12}
    \centering
    \includegraphics[scale=.5]{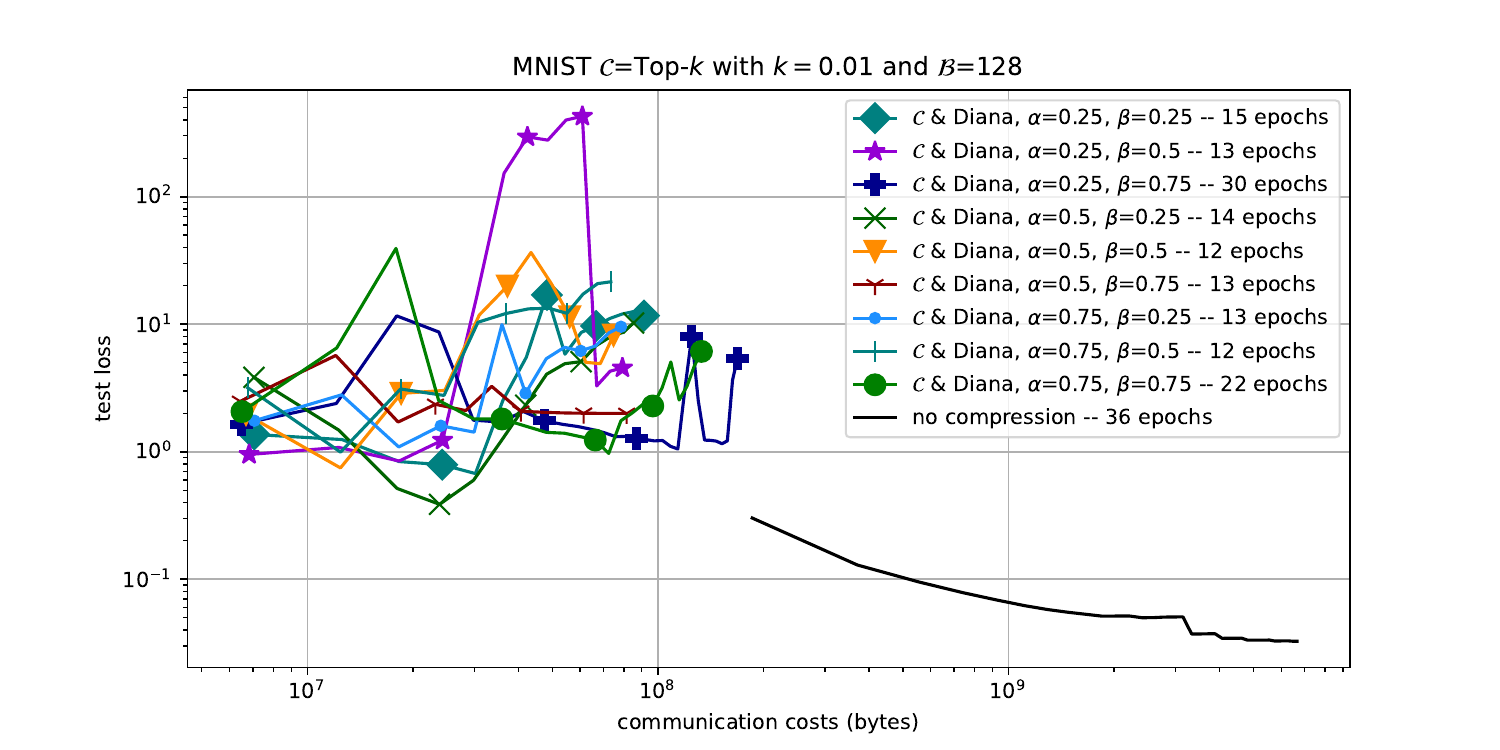} 
    \includegraphics[scale=.5]{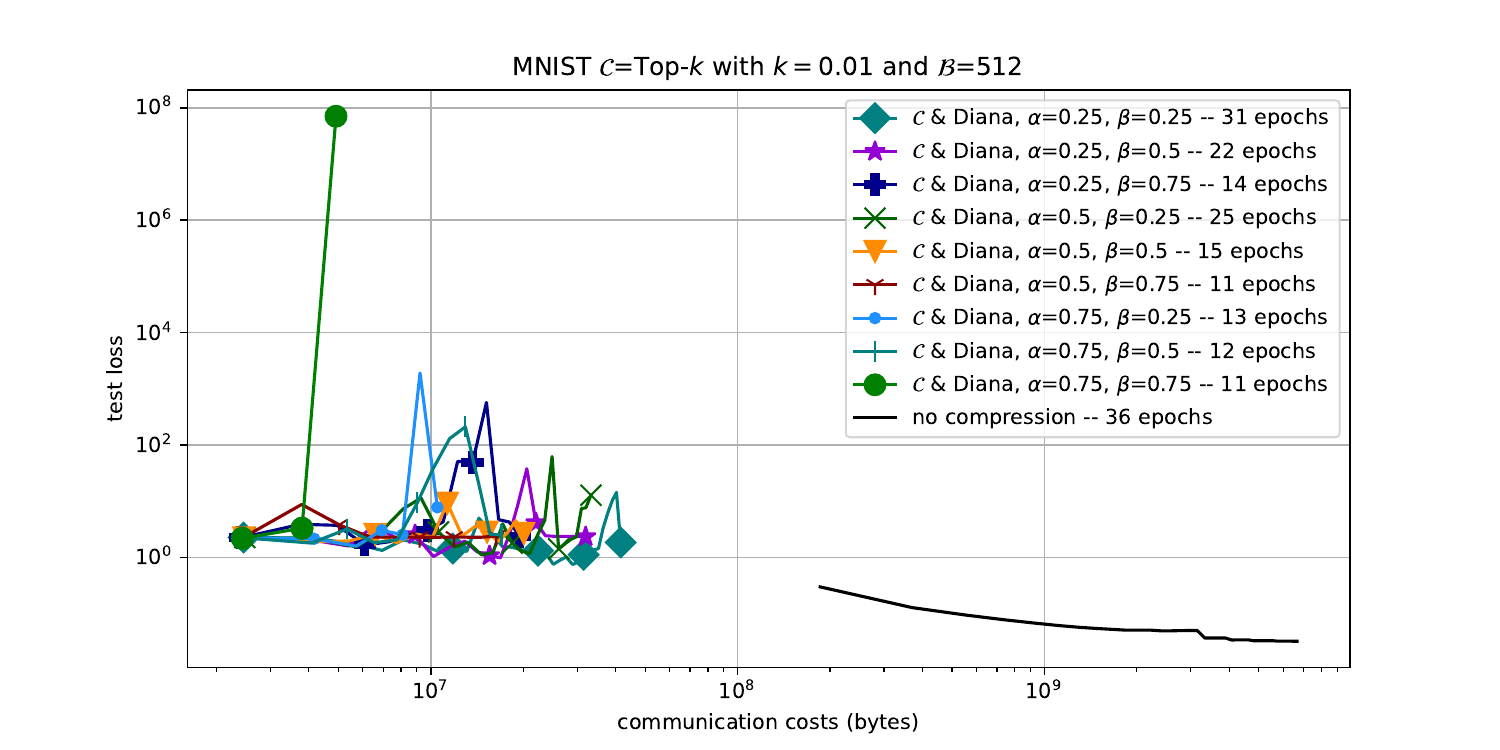} 
    \includegraphics[scale=.5]{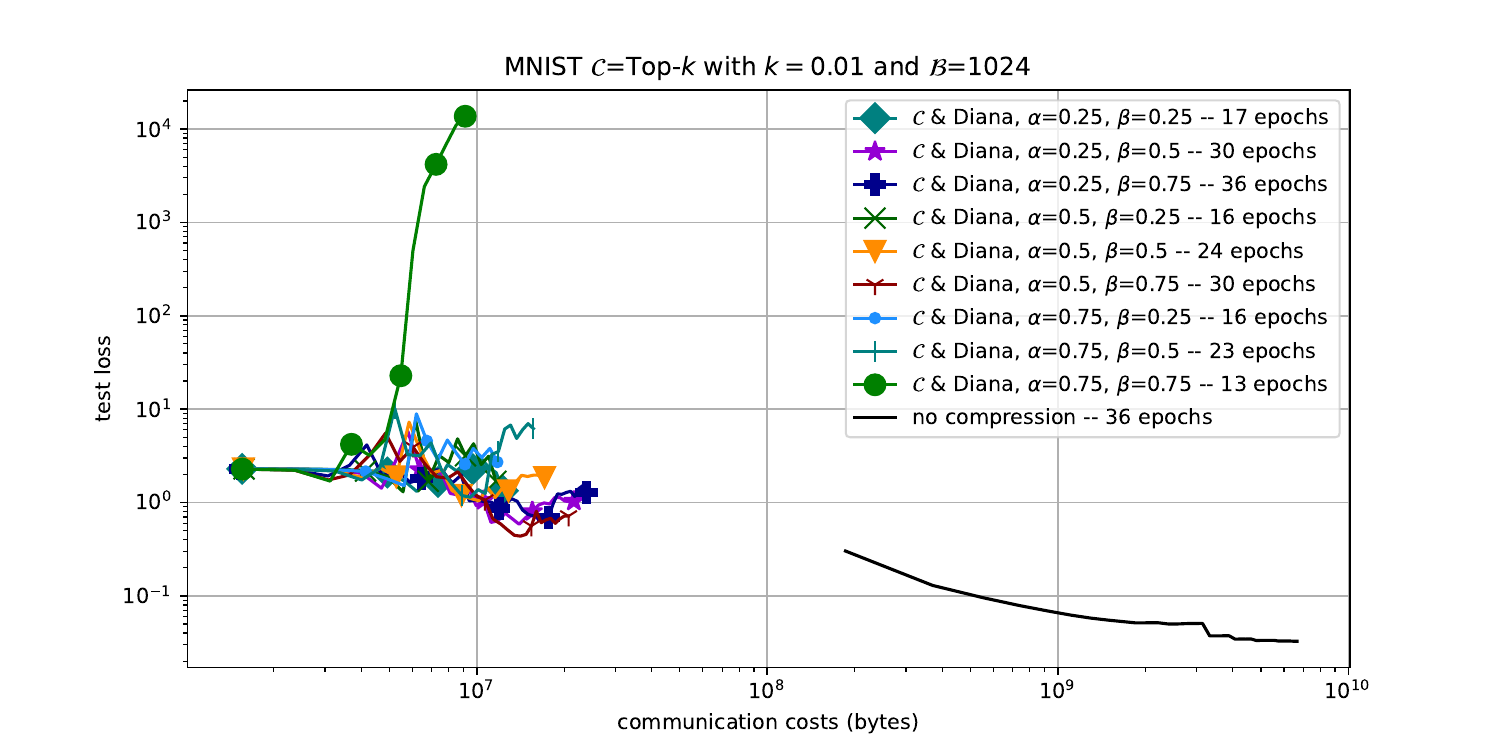}
\end{figure}

\begin{figure}
    \caption{\texttt{DIANA} (Alg. \ref{alg:diana}) on MNIST under different batch sizes $\mathcal B$ and hyperparameters, with $M=3$ clients.}
    \centering
    \includegraphics[scale=.5]{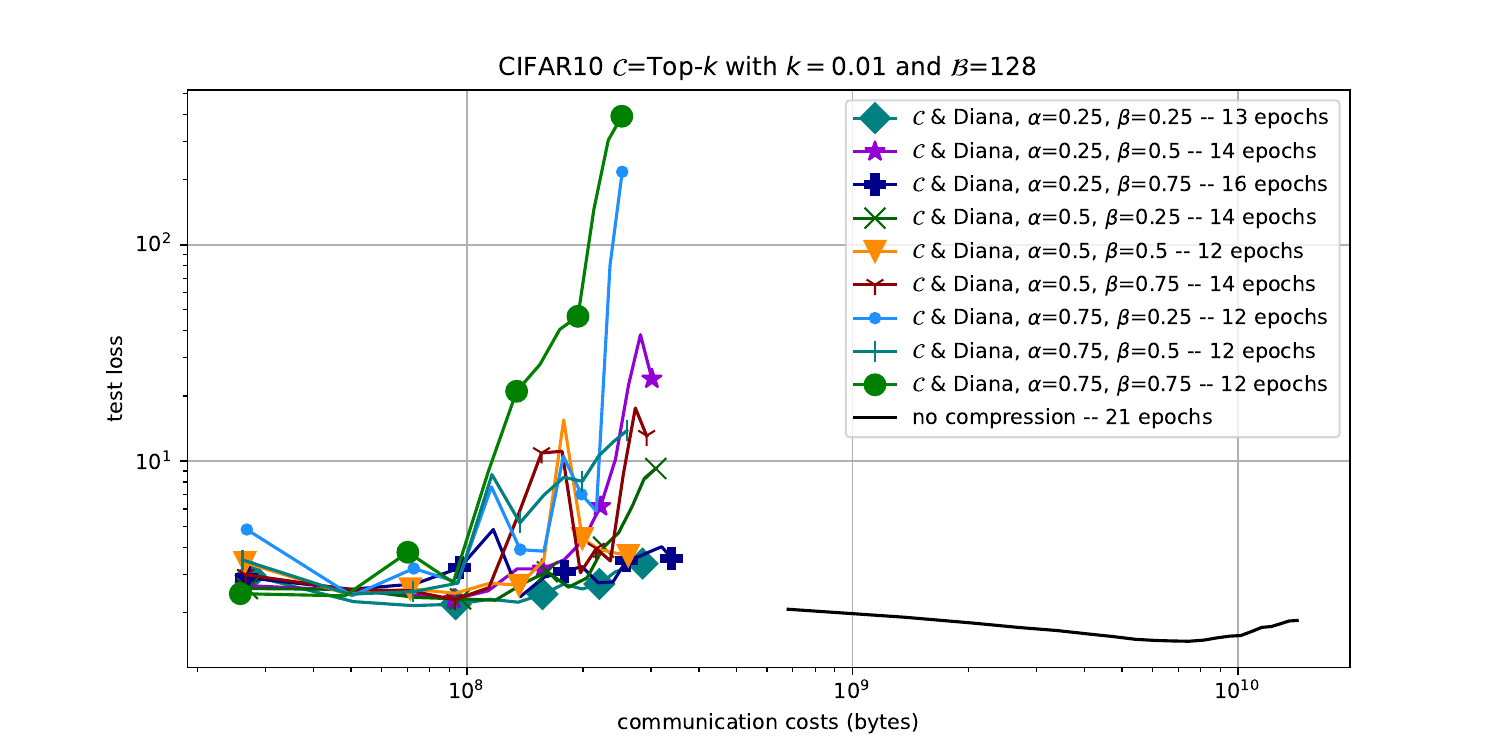} 
    \includegraphics[scale=.5]{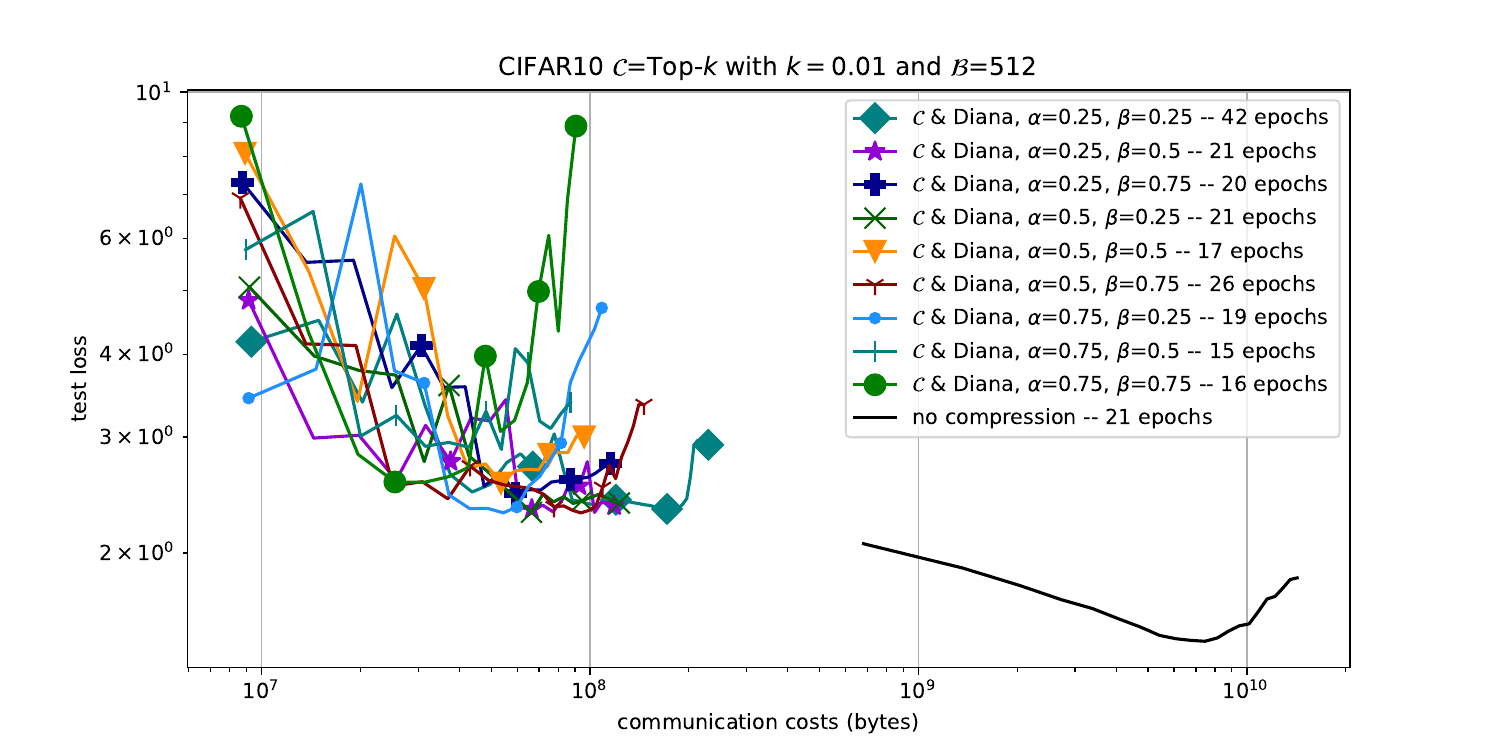} 
    \includegraphics[scale=.5]{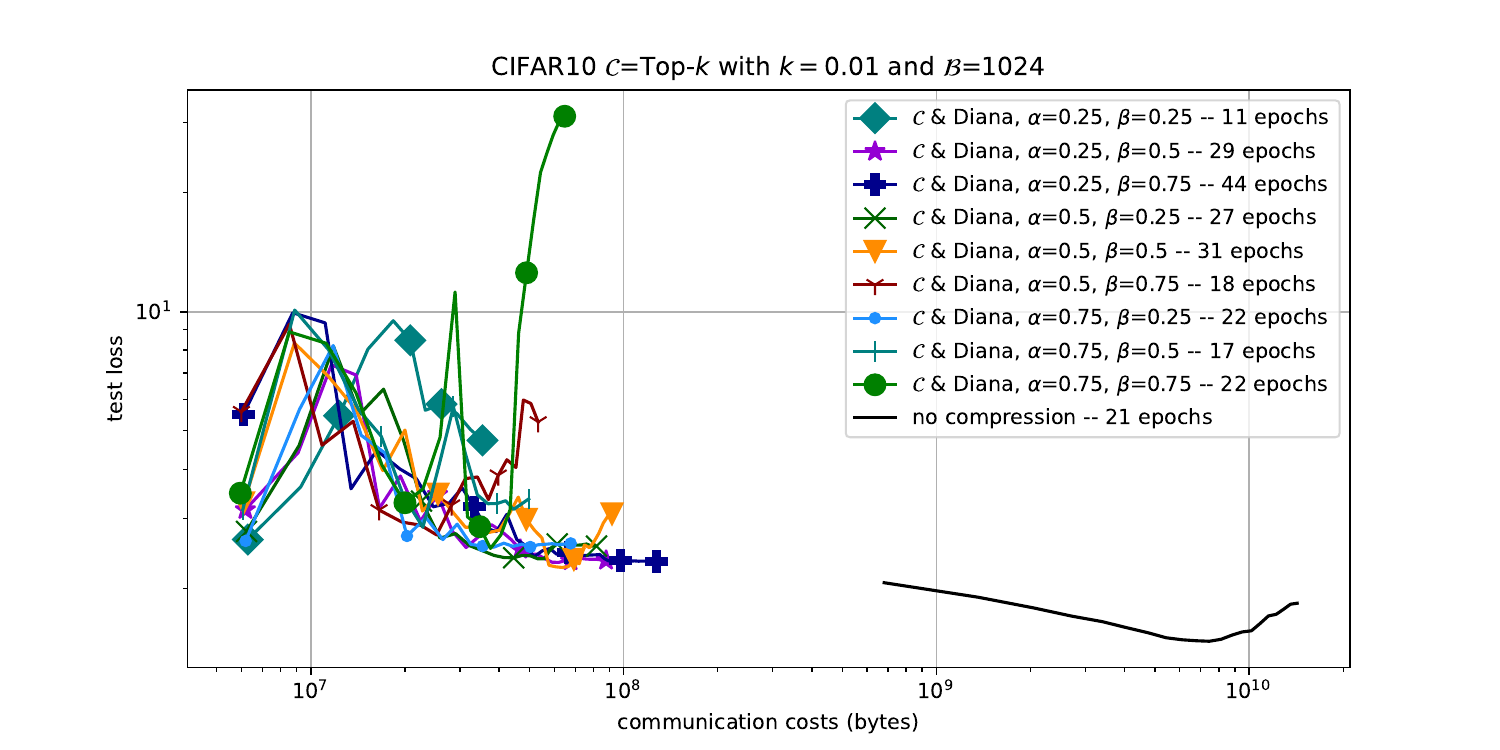}
\end{figure}

\end{appendices}
\end{document}